\documentclass{article}

\usepackage[utf8]{inputenc} 
\usepackage[T1]{fontenc}    
\usepackage{hyperref}       
\usepackage{url}            
\usepackage{booktabs}       
\usepackage{amsfonts}       
\usepackage{nicefrac}       
\usepackage{microtype}      
\usepackage[final]{neurips_2020}
\usepackage{amsmath}
\usepackage{savesym}
\usepackage{amssymb}
\usepackage{mathrsfs}
\usepackage{indentfirst}
\usepackage{graphicx}
\usepackage{multirow}
\usepackage{url}
\usepackage{subfigure}
\usepackage[linesnumbered,ruled,vlined]{algorithm2e}
\newtheorem{theorem}{Theorem}[section]
\newtheorem{lemma}[theorem]{Lemma}
\newtheorem{assumption}[theorem]{Assumption}
\newtheorem{definition}[theorem]{Definition}
\newtheorem{remark}[theorem]{Remark}

\newtheorem{proposition}[theorem]{Proposition}
\newtheorem{corollary}[theorem]{Corollary}
\usepackage{xcolor}
\usepackage{titlesec}
\usepackage{footmisc}

\def\epsilon{\varepsilon}
\def\RR{\mathbb{R}}

\title{Improved Analysis of Clipping Algorithms for Non-convex Optimization}

\author{%
  Bohang Zhang\thanks{Equal Contributions.}\\
  Key Laboratory of Machine Perception, MOE, \\
  School of EECS, Peking University\\
  \texttt{zhangbohang@pku.edu.cn} \\
   \And
   Jikai Jin$^*$\\
   School of Mathematical Sciences \\
   Peking University\\
   \texttt{1900010640@pku.edu.cn} \\
   \And
   Cong Fang \\
   University of Pennsylvania\\
   \texttt{fangcong@pku.edu.cn} \\
   \And
   Liwei Wang \\
   Key Laboratory of Machine Perception, MOE, \\
  School of EECS, Peking University\\
  Center of Data Science, Peking University\\
   \texttt{wanglw@cis.pku.edu.cn} \\
}

\begin{document}

\maketitle

\begin{abstract}
Gradient clipping is commonly used in training deep neural networks partly due to its practicability in relieving the exploding gradient problem. 
Recently,  \citet{zhang2019gradient} show that clipped (stochastic) Gradient Descent (GD)  converges faster than vanilla GD/SGD via introducing a new  assumption called $(L_0, L_1)$-smoothness, which characterizes the violent fluctuation of gradients typically encountered in deep neural networks. However, their iteration complexities on the problem-dependent parameters are rather pessimistic, and theoretical justification of clipping combined with other crucial techniques, e.g. momentum acceleration, are still lacking. In this paper, we bridge the gap by  presenting a general framework to study the clipping algorithms, which also takes momentum methods into consideration.
We provide convergence analysis of the framework in both deterministic and stochastic setting, and demonstrate the tightness of our results by comparing them with existing lower bounds. Our results imply that the efficiency of clipping methods will not degenerate even in highly non-smooth regions of the landscape. Experiments confirm the superiority of clipping-based methods in  deep learning tasks.
\end{abstract}

\section{Introduction}
The problem of the central interest in this paper is to minimize a general non-convex function presented below:
\begin{equation}\label{problem1}
    \min_{x\in \mathbb{R}^d}  F(x),
\end{equation}
where $F(x)$ can be potentially stochastic, i.e.
$$ F(x) = \mathbb{E}_{\xi \sim \mathcal{D}} \left[ f(x,\xi )\right]. $$
For   non-convex optimization problems in form of \eqref{problem1}, since obtaining the global minimum  is NP-hard in general,  this paper takes the concern on a reasonable relaxed criteria: finding an $\epsilon$-approximate first-order stationary point such that $\|\nabla F(x)\|\leq \epsilon$.

We consider gradient-based algorithms to solve \eqref{problem1} and separately study two cases: $\mathrm{i)}$ the  gradient of $F$ given a point $x$ is accessible;  $\mathrm{ii)}$ only a   stochastic  estimator is   accessible. We shall  refer the former  as the deterministic setting and the latter  as the stochastic setting, and we analyze the (stochastic) gradient complexities to search an approximate first-order stationary point for Problem \eqref{problem1}.

Gradient clipping \citep{pascanu2012understanding}  is a simple and commonly used trick in algorithms that   adaptively  choose step sizes  to make optimization stable.   For the task of training deep neural networks (especially for language processing tasks), it is often a standard practice and is believed to be efficient in  relieving the exploding gradient problem
from  empirical studies \citep{pascanu2013difficulty}. More recently, \citet{zhang2019gradient} proposed an inspiring theoretical justification on the clipping technique via introducing the $(L_0,L_1)$-smoothness assumption. The concept of $(L_0,L_1)$-smoothness is defined as follows.

\begin{definition}
We say that a twice differentiable function $F(x)$ is $(L_0,L_1)$-smooth, if for all $x \in \mathbb{R}^d$ we have $\|\nabla^2 F(x) \| \leq L_0+L_1\|\nabla F(x)\|$.
\end{definition}

This assumption can be further relaxed such that twice differentiability is not required (see Remark \ref{remark_l0l1smooth}). Therefore the standard $L$-smoothness assumption (i.e. the gradient of $f$ is $L$-Lipschitz continuous)  is stronger than the
$(L_0,L_1)$-smoothness one in the sense that the latter allows $\|\nabla^2 F(x)\|$ to have a linear growth with respect to $\|\nabla F(x)\|$.

\begin{figure}
  \centering
  \subfigure[]{
  \label{fig_example}
  \includegraphics[width=0.295\textwidth]{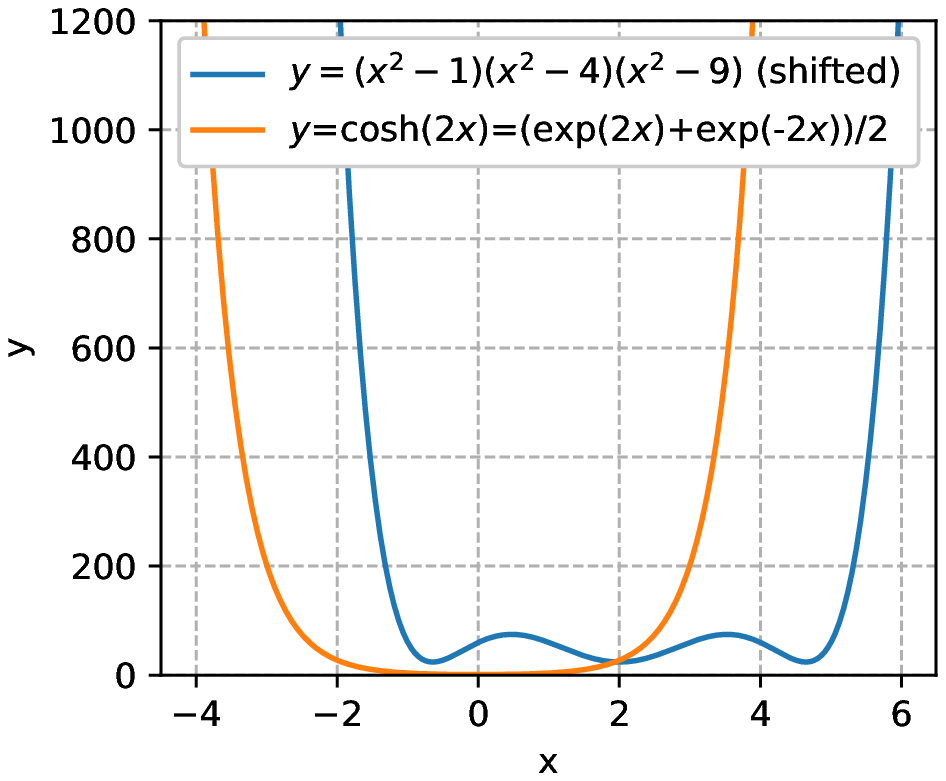}
  }
  \subfigure[]{
  \label{fig_x4_smooth}
  \includegraphics[width=0.30\textwidth]{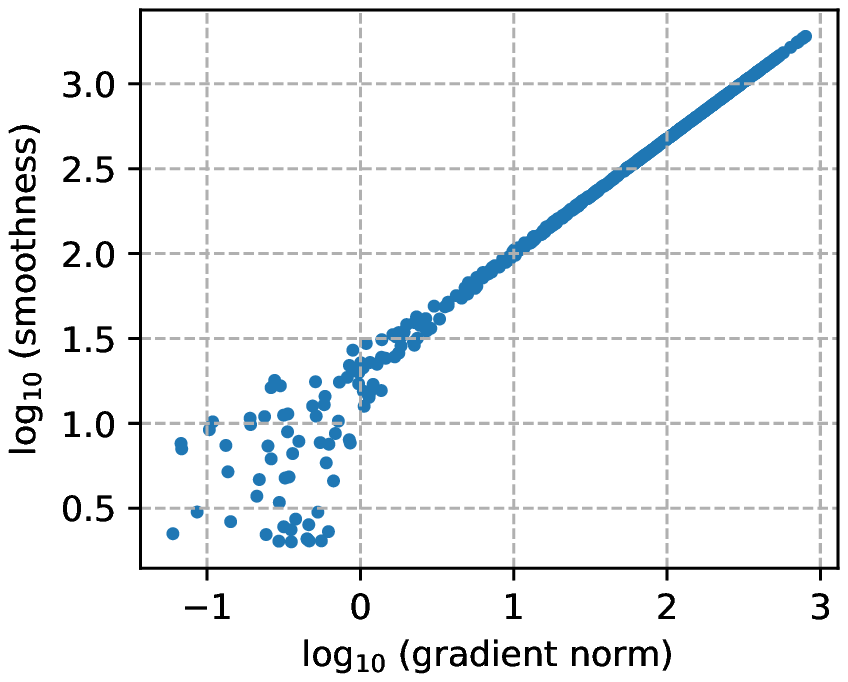}
  }
  \subfigure[]{
  \label{fig_lstm_smooth}
  \includegraphics[width=0.35\textwidth]{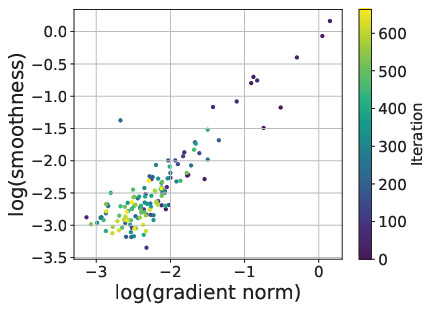}
  }
  \caption{(a) Some simple examples of $(L_0,L_1)$-smooth functions that are not $L$-smooth. (b) The magnitude of gradient norm $\|\nabla F(x)\|$ w.r.t the local smoothness $\|\nabla^2 F(x)\|$ on some sample points for a polynomial $F(x,y)=x^2+(y-3x+2)^4$. We use log-scale axis. The local smoothness strongly correlates to the gradient. (c) Gradient and smoothness in the process of LSTM training, taken from \citet{zhang2019gradient}.}
  \label{fig_introduction}
\end{figure}

$(L_0,L_1)$-smoothness is more realistic than $L$-smoothness. \textit{Firstly}, it includes a variety of simple and important functions which, unfortunately, do not satisfy $L$-smoothness. For example, all univariate polynomials (which can possibly be non-convex) are $(L_0,L_1)$-smooth for $L_1=1$, while a simple function $x^4$ is not globally $L$-smooth for any $L$. Moreover, $(L_0,L_1)$-smoothness also encompasses all functions that belongs to the so-called exponential family. Figure \ref{fig_introduction}(a) presents some simple examples and Figure \ref{fig_introduction}(b) shows that the local smoothness of $(L_0,L_1)$-smooth functions strongly correlates to the gradient norm.

\textit{Secondly}, \citet{zhang2019gradient} performed experiments to show that $(L_0,L_1)$-smoothness is a preciser characterization of the landscapes for  objective functions in many real-world tasks, especially for training a deep neural network model. It was observed that the local Lipschitz constant $L_0$ near the stationary point is thousands of times smaller than the global one $L$ in the LSTM training (see Figure \ref{fig_introduction}(c) taken from \citet{zhang2019gradient}).

Seeing this, it is desirable to give a comprehensive and deep analysis on iteration complexities for $(L_0,L_1)$-smooth objectives. How fast can we achieve to find  a first-order stationary point for $(L_0,L_1)$-smooth functions? What are simple algorithms that provably achieve such a convergence rate? In this paper, we give  \emph{affirmative} answers to the above questions.
In fact,  due to the violent fluctuation of gradients, the efficiency of (stochastic) Gradient Descent with a constant step size degenerates, whereas we will show in this paper that by  simply combining the clipping technique,  a wide range of algorithms can achieve much better convergence rate for $(L_0,L_1)$-smooth functions. In fact, when $\epsilon$ is small, the complexities (i.e. the number of gradient queries required) are $\mathcal{O}\left( \Delta L_0\epsilon^{-2} \right)$ for the deterministic setting and $\mathcal{O}\left( \Delta L_0\sigma^2\epsilon^{-4} \right)$ for the stochastic setting (see Section \ref{section_analysis} for details), which are both independent of $L_1$.  Compared with \citet{zhang2019gradient} who only studied clipped (stochastic) gradient descent, we consider proposing a unified framework which contains a variety of clipping-based algorithms and achieve much  sharper complexities.  The main  technique for our proof is by introducing a novel Lyapunov function which does not appear in existing studies. We believe that our work provides better understandings for the clipping technique in training deep neural networks. We summarize the contributions of the paper in the following.
\begin{itemize}
    \item We provide a general framework to analyze the clipping technique for optimizing $(L_0,L_1)$-smooth functions. It contains a variety of clipping algorithms, including gradient clipping and momentum clipping as special cases.
    \item We provide convergence analysis for the general framework we propose. We show that our bounds are \textit{tight} by comparing with existing lower bounds. For gradient clipping, a special case in our framework, our result is much sharper than that proposed by \citet{zhang2019gradient}.
    \item   We conduct extensive experiments on a variety of different tasks, and observe that the clipping algorithms consistently perform better than vanilla ones.
\end{itemize}


\textbf{Notations.} For a vector $x\in \mathbb R^d$, we denote $\|x\|$ as the $l_2$-norm of $x$. For a matrix $A\in \mathbb R^{m\times n}$, let $\|A\|$ be the spectral norm of $A$. Given functions $f, g$ : $\mathcal X \rightarrow [0, \infty)$ where $\mathcal X$ is any set, we say $f = \mathcal{O}(g)$ if there exists a constant $c >0$ such that $f(x) \le c g(x)$ for all $x \in \mathcal X$, and $f = \Omega (g)$ if there exists a constant $c >0$ such that $f(x) \ge c g(x)$ for all $x \in \mathcal X$. We say $f = \Theta (g)$ if $f = \mathcal O (g)$ and $f = \Omega (g)$.

\subsection{Related Work}
\label{section_related_work}
\textbf{Clipping/normalizing Techniques. } Clipping/normalizing has long been a popular technique in optimizing large-scale non-convex optimization problems (e.g. \citep{mikolov2012statistical,pascanu2013difficulty,goodfellow2016deep, you2017scaling}). There are several views which provide understandings for the clipping and normalizing techniques.   Some show that clipping can  reduce the stochastic noise. For example, 
  \citet{zhang2019adam,gorbunov2020stochastic} showed that clipping is crucial for convergence when the stochastic gradient noise is heavy-tailed. \citet{menon2020can} pointed out that clipping can mitigate the effect of label noise.  
  \citet{cutkosky2020momentum} found that adding momentum in normalized SGD provably reduces the stochastic noise. Another line of works try to understand the function of clipping and normalizing for  the standard smooth  optimization. For example, \citet{levy2016power} showed that normalized GD can provably escape saddle points. \citet{fang2018spider} designed a new algorithm based on normalized GD which achieves a faster convergence rate under suitable conditions. Gradient clipping has also been used to design differentially private optimization algorithms \citep{abadi2016deep}.
  
  The work of \citet{zhang2019gradient} is mostly related to this paper. A detailed comparison between the two works is shown in Subsection \ref{sub:comparation}. 

\textbf{Lower Bounds For Non-convex Optimization. }A series of recent works establish lower bounds for finding an $\epsilon$-stationary point of a general non-convex and  $L$-smooth function, either in deterministic setting \citep{carmon2019lower} or in stochastic setting \citep{drori2019complexity,arjevani2019lower}. In this paper we borrow their counter examples to show the  tightness of our obtained complexities  for the general $(L_0,L_1)$-smooth functions.

\section{Assumptions \& Comparisons of  Results}
\label{section_assumption}
\subsection{Assumptions}
We first present the assumptions that will be used in our theoretical analysis, which  follows from \cite{zhang2019gradient}.

\begin{assumption}
\label{Delta}
We assume  $\Delta := F(x_0)- F^* < \infty$ where  $F^*  =  \inf_{x\in\RR^d} F(x)$ is the global infimum value of $F(x)$.
\end{assumption}

\begin{assumption}\label{ass:2}
\label{assumption_ls+}
We assume  that $F(x)$ is $(L_0,L_1)$-smooth. 
\end{assumption}

\begin{remark}
\label{remark_l0l1smooth}
This assumption can be relaxed to the following: there exists $K_0,K_1 >0$ such that for all $x, y \in \mathbb{R}^d$, if $\|x-y\| \leq \frac{1}{K_1}$, then
$$\|\nabla F(x) - \nabla F(y) \| \leq (K_0 + K_1 \|\nabla F(y)\|) \|x-y\|$$
It does not need $F$ to be twice differentiable and is \textit{strictly weaker than $L$-smoothness}.
\end{remark}
In the stochastic setting, for the briefness of our analysis,  we assume that the noise is unbiased and bounded.
\begin{assumption}
\label{BN}
For all $x \in \mathbb{R}^d$, $\mathbb{E}_{\xi } \left[ \nabla f(x,\xi) \right] = \nabla F(x)$. Furthermore, there exists $\sigma >0$ such that for all $x \in \mathbb{R}^d$, the noise satisfies $\|\nabla f(x,\xi ) - \nabla F(x) \| \leq \sigma$ with probability $1$.
\end{assumption}

Note that another commonly used noise assumption in the optimization literature is \textit{bounded variance} assumption, i.e. $\mathbb{E}_{\xi} \left[ \|\nabla f(x,\xi ) - \nabla F(x) \|^2  \right] \leq \sigma^2$, therefore Assumption \ref{BN} is stronger than it. However, we adopt Assumption \ref{BN} as it is also used in the original work of gradient clipping in \cite{zhang2019gradient}, and it makes our analysis much simpler.

\subsection{ Comparisons of  Results}\label{sub:comparation}

In this paper, we mainly take our concern on the stochastic setting, though we also establish the complexities for the deterministic setting.  We summarize the comparison in the stochastic setting with existing complexity results in Table \ref{complexity-table}, in which we only present the dominating complexities with respect to $\epsilon$.
  
\begin{table}[h]
\centering
\caption{Comparisons of  gradient complexity  
in the stochastic setting. }
\label{complexity-table}
\centering
\begin{tabular}{cc}
\hline
Algorithms                               & Complexities$^{*}$   \\\hline
SGD \citep{ghadimi2013stochastic} & $\mathcal{O}\left( \Delta (L_0+L_1M)\sigma^2\epsilon^{-4} \right)$ $^{**}$\\
Clipped SGD \citep{zhang2019gradient} & $\mathcal{O}\left((\Delta + (L_0+L_1\sigma)\sigma^2+\sigma L_0^2/L_1 )^2\epsilon^{-4} \right)$       \\
Clipping Framework (this paper) & $\mathcal{O}\left(\Delta L_0\sigma^2\epsilon^{-4}\right)$   \\
Lower Bound  & $\Omega\left(\Delta L_0\sigma^2\epsilon^{-4}\right)$ $^{***}$  \\
\hline
\end{tabular}
\\
\footnotesize
\begin{flushleft}
{$^{*}$For clarity, we only present the dominating term (with respect to $\epsilon$) here.\\}
{$^{**}$For SGD, we further assume the gradient norm is upper bounded by $M$.\\}
{$^{***}$See section \ref{lower_bound_discussion} for a detailed discussion of the lower bound.}
\end{flushleft}
\end{table}
  
For standard SGD, if we further assume that the gradient is upper bounded by $M$, i.e. 
$M:=\sup_{x\in \RR^d}\|\nabla f(x)\|<\infty$, Assumption \ref{ass:2} leads to an upper bound of the global Lipschitz constants $L:=L_0+L_1M$. Therefore, the standard results for $L$-smooth functions (e.g. \citep{ghadimi2013stochastic}) implies that  Gradient Descent with a constant step size  can achieve  complexity of $\mathcal{O}(\Delta(L_0+L_1 M)\sigma^{2}\epsilon^{-4}) $ for finding a first-order stationary point in the  stochastic setting\footnote {We will show in Appendix \ref{appendix_lower_bound} that even using Assumption \ref{BN}, such upper bound can not be improved.}. However, the upper bound of the gradient $M$ is typically very large, especially when the parameters have a poor initialization, which makes SGD converges arbitrarily slow. In contrast, our result indicates that the clipping framework (shown in Algorithm \ref{FW} which includes a variety of clipping methods) achieves complexity of $\mathcal{O}(\Delta L_0\sigma^{2}\epsilon^{-4})$, therefore the dominating term of our bound is independent of both $M$ and $L_1$.  This provides a strong justification for the efficacy of clipping methods.

Compared with the bound for clipped SGD established  in \citet{zhang2019gradient}, our results improve theirs on the dependencies for all problem-dependent parameters, i.e. $\Delta$, $\sigma$, $L_0$, and especially $L_1$ by order. For SGD with arbitrarily chosen step sizes (thus include clipped SGD), the example in \citet{drori2019complexity} can be used to show that clipped SGD is optimal (cf. Section \ref{lower_bound_discussion}).


\section{General Analysis of Clipping}
\label{section_analysis}
We  aim to present a general framework in which we can provide a unified analysis for commonly used clipping-based algorithms. Since momentum is one of the most popular acceleration technique in optimization community, our framework takes this acceleration procedure  into account. We show our framework  in  Algorithm \ref{FW}, where  we can simply replace $\nabla f(x_t,\xi_t)$ by $\nabla F(x_t)$ for  the deterministic setting.

\begin{algorithm}[!htbp]
\label{FW}
\SetKwInOut{KIN}{Input}
\caption{The General Clipping Framework}
\KIN{Initial point $x_0$, learning rate $\eta$, clipping parameter $\gamma$, momentum $\beta\in [0,1)$, interpolation parameter $\nu\in [0,1]$ and the total number of iterations $T$}
Initialize $m_0$ arbitrarily\;
\For{$t \gets 0$ {to} $T-1$}{
    Compute the stochastic gradient $\nabla f(x_t,\xi_t)$ for the current point $x_t$\;
	$m_{t+1} \gets \beta m_{t} + (1-\beta )\nabla f(x_t,\xi_t)$\;
	$x_{t+1} \gets x_t - \left[ \nu \min \left( \eta, \dfrac{\gamma}{\|m_{t+1}\|} \right)m_{t+1} + (1-\nu )\min \left(\eta, \dfrac{\gamma}{\|\nabla f(x_t,\xi_t)\|}\right)\nabla f(x_t,\xi_t) \right]$\;
}

\end{algorithm}

 We notice that our framework is similar to the Quasi-Hyperbolic Momentum(QHM) algorithm proposed by \citet{ma2018quasi}, while they did not consider the clipping technique. They pointed out that QHM contains a wide range of popular algorithms (e.g. SGD+momentum, Nesterov Accelerated SGD, AccSGD, etc). As a result, for different choice of hyper-parameters, our framework encompasses the clipping version of all these algorithms. We now discuss several representative examples in our framework.
\begin{itemize}
    \item \textbf{Gradient Clipping.} By choosing $\nu = 0$ in Algorithm \ref{FW}, we obtain the \textit{clipped GD/SGD algorithm} which can be written as:
    $$x_{t+1} \gets x_t - \min \left(\eta, {\gamma}/{\|\nabla f(x_t,\xi_t)\|}\right)\nabla f(x_t,\xi_t) $$
    It follows that in gradient clipping, the gradient is clipped to have its norm no more than $\gamma/\eta$.
    \item \textbf{Momentum Clipping.} By choosing $\nu = 1$ in Algorithm \ref{FW}, we perform the update using a clipped version of momentum which can be written as:
    $$x_{t+1} \gets x_t - \min \left(\eta, {\gamma}/{\|m_{t+1}\|}\right)m_{t+1} $$
    The approach has already been used in previous works \citep{zhang2019adam,zhang2020complexity}, albeit in different settings. To the best of our knowledge, there is no existing analysis of this algorithm even for optimizing standard $L$-smooth functions.
    \item \textbf{Mixed Clipping.} By choosing  $\nu\in(0,1)$, we obtain the mixed clipping algorithm. Although this form of clipping is not widely used in practice, we observe from experiments that it typically converges faster than both gradient clipping and momentum clipping. Some explanations of this observation are provided in Appendix \ref{section_explanation}.
    \item \textbf{Normalized Momentum.} By choosing $\nu = 1$ and $\eta \rightarrow + \infty$ in Algorithm \ref{FW}, we recover the \textit{normalized SGD+momentum algorithm}. This algorithm performs a normalized (rather than clipped) update in each iteration. It has been analyzed in \citet{cutkosky2020momentum} for $L$-smooth functions and a layer-wise variant was used in the LARS algorithm \citep{you2017scaling}. We will provide a detailed discussion of this algorithm in the Appendix \ref{appendix_normalized_momentum}.
\end{itemize}

\subsection{Main Results}
In this section we first deal with the deterministic case, in which we can get a strong justification that clipping is a natural choice to optimize $(L_0,L_1)$-smooth functions. We have the following Theorem.
\begin{theorem}
\label{theorem3}
\textbf{[Convergence of Algorithm \ref{FW}, Deterministic Setting]} Let the function $F$ satisfy Assumptions \ref{Delta} and \ref{assumption_ls+}. Set $m_0=\nabla F(x_0)$ in Algorithm \ref{FW} for simplicity. Fix $\epsilon>0$ be a small constant. For any $0\le \beta<1$ and $0\le \nu\le 1$, if $\gamma \le \frac {1-\beta} {10BL_1}$ and $\eta \le \frac{1-\beta}{10AL_0}$
where $A=1.06$, $B=1.06$, then 
$$\dfrac{1}{T}\sum_{t=1}^T \|\nabla F(x_t)\| \le 2\epsilon $$
as long as
\begin{equation}
\begin{aligned}
T &\ge  3\Delta \max \left\{ \dfrac{1}{\epsilon^2\eta},\dfrac{25\eta}{\gamma^2} \right\}.
\end{aligned}
\end{equation}
\end{theorem}
In Theorem \ref{theorem3}, the $(L_0,L_1)$-smoothness is precisely reflected in the restriction of hyper-parameters $\gamma=\mathcal O(1/L_1)$ and $\eta=\mathcal O(1/L_0)$. For large $L_1$, we must use a small clipping hyper-parameter to guarantee convergence. 
This also coincides with the intuition that in highly non-smooth regions we should take a small step.

Theorem \ref{theorem3} states that in the deterministic setting, for any $\epsilon>0$, our framework can find an $\epsilon$-approximate stationary point in $\mathcal{O}\left(\Delta \max \left\{ \frac{L_0}{\epsilon^2},\frac{L_1^2}{L_0} \right\} \right)$ gradient evaluations if we choose $\gamma = \Theta\left( {1}/{L_1} \right)$ and $\eta = \Theta \left( {1}/{L_0} \right)$. When $\epsilon = \mathcal{O}(L_0/L_1)$ , the dominating term is $\mathcal{O}\left(\Delta{ L_0} {\epsilon^{-2}} \right)$.

Now we turn to our main result in the stochastic setting. We have the following theorem.

\begin{theorem}
\label{theorem4}
\textbf{[Convergence of Algorithm \ref{FW}, Stochastic setting]} Let the function $F$ satisfy Assumptions \ref{Delta} and \ref{assumption_ls+}, and the noise satisfies Assumption \ref{BN} with $\sigma\ge 1$. Set $m_0=\nabla F(x_0)$ in Algorithm \ref{FW} for simplicity. Fix $0<\epsilon\le 0.1$ be a small constant. For any $0\le \beta<1$ and $0\le \nu\le 1$, if $\gamma\le\frac {\epsilon}{2\sigma}\min\left\{\frac {\epsilon}{AL_0}, \frac {1-\beta}{AL_0},\frac {1-\beta}{25BL_1}\right\}$ and $\gamma/\eta=5\sigma$ where constants $A=1.01, B=1.01$,  then  
 \begin{equation}
     \dfrac{1}{T} \sum_{t=1}^T \mathbb{E} \|\nabla F(x_t)\| \leq 3\epsilon
 \end{equation}
as long as
\begin{equation}
\begin{aligned}
\label{Tstoc}
T &\ge \frac {3}{\epsilon^2\eta}\Delta.
\end{aligned}
\end{equation}
Here the expectation is taken over all the randomness $\xi_0,\cdots,\xi_{T-1}$.
\end{theorem}

Theorem \ref{theorem4} shows that in the stochastic setting, for any $\epsilon>0$, our framework can find an $\epsilon$-approximate stationary point in $\mathcal{O}\left(\Delta\sigma^2 \left( \max\left\{\frac{ L_0}{\epsilon^{4}},\frac{L_1^4}{L_0^3}\right\}\right) \right)$ gradient evaluations. When $\epsilon < \min \left\{1, \frac{L_0}{25L_1} \right\} (1-\beta)$ , the term $\min\left\{ \frac {\epsilon}{AL_0},\frac {1-\beta}{AL_0},\frac {1-\beta}{25BL_1}\right\}$ reduces to $\frac {\epsilon} {AL_0}$. In this case $L_1$ no longer affects the choice of steps sizes $\eta$ and $\gamma$, and the complexity in \eqref{Tstoc} reduces to $\mathcal{O}\left(\Delta L_0 \sigma^2 \epsilon^{-4}\right) $.

Theorem \ref{theorem4} suggests that the clipping threshold should take $\gamma/\eta=\Theta(\sigma)$, which only depends on the noise and is several times larger than the its variance. This matches previous understanding of gradient clipping, in that clipping the stochastic gradient controls the variance while introducing some additional bias, and the clipping threshold should be tuned
to trade-off variance with the introduced bias \citep{zhang2019adam}.

We emphasize that in both settings, the dominating  terms in our upper bounds are independent of the gradient upper bound $M$ and the smoothness parameter $L_1$.
In other words, the efficiency of Algorithm \ref{FW} is essentially unaffected by these quantities. Recall that $M$ and $L_1$ are related to steep cliffs in the landscape where the gradient may be large or fluctuate violently. Therefore, our results suggest that such non-smoothness can be tackled with clipping methods without sacrificing efficiency.

\subsection{Proof Sketch}
The analysis of Algorithm \ref{FW} is in fact challenging, as it uses both momentum and adaptive step sizes. Also, the general $(L_0,L_1)$-smoothness assumption makes things more complicated. In this subsection we briefly introduce our proof technique. We hope our proof is also useful to a better understanding of other adaptive algorithms that combine momentum (such as Adam \citep{kingma2014adam}).

\textbf{Proof sketch of Theorem \ref{theorem3}.} Due to the momentum term, each step in Algorithm \ref{FW} is not necessarily a descent one, which makes it difficult to prove convergence using traditional techniques. Instead, we construct a novel \textit{Lyapunov function} as follows:
\begin{equation}
    G(x,m)=F(x)+\frac {\beta\nu} {2(1-\beta)} \min\left({\eta} \|m\|^2,{\gamma}\|m\|\right)
\end{equation}
We aim to analyze the descent property of the sequence $\{G(x_t,m_t)\}_{t=0}^{T}$. Define $\rho=\gamma/\eta$, $\mathcal{S} := \{ t\in\mathbb N : t< T, \max(\|\nabla F(x_t)\|,\|m_t\|,\|m_{t+1}\|)\ge \rho \}$ and $\overline{\mathcal{S}}=\{ t\in\mathbb N : t< T\}\backslash \mathcal S$. Let $T_{\mathcal{S}} = \left| \mathcal{S} \right|$. We separately provide one-step analysis for the two cases, as stated in Lemma \ref{proof_techinique_1} and \ref{proof_techinique_1.5} respectively.

\begin{lemma}
\label{proof_techinique_1}
For any $t \in \mathcal{S}$, we have
\begin{equation}
        G(x_{t},m_{t})-G(x_{t+1},m_{t+1})=\Omega\left(\gamma(1-\gamma)(\|\nabla F(x_t)\|+\rho)-\gamma\|\nabla F(x_t)-m_t\|\right)
    \end{equation}
\end{lemma}
We prove this lemma by using the $(L_0,L_1)$-smoothness properties deduced in Appendix \ref{appendix_property} and conducting a comprehensive discussion on three cases in Lemmas B.1-B.3. Furthermore, we show in Lemma B.4 that $\sum_{t \in \mathcal{S}} \|\nabla F(x_t)-m_t\| = \mathcal{O}\left( \gamma T_{\mathcal{S}} (\rho+ \sum_{t \in \mathcal{S}}\|\nabla F(x_t)\|) \right)$. By choosing a small enough $\gamma$ and carefully dealing with constants, we can conclude that the total amount of decrease of the Lyapunov function is $\Omega \left( \rho\gamma T_{\mathcal{S}} \right)$.

\begin{lemma}
\label{proof_techinique_1.5}
For any $t \in \overline{\mathcal{S}}$, if $\eta=\mathcal O(1/L_0)$ and $\gamma=\mathcal O(1/L_1)$, then we have
\begin{equation}
    \label{proof_techinique_2}
    G(x_{t},m_{t})-G(x_{t+1},m_{t+1})=\Omega\left(\eta\left((1-\nu\beta)\|\nabla F(x_t)\|^2+\nu\beta \|m_t\|^2\right)\right)
\end{equation}
\end{lemma}
We prove \eqref{proof_techinique_2} in Lemma B.5-B.7 by the fact that Algorithm \ref{FW} performs an unclipped update if $t\in \overline{\mathcal{S}}$. Since the bound \eqref{proof_techinique_2} is small in term of $\|\nabla F(x_t)\|$ if $\beta$ and $\nu$ are close to 1, we convert $\|m_t\|$ to $\|\nabla F(x_t)\|$ by proving that $\sum_{t \in \overline{\mathcal{S}}} \|m_t\| = \Omega \left( \sum_{t \in \overline{\mathcal{S}}} \|\nabla F(x_t)\| -  \gamma(\rho T_{\mathcal{S}} + \sum_{t \in \mathcal{S}}\|\nabla F(x_t)\| )\right)$ in Lemma B.8. The term related to $\mathcal{S}$ can all be offset by the terms in Lemma \ref{proof_techinique_1}, and the term related to $\overline{\mathcal{S}}$ combined with $\eta (1-\nu\beta)\|\nabla F(x_t)\|^2$ can be shown to ensure a descent amount of $\Omega \left(\epsilon\eta \|\nabla F(x_t)\| - \epsilon^2\eta\right)$, which is $\Omega \left( \epsilon^2\eta\right)$ as long as $\|\nabla F(x_t)\|\geq 2\epsilon$.

Finally, by combining the above two cases we obtain the conclusion in Theorem \ref{theorem3}.
 
\textbf{Proof sketch of Theorem \ref{theorem4}.} In the stochastic setting, it requires a different treatment to deal with the noise. We define the \textit{true momentum} $\tilde{m}$ recursively by
\begin{equation}
\label{proof_techinique_3}
    \tilde m_{t+1}=\beta \tilde m_t+(1-\beta) \nabla F(x_t)
\end{equation}
where $\tilde m_0=m_0$, and analyze the descent property of the following sequence $\{G(x_t,\tilde m_t)\}_{t=0}^{T}$. We also consider two cases:  $\max(5\|\nabla F(x_t)\|/4,\|\tilde m_t\|,\|m_{t+1}\|)\ge \rho$ and $\max(5\|\nabla F(x_t)\|/4,\|\tilde m_t\|,\|m_{t+1}\|)< \rho$. We split $m_t$ into $\tilde m_t$ and $m_t-\tilde m_t$ such that the latter term is merely composed of noises $\nabla f(x_{\tau},\xi_{\tau})-\nabla F(x_{\tau})$ ($\tau<t$). While most of the procedure (Lemmas B.10-B.14) parallels the deterministic setting, there are two additional challenges due to the presence of noise:
\begin{itemize}
    \item \textit{Firstly}, since the gradients are not exact, the stochastic gradient we have access to is not guaranteed to be small even for the case of $t\in\overline{\mathcal S}$. Fortunately, the choice of parameters in Theorem \ref{theorem4} ($\rho = 5\sigma$) settles such difficulty.
    \item \textit{Secondly}, we need to deal with the noise in momentum, i.e. $m_t - \tilde m_t$. In particular, we use a recursive argument to obtain a good bound of $ \mathbb{E} \left\langle \nabla F(x_t),m_{t+1}-\tilde m_{t+1}\right\rangle$.
\end{itemize}
Finally, by choosing proper $\eta$ and $\gamma$, we can obtain Theorem \ref{theorem4}.

\subsection{Lower Bounds and Discussions}
\label{lower_bound_discussion}
Theorem \ref{theorem3} and \ref{theorem4} provide \textit{upper bounds} for the complexity of Algorithm \ref{FW}. Now we compare these results with existing lower bounds and discuss the tightness of our results.

\noindent{\bf Deterministic Setting.} \citet{carmon2019lower} have shown that there exists an $L$-smooth function $F$ such that any (possibly randomized) algorithm requires at least $\Omega \left( \Delta L\epsilon^{-2} \right)$ queries to gradient to ensure finding a point $x$ such that $\|\nabla F(x)\|\leq \epsilon$. Since Assumption \ref{ass:2} is weaker than $L$-smoothness ($L_0\leq L$), we have that  the lower bound for $(L_0,L_1)$-smooth functions is $\Omega \left( \Delta L_0\epsilon^{-2} \right)$.  From Theorem \ref{theorem3}, Algorithm \ref{FW} is \emph{optimal} since it can achieve the lower bound when ignoring numerical constants.

\noindent{\bf Stochastic Setting.} 
From the example constructed in \citet{drori2019complexity},  we have that  for any \textit{SGD} method with arbitrary (possibly adaptive) step sizes and aggregation schemes\footnote{See details in Appendix \ref{appendix_lower_bound}.},  the complexity lower bound  is exactly  $\Omega \left( \Delta L_0 \sigma^2 \epsilon^{-4} \right)$ for $(L_0,L_1)$-smooth functions that have   $\sigma$-bounded gradient noises.  Therefore Theorem \ref{theorem4} indicates that  clipped SGD \emph{matches} the lower bound.

One may ask : what is the lower bound for general stochastic gradient-based algorithms? In fact, from  the  example  in \citet{arjevani2019lower}, we have that
any algorithm needs $\Omega \left( \Delta L_0 \sigma^2_1 \epsilon^{-4}\right)$ stochastic gradient queries to find an $\epsilon$-approximate stationary point for a  hard $(L_0,L_1)$-smooth function whose  gradient noise has a $\sigma^2_1$-bounded variance.  It is our conjecture that the lower bound  for optimizing $(L_0,L_1)$-smooth functions that have  $\sigma$-bounded gradient noises  is also   $\mathcal{O} \left( \Delta L_0 \sigma^2 \epsilon^{-4} \right)$.  We leave  the study as a future work. 

\begin{figure}
  \centering
  \subfigure[CIFAR-10]{
  \label{fig_cifar}
  \includegraphics[width=0.31\textwidth]{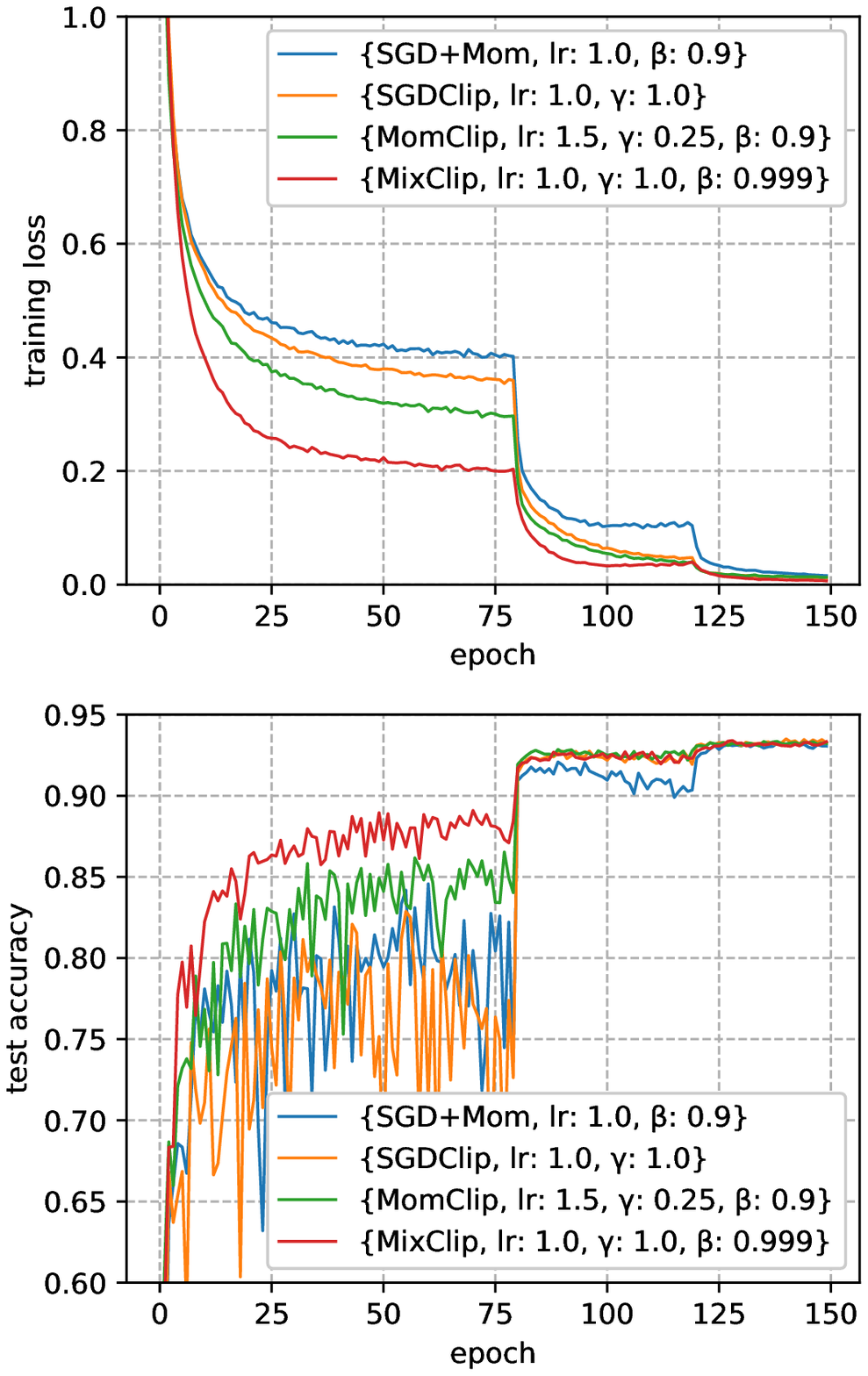}
  }
  \subfigure[ImageNet]{
  \label{fig_imagenet}
  \includegraphics[width=0.31\textwidth]{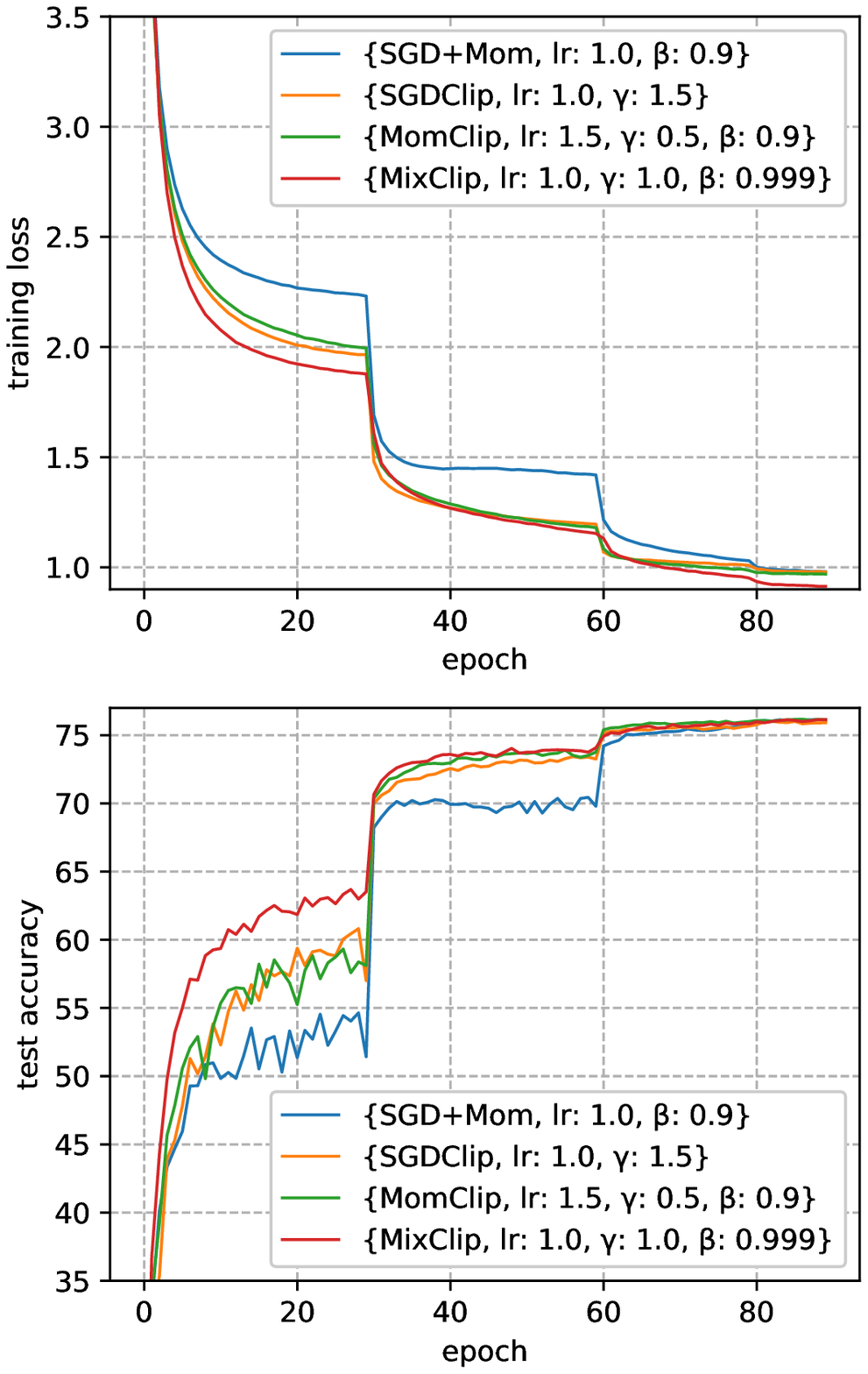}
  }
  \subfigure[PTB]{
  \label{fig_lstm}
  \includegraphics[width=0.31\textwidth]{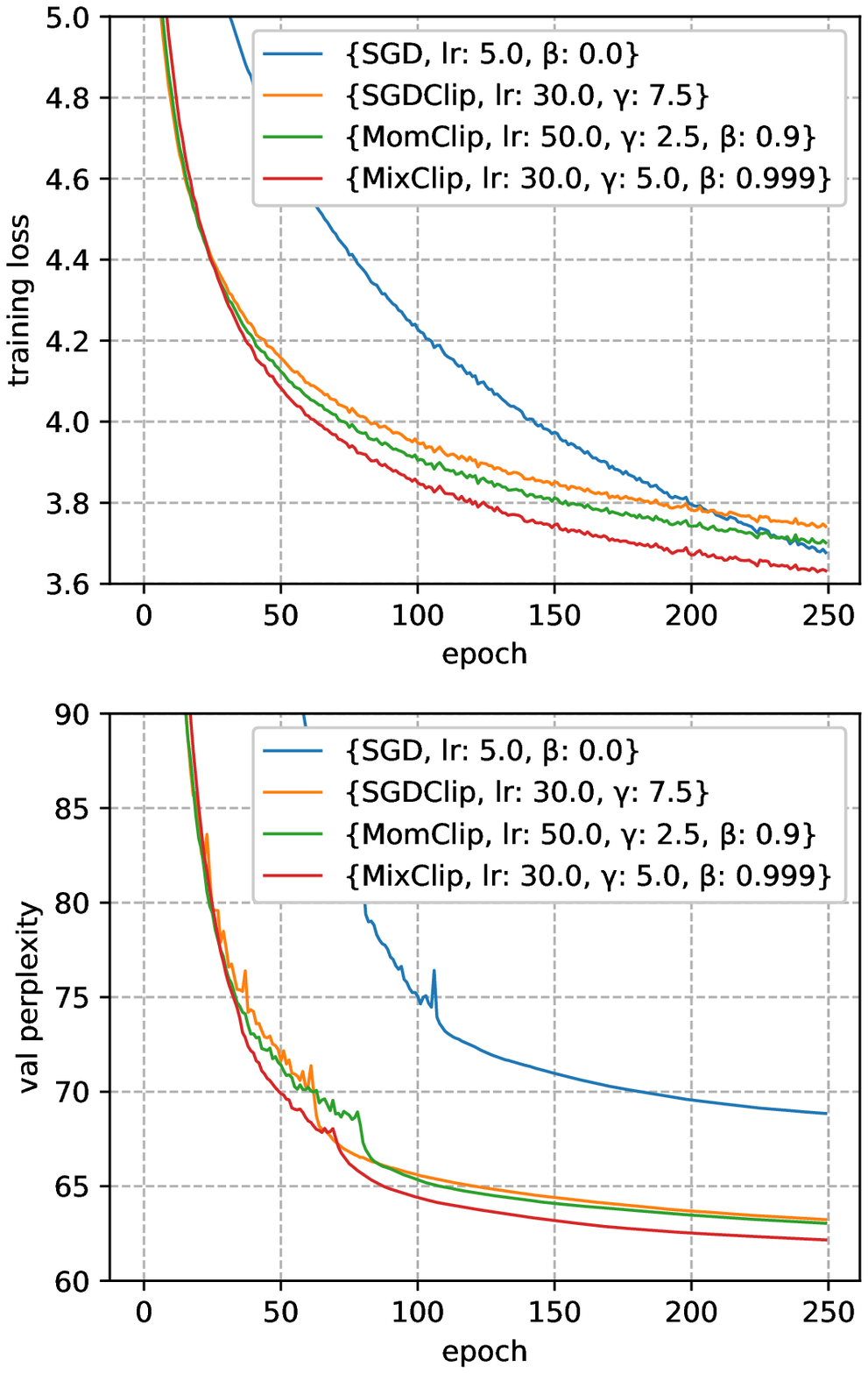}
  }
  \caption{Training loss curve and test accuracy/perplexity curve on CIFAR-10, ImageNet and PTB datasets. All clipping algorithms outperform SGD. Mixed clipping has the best training speed on these three datasets.}
  \label{fig_real_experiments}
\end{figure}

\section{Experiments}
\label{section_experiments}
We conduct extensive experiments and find the clipping algorithms indeed consistently outperform their unclipped counterpart. We present experimental results on three deep learning benchmarks: CIFAR-10 classification using ResNet-32, Imagenet classification using ResNet-50 and language modeling on Penn Treebank (PTB) dataset using AWD-LSTM. We put all the experimental details in the Appendix \ref{appendix_experiment}. Our code is available at \url{https://github.com/zbh2047/clipping-algorithms}. \par

\textbf{CIFAR-10 classification with ResNet.}
We train the standard ResNet-32 \citep{he2016deep} architecture on CIFAR-10. We use SGD with momentum for the baseline algorithm with a decaying learning rate schedule, which is the standard choice to train the ResNet architecture. We set learning rate $\eta=1.0$, momentum $\beta=0.9$ and minibatch size 128, following the common practice. For all the clipping algorithms, we choose the best $\eta$ and $\gamma$ based on a course grid search, while keeping other hyper-parameters and training strategy the same as SGD+momentum. We simply set the hyper-parameters $\nu=0.7$ and $\beta=0.999$ in mixed clipping, as suggested in \citet{ma2018quasi} (for its unclipped counterpart QHM). We run 5 times for each algorithm using different random seeds to make the results more reliable.\par
Figures \ref{fig_cifar} demonstrates the results. It can be seen that all the algorithms achieve a test accuracy more than 93\% on CIFAR-10. Note that all clipping algorithms converge faster than SGD+momentum. Particularly, the mixed clipping (Algorithm \ref{FW}) outperforms SGD+momentum by a large margin in term of training speed. As a result, one can possibly adopt a more aggressive learning rate decaying schedule to reduce training time considerably.\par

\textbf{ImagNet classification with ResNet.} We train the standard ResNet-50 \citep{he2016deep} architecture on ImageNet. For the baseline algorithm, we choose SGD with learning rate $lr=1.0$ and momentum $\beta=0.9$, following \citet{goyal2017accurate}. We use batch size 256 on 4 GPUs.\par
Figure \ref{fig_imagenet} plot the training loss curve and validation accuracy curve on ImageNet. All the algorithms reach a validation accuracy of about 76\%. However, all the clipping algorithms train faster than the baseline SGD. Mixed clipping performs the best among the four algorithms.

\textbf{Language modeling with LSTM.} We train the state-of-the-art AWD-LSTM \citep{merity2017regularizing} on Penn Treebank (PTB) dataset \citep{mikolov2010recurrent}. We first follow the training strategy in \citet{merity2017regularizing}, where they use averaged SGD without momentum with learning rate $\eta=30$ and clipping parameter $\gamma=7.5$. Since our purpose is to compare different algorithms rather than to achieve state-of-the-art results, we only train AWD-LSTM for 250 epochs. We then evaluate other algorithms including standard SGD without clipping, momentum clipping, and mixed clipping. We choose the best $\eta$ and $\gamma$ (using validation perplexity criterion) based on a course grid search. Results are shown in Figure \ref{fig_lstm}.\par
Figure \ref{fig_lstm} clearly shows all clipping methods  converge much faster than SGD without clipping, and are much better in term of validation perplexity. This is consistent with our theory, in that the vanilla SGD must use a very small learning rate to guarantee convergence \citep{zhang2019gradient}, which will be slow and be harmful to generalization on validation set according to previous works \citep{huang2017snapshot,kleinberg2018alternative} . Therefore clipping technique is crucial in LSTM models. We can also find that the training and test curve of mixed clipping is much better than both gradient clipping and momentum clipping. The mixed clipping improves validation perplexity for more than 1 point compared to clipped SGD after 250 epochs.

\textbf{Other experiments.} We also conduct experiments to compare clipping algorithms with Adam, and to directly compare our work with previous results \citep{zhang2019gradient} under the same setting. See Appendix \ref{appendix_experiment_other} for details. Finally, we construct a provably $(L_0,L_1)$-smooth optimization problem using MNIST dataset. We then run experiments in both deterministic setting and stochastic setting. The results are shown in Appendix \ref{appendix_experiment_mnist}.

\section{Conclusion}
\label{conclusion}
This paper proposes a detailed study for clipping methods under a general framework. In particular, we explore the possibility of combining clipping with other popular techniques, e.g. momentum acceleration, in deep learning. We provide a general and tight analysis for the framework, showing the efficiency of clipping methods in optimizing a class of non-convex and non-smooth (in traditional sense) functions. Experiments confirm that these methods have superior performance. We hope that our work affords more understandings  on the clipping technique and $(L_0, L_1)$ smooth functions.

There are still many open questions that have not yet been answered. Firstly, as discussed in Section \ref{lower_bound_discussion}, we are not aware of any lower bounds for general first-order methods that can be applied our setting. Thus, it is interesting to explore such lower bound, or to relax Assumption \ref{BN} to the more general bounded variance assumption. Secondly, although we have shown the superiority of clipping-based methods, we do not provide theoretical explanation why some clipping schemes are better than others as observed in experiments. We believe that this can only be done by exploring new and better smoothness assumptions. Thirdly, the empirical superiority of other adaptive methods ( e.g. AdaGrad \citep{duchi2011adaptive}, Adam \citep{kingma2014adam} ) have not been justified from a theoretical point of view. We hope that our analysis is helpful for the analysis of these methods. Finally, we are looking forward to seeing better optimization algorithms with better convergence properties in future work.

\section*{Acknowledgement}
This work was supported by National Key R\&D Program of China (2018YFB1402600), Key-Area Research and Development Program of Guangdong Province (No. 2019B121204008)] and Beijing Academy of Artificial Intelligence.

\bibliographystyle{plainnat}
\bibliography{reference}

\medskip
\newpage

\titleformat{\section}{\normalfont\large\bfseries}{Appendix \thesection}{1em}{}
\newenvironment{proof}{\paragraph{Proof:}}{\hfill$\square$}
\renewcommand{\thesection}{\Alph{section}}
\setcounter{section}{0}

\section{Properties of $(L_0,L_1)$-smooth functions}
\label{appendix_property}
In this section, we prove some important properties of $(L_0,L_1)$-smooth functions. These properties will be frequently used in subsequent sections.

We first present a basic lemma without proof.
\begin{lemma}
(Gronwall's inequality) \citep{gronwall1919note}  Let $I=[a,b]$ denote an interval of the real line with $a < b$. Let $f,g,h$ be continuous real-valued functions defined on $I$. Assume $g$ is non-decreasing, $h$ is non-negative, and the negative part of $g$ is integrable on every closed and bounded subinterval of $I$.If 
\begin{equation}
    f(t) \leq g(t)+ \int_{a}^{t} h(s) f(s)\mathrm{d} s, \quad \forall t \in I,
\end{equation}
then
\begin{equation}
    f(t) \leq g(t) \exp \left(\int_{a}^{t} h(s) \mathrm{d} s\right), \quad \forall t \in I.
\end{equation}
\end{lemma}

The following result, Lemma \ref{BasicProperty} , is a generalization of Lemma 9 in \citet{zhang2019gradient}.

\begin{lemma}
\label{BasicProperty}
Let $F$ be $(L_0,L_1)$-smooth, and $c>0$ be a constant. Given $x$, for any $x^{+}$ such that $\left\|x^{+}-x\right\| \leq c / L_{1}$, we have $\left\|\nabla f\left(x^{+}\right)\right\| \leq e^{c}\left(\frac{c L_{0}}{L_{1}}+\|\nabla F(x)\|\right)$.\par
\end{lemma}
\begin{proof}
Let $\gamma(t)$ be defined as $\gamma(t)=t(x^{+}-x)+x, t \in[0,1]$, then we have
$$
\nabla F(\gamma(t))=\int_{0}^{t} \nabla^2 F(\gamma(\tau))\left(x^{+}-x\right) \mathrm{d} \tau+\nabla F(\gamma(0))
$$
We then bound the norm of $\nabla F(\gamma(t))$:
\begin{align}
\|\nabla F(\gamma(t))\| & \leq\int_{0}^{t} \|\nabla^2 F(\gamma(\tau))\left(x^{+}-x\right)\| \mathrm{d}\tau+\|\nabla F(\gamma(0))\| \\
& \leq\left\|x^{+}-x\right\| \int_{0}^{t}\left\|\nabla^{2} F(\gamma(\tau))\right\| \mathrm{d} \tau+\|\nabla F(x)\| \\
& \leq \frac{c}{L_{1}} \int_{0}^{t}\left(L_{0}+L_{1}\|\nabla F(\gamma(\tau))\|\right) \mathrm{d} \tau+\|\nabla F(x)\|
\end{align}
The first inequality uses the triangular inequality of 2-norm; The second inequality uses the property of spectral norm; The third inequality uses the definition of $(L_0,L_1)$-smoothness. By applying the Gronwall’s inequality we get
\begin{equation}
    \label{eq:first_lemma}
    \|\nabla F(\gamma(t))\| \leq \left(\frac{L_{0}}{L_{1}}ct+\|\nabla F(x)\|\right)\exp (ct)
\end{equation}
The Lemma follows by setting $t=1$.
\end{proof}

Now we are able to prove a \textit{descent inequality}, which is similar to the descent inequality for $L$-smooth functions. In fact, if a function $F$ is $L$-smooth, it is well-known that for any $x,y$, we have
\begin{equation}
    F(y) \leq F(x) + \left\langle \nabla F(x),y-x \right\rangle + \dfrac{L}{2} \|y-x\|^2 \notag
\end{equation}
\begin{lemma}
\label{DesIneq}
(Descent Inequality) Let $F$ be $(L_0,L_1)$-smooth, and $c>0$ be a constant. For any $x_k$ and $x_{k+1}$, as long as $\|x_k - x_{k+1}\| \le c/L_1$, we have
\begin{equation}
    F\left(x_{k+1}\right) \leq F\left(x_{k}\right)+\left\langle\nabla F\left(x_{k}\right), x_{k+1}-x_{k}\right\rangle+\frac{A L_{0}+B L_{1}\left\|\nabla F\left(x_{k}\right)\right\|}{2}\left\|x_{k+1}-x_{k}\right\|^{2}
\end{equation}
where $A=1+e^c-\frac {e^c-1} c,B=\frac {e^c-1} c$.\par
\end{lemma}
\begin{proof}
Let $\gamma(t)$ be defined as $\gamma(t)=t(x_{k+1}-x_k)+x_k, t \in[0,1]$. The following derivation uses Taylor's theorem (in \eqref{eq:taylor}), then uses triangular inequality, Cauchy-Schwarz inequality and the property of spectral norm (in \eqref{eq:next_derivation}):
\begin{align}
    \label{eq:taylor}
    F\left(x_{k+1}\right) &\leq F\left(x_{k}\right)+\left\langle\nabla F\left(x_{k}\right), x_{k+1}-x_{k}\right\rangle + \int_{0}^{1} (x_{k+1}-x_k)^T\nabla^{2} F(\gamma(t))(x_{k+1}-\gamma(t)) \mathrm{d} t\\
    \label{eq:next_derivation}
    &\leq F\left(x_{k}\right)+\left\langle\nabla F\left(x_{k}\right), x_{k+1}-x_{k}\right\rangle +  \int_{0}^{1} \|(x_{k+1}-x_k)\|\|\nabla^{2} F(\gamma(t))\|\|x_{k+1}-\gamma(t)\| \mathrm{d} t\\
    \label{eq:next2}
    &=F\left(x_{k}\right)+\left\langle\nabla F\left(x_{k}\right), x_{k+1}-x_{k}\right\rangle + \frac{\left\|x_{k+1}-x_{k}\right\|^{2}}{2} \int_{0}^{1}\left\|\nabla^{2} F(\gamma(t))\right\| \mathrm{d} t
\end{align}
Then we use $(L_0,L_1)$-smoothness and \eqref{eq:first_lemma} to bound $\left\|\nabla^{2} F(\gamma(t))\right\|$:
\begin{equation}
\begin{aligned}
    \left\|\nabla^{2} F(\gamma(t))\right\| &\le L_0 + L_1 \left\|\nabla F(\gamma(t))\right\|\\
    &\le L_0 + L_1  \left(\frac{L_{0}}{L_{1}}ct+\|\nabla F(x_k)\|\right)\exp (ct)
\end{aligned}
\end{equation}
Taking integration we get
\begin{equation}
    \label{eq:final}
    \int_0^1 \left\|\nabla^{2} F(\gamma(t))\right\| \mathrm d t\le L_0\left(1+e^c-\frac {e^c-1} c\right) + \frac {e^c-1} c L_1 \|\nabla F(x_k)\|
\end{equation}
Substituting \eqref{eq:final} into \eqref{eq:next2} concludes the proof.
\end{proof}

\begin{corollary}
\label{DesCor}
Let $F$ be $(L_0,L_1)$-smooth, and $c>0$ be a constant. For any $x_k$ and $x_{k+1}$, as long as $\|x_k - x_{k+1}\| \le c/L_1$, we have
\begin{equation}
    \|\nabla F(x_{k+1})-\nabla F(x_k)\|\le ({A L_{0}+B L_{1}\left\|\nabla F\left(x_{k}\right)\right\|})\left\|x_{k+1}-x_{k}\right\|
\end{equation}
where $A=1+e^c-\frac {e^c-1} c,B=\frac {e^c-1} c$.
\end{corollary}
\begin{proof}
\begin{equation}
\begin{aligned}
\|\nabla F(x_{k})-\nabla F(x_k-1)\|
&=\left\|\int_{0}^{1} \nabla^2 F(t x_{k-1} + (1-t) x_k)(x_k-x_{k-1})\mathrm d t\right\|\\
&\le \int_{0}^{1} \|\nabla^2 F(t x_{k-1} + (1-t x_k)\|\|x_k-x_{k-1}\|\mathrm d t\\
\end{aligned}
\end{equation}
Using \eqref{eq:final} leads to the results.
\end{proof}

Finally we prove a result which provides a way to upper-bound the gradient norm. A similar result for $L$-smooth functions is the following: if $F$ is $L$-smooth, then for any $x$, we have
\begin{equation}
    \|\nabla F(x)\|^2 \leq 2L \left( F(x)- \inf_{y \in \mathbb{R}^d} F(y) \right) \notag
\end{equation}
\begin{lemma}
\label{GradNormBd}
(Bounding the gradient norm) Let $F(x)$ be an $(L_0,L_1)$-smooth function, and $F^*$ be the optimal value. Then for any $x_0$, we have
\begin{equation}
    \min\left(\frac {\|\nabla F(x_0)\|} {L_1}, \frac {\|\nabla F(x_0)\|^2} {L_0}\right)\le 8(F(x_0)-F^*)
\end{equation}
\end{lemma}
\begin{proof}
Define the constant $c=\frac {L_1\|\nabla F(x_0)\|}{AL_0+BL_1 \|\nabla F(x_0)\|}$ and $A=1+e^c-\frac {e^c-1} c,B=\frac {e^c-1} c$. It is easy to see that such $0\le c< 1$ exists. Let $\lambda = \frac 1 {AL_0+BL_1 \|\nabla F(x_0)\|}$ and $x=x_0-\lambda \nabla F(x_0)$. Then $\|x-x_0\|\le c/L_1$. By the descent inequality we have
\begin{equation}
\begin{aligned}
    F^*\le F(x)&\le F(x_0)-\lambda \|\nabla F(x_0)\|^2+\frac {AL_0+BL_1\|\nabla F(x_0)\|} 2 \lambda^2\|\nabla F(x_0)\|^2\\
    &=F(x_0)-\frac 1 2 \lambda \|\nabla F(x_0)\|^2
\end{aligned}
\end{equation}
If $\|\nabla F(x)\|\ge \frac {AL_0} {BL_1}$, then
\begin{equation}
\begin{aligned}
    F(x_0)-F^*\ge \frac{\|\nabla F(x_0)\|} {2\left(\frac {AL_0} {\|\nabla F(x_0)\|}+BL_1\right)} \ge \frac {\|\nabla F(x_0)\|} {4BL_1}\ge\frac {\|\nabla F(x_0)\|} {8L_1}
\end{aligned}
\end{equation}
If $\|\nabla F(x)\|< \frac {AL_0} {BL_1}$, then
\begin{equation}
\begin{aligned}
    F(x_0)-F^*\ge \frac 1{2}\lambda {\|\nabla F(x_0)\|^2}\ge \frac{\|\nabla F(x_0)\|^2} {4AL_0}\ge  \frac {\|\nabla F(x_0)\|^2} {8L_0}
\end{aligned}
\end{equation}
\end{proof}

\subsection{Relaxation of $(L_0,L_1)$-smoothness (Remark 2.3)}
The original definition of $(L_0,L_1)$-smoothness requires the function to be twice-differentiable. Under this definition, $(L_0,L_1)$-smoothness is actually \textit{not} weaker than $L$-smoothness, which only requires the function to be continuous differentiable. In this section we prove that the alternative definition provided in Remark 2.3 is sufficient for all the results in this paper. 

Now, suppose that there exists $K_0,K_1 >0$ such that for all $x, y \in \mathbb{R}^d$, if $\|x-y\| \leq \frac{1}{K_1}$, then
\begin{equation}
\label{new_ass}
   \|\nabla F(x) - \nabla F(y) \| \leq (K_0 + K_1 \|\nabla F(y)\|) \|x-y\| 
\end{equation}
We check that Lemma \ref{BasicProperty} and \ref{DesIneq} still holds under the new assumption (with $L_0,L_1$ replaced by $K_0,K_1$, up to numerical constants)
We immediately obtain from \eqref{new_ass} above that
\begin{equation}
    \|\nabla F(x)\| \leq 2\|\nabla F(y)\| + \frac{K_0}{K_1}
\end{equation}
which is of the same form as Lemma \ref{BasicProperty}. Next, we have
\begin{equation}\begin{aligned}
&\quad F(y) - F(x) - \left\langle y-x, \nabla F(x) \right\rangle \\
&= \int_{0}^{1} \left\langle \nabla F \left(\theta y+(1-\theta )x \right) - \nabla F(x), x-y \right\rangle \text{d}\theta \\
&\leq \int_{0}^{1} \left( K_0\theta\|x-y\|^{2} + K_1\theta \|x-y\|^2 \|\nabla F(x)\| \right) \text{d}\theta \\
&\leq \dfrac{K_0+K_1\|\nabla F(x)\|}{2}\|x-y\|^{2}
\end{aligned}\end{equation}
which is of the same form as Lemma \ref{DesIneq}.

Since all the other results are established on the basis of these two lemmas, we can see that the conclusion still holds under \eqref{new_ass}.

\section{Proof of Theorems}
We first prove the deterministic case (Theorem 3.1), then generalize the result to stochastic case (Theorem 3.2). In deterministic case we can use fewer notations, which will make the proof more readable and elegant. The proof in stochastic case will rely on all the techniques used in the deterministic case,  as well as some new methods.\par
\subsection{Proof of Theorem 3.1}
To simplify the notation, we write the update formula as
\begin{equation}
    \begin{aligned}
        m^+&=\beta m + (1-\beta) \nabla F(x) \\
        x^+&=x - \left( \nu \min\left(\eta, \frac {\gamma} {\|m^+\|} \right)m^+ + (1-\nu) \min\left(\eta, \frac {\gamma} {\|\nabla F(x)\|} \right)\nabla F(x) \right)
    \end{aligned}
\end{equation}
when analyzing a single iteration. The error between $m^+$ and $\nabla F(x)$ is denoted as $\delta = m^+ - \nabla F(x)$. Suppose $\gamma\le c/L_1$ for some constant $c$, and we denote $A=1+e^c-\frac {e^c-1} c$ and $B=\frac {e^c-1} c$, just the same as in the descent inequality (Lemma \ref{DesIneq}).
\begin{lemma}
\label{lemma_mom_clip_1}
Let $\mu\ge 0$ be a real constant. For any vector $u$ and $v$,
\begin{align}
     \label{lemma_mom_clip_1_1}
    -\frac {\left\langle u, v\right\rangle} {\|v\|}\le -\mu\|u\|-(1-\mu)\| v\|+(1+\mu)\|v-u\|
\end{align}
\end{lemma}
\begin{proof}
\begin{equation*}
\begin{aligned}
    -\frac {\left\langle u, v\right\rangle} {\|v\|}
    &=- \|v\| + \frac {\left\langle v-u, v\right\rangle} {\|v\|}\\
    &\le-\|v\|+\|v-u\|\\
    &\le -\|v\| + \|v-u\|+\mu(\|v-u\|+\|v\|-\|u\|)\\
    &= -\mu \|u\| -(1-\mu) \|v\| + (1+\mu) \|v-u\|
\end{aligned}
\end{equation*}
\end{proof}

To prove the theorem, we will construct an \textit{Lyapunov function} and explore the decreasing property of this function. We define the Lyapunov function $G(x,m)$ to be
\begin{equation}
    G(x,m)=F(x)+\frac {\nu\beta} {2(1-\beta)} \min\left({\eta} \|m\|^2,{\gamma}\|m\|\right)
\end{equation}
and analyze $G(x^+,m^+)-G(x,m)$. We first bound $\min\left( {\eta} \|m^+\|^2,{\gamma} \|m^+\|\right)-\min\left( {\eta}  \|m\|^2,{\gamma} \|m\|\right)$.

\begin{lemma}
\label{lemma_mom_clip_2}
For any momentum vectors $m$ and $m^+=\beta m+(1-\beta)\nabla F(x)$, let $\delta=m^+-\nabla F(x)$, then
\begin{equation}
    \min\left( {\eta} \|m^+\|^2,{\gamma} \|m^+\|\right)-\min\left( {\eta}  \|m\|^2,{\gamma} \|m\|\right)\le \frac {2(1-\beta)}{\beta}\gamma \|\delta\|
\end{equation}
\end{lemma}
\begin{proof}
Consider the following three cases:
\begin{itemize}
    \item $\|m\|\ge \gamma/\eta$. In this case 
    \begin{equation*}
    \begin{aligned}
        \min\left( {\eta}  \|m^+\|^2, {\gamma} \|m^+\|\right)-\min\left( {\eta}  \|m\|^2, {\gamma} \|m\|\right)
        &\le \gamma \|m^+\|-\gamma \|m\|\\
        &\le \gamma\|m^+-m\|\\
        &=\frac {1-\beta}{\beta}\gamma \|\delta\|
    \end{aligned}
    \end{equation*}
    \item $\|m\|< \gamma/\eta$ and $\|m^+\|< \gamma/\eta$. In this case 
    \begin{equation*}
    \begin{aligned}
        \min\left( {\eta}  \|m^+\|^2, {\gamma} \|m^+\|\right)-\min\left( {\eta}  \|m\|^2, {\gamma} \|m\|\right)
        &= \eta \|m^+\|^2-\eta \|m\|^2\\
        &= \eta(\|m^+\|-\|m\|)(\|m^+\|+\|m\|)\\
        &\le\frac {2(1-\beta)}{\beta}\gamma \|\delta\|
    \end{aligned}
    \end{equation*}
    \item $\|m\|< \gamma/\eta$ and $\|m^+\|> \gamma/\eta$. In this case
    \begin{equation*}
    \begin{aligned}
        \min\left( {\eta}  \|m^+\|^2, {\gamma} \|m^+\|\right)-\min\left( {\eta}  \|m\|^2, {\gamma} \|m\|\right)
        &= \gamma \|m^+\|-\eta \|m\|^2\\
        &\le \gamma \|m^+\|-\left[2\gamma\|m\|-\frac {\gamma^2} {\eta}\right]\\
        &\le \gamma \|m^+\|-2\gamma\|m\|+\gamma \|m^+\|\\
        &= 2\gamma (\|m^+\|-\|m\|)\le \frac {2(1-\beta)}{\beta}\gamma \|\delta\|
    \end{aligned}
    \end{equation*}
\end{itemize}
Thus in all cases $\min\left( {\eta}  \|m^+\|^2, {\gamma} \|m^+\|\right)-\min\left( {\eta}  \|m\|^2, {\gamma} \|m\|\right)$ can be upper bounded by $ \frac {2(1-\beta)}{\beta}\gamma \|\delta\|$.
\end{proof}

\begin{lemma}
\label{lemma_mom_clip_deterministic_large_basic}
Suppose $\max(\|\nabla F(x)\|,\|m^+\|,\|m\|)\ge \gamma / {\eta}$. Then
\begin{equation}
\begin{aligned}
\label{lemma_mom_clip_deterministic_basic_ineq}
    G(x^+,m^+)-G(x,m) \le -\frac{2}{5} \gamma \|\nabla F(x)\| - \frac{3}{5}\frac{\gamma^2}{\eta}+\frac{12}{5\beta}\gamma\|\delta\| + \frac{AL_0+BL_1 \|\nabla F(x)\|}{2}\gamma^2 
\end{aligned}
\end{equation}
\end{lemma}
\begin{proof}
We first write $G(x^+,m^+)-G(x,m)$ as
\begin{equation}
\begin{aligned}
    &\quad G(x^+,m^+)-G(x,m)\\
    \label{lemma_mom_clip_deterministic_large_basic_1}
    &=\left(F(x^+)-F(x)\right)+\frac{\nu\beta}{2(1-\beta)}\left[\min\left( {\eta} \|m^+\|^2,{\gamma} \|m^+\|\right)-\min\left( {\eta}  \|m\|^2,{\gamma} \|m\|\right)\right]
\end{aligned}
\end{equation}
Based on Lemma \ref{lemma_mom_clip_2}, we only need to bound $F(x^+)-F(x)$. We will use the $(L_0,L_1)$-smoothness assumption. 
\begin{equation}\begin{aligned}
&\quad F(x^+)-F(x) \\
&\leq \left\langle x^+ - x, \nabla F(x) \right\rangle + \frac{AL_0+BL_1\|\nabla F(x)\|}{2}\|x^{+}-x\|^2 \\
&= - \left[ \nu\min \left( \eta, \frac{\gamma}{\|m^{+}\|} \right)\left\langle m^+, \nabla F(x) \right\rangle  + (1-\nu)\min \left(\eta, \frac{\gamma}{\|\nabla F(x)\|}\right) \left\langle \nabla F(x),\nabla F(x) \right\rangle \right]\\ 
&+ \frac{AL_0+BL_1 \|\nabla F(x)\|}{2}\gamma^2 \\
&\leq \nu\left[ -\frac{2}{5} \gamma \|\nabla F(x)\| - \frac{3}{5}\frac{\gamma^2}{\eta}+\left(\frac{12}{5\beta}-1\right)\gamma\|\delta\| \right] + (1-\nu) \left( -\dfrac{2}{5}\gamma \|\nabla F(x)\| - \frac{3}{5}\frac{\gamma^2}{\eta} + \dfrac{8}{5\beta}\gamma\|\delta\| \right) \\
&+ \frac{AL_0+BL_1 \|\nabla F(x)\|}{2}\gamma^2 \label{des}
\end{aligned}\end{equation}
Where the first inequality uses the descent inequality (Lemma \ref{DesIneq}), the second equation follows from the update rule, and the last inequality is obtained by the following two inequalities:
\begin{align}
\label{MomIneq}
    -\min \left( \eta, \frac{\gamma}{\|m^{+}\|} \right)\left\langle m^+, \nabla F(x) \right\rangle \leq -\frac{2}{5} \gamma \|\nabla F(x)\| - \frac{3}{5}\frac{\gamma^2}{\eta}+\left(\frac{12}{5\beta}-1\right)\gamma\|\delta\|\\
\label{GradIneq}
    -\min \left(\eta, \frac{\gamma}{\|\nabla F(x)\|}\right) \|\nabla F(x)\|^2 \leq -\dfrac{2}{5}\gamma \|\nabla F(x)\| - \frac{3}{5}\frac{\gamma^2}{\eta} + \dfrac{8}{5\beta}\gamma\|\delta\|
\end{align}
First we prove that \eqref{MomIneq} holds by considering the following three cases:
\begin{itemize}
    \item $\|m^+\|\ge \gamma/\eta$. In this case the algorithm performs a normalized update. Then \eqref{MomIneq} follows by directly using Lemma \ref{lemma_mom_clip_1} with $\mu=2/5$:
    \begin{equation*}
    \begin{aligned}
     -\min \left( \eta, \frac{\gamma}{\|m^{+}\|} \right)\left\langle m^+, \nabla F(x) \right\rangle&=-\left\langle \nabla F(x), \frac {\gamma m^+} {\|m^+\|}\right\rangle\\
     &\le-\frac 2 5\gamma \|\nabla F(x)\|-\frac 3 5\gamma \|m^+\|+\frac 7 5\gamma \|\delta\|
    \end{aligned}
    \end{equation*}
    \item $\|m^+\|< \gamma/\eta$ and $\|\nabla F(x)\|\ge \gamma/\eta$. In this case the algorithm performs an unnormalized update. We now prove $-\eta\left\langle \nabla F(x),m^+\right\rangle\le -\frac 2 5 \gamma\|\nabla F(x)\|-\frac {3\gamma^2} {5\eta} +\frac 7 5 \gamma\|\nabla F(x)-m^+\|$.
    \begin{equation*}
    \begin{aligned}
        &\quad\eta\left\langle \nabla F(x),m^+\right\rangle -\frac 2 5 \gamma\|\nabla F(x)\|-\frac {3\gamma^2} {5\eta} +\frac 7 5 \gamma\|\nabla F(x)-m^+\|\\
        &\ge \eta\left\langle \nabla F(x),m^+\right\rangle -\frac 2 5 \gamma\|\nabla F(x)\|-\frac {3\gamma^2} {5\eta} +\frac 7 5 \gamma\left(\|\nabla F(x)\|-\frac{\left\langle \nabla F(x),m^+\right\rangle}{\|\nabla F(x)\|}\right)\\
        &=\|\nabla F(x)\|\left(\gamma+\eta\frac{\left\langle \nabla F(x),m^+\right\rangle}{\|\nabla F(x)\|}\right)-\frac 7 5\gamma \frac{\left\langle \nabla F(x),m^+\right\rangle}{\|\nabla F(x)\|}-\frac {3\gamma^2} {5\eta}\\
        &\ge \frac{\gamma^2}{\eta}+\gamma \frac{\left\langle \nabla F(x),m^+\right\rangle}{\|\nabla F(x)\|}-\frac 7 5\gamma \frac{\left\langle \nabla F(x),m^+\right\rangle}{\|\nabla F(x)\|}-\frac {3\gamma^2} {5\eta}\\
        &\ge \frac {2\gamma^2}{5\eta}-\frac 2 5 \gamma\|m^+\|\ge 0
    \end{aligned}
    \end{equation*}
    \item $\|m^+\|< \gamma/\eta$ and $\|\nabla F(x)\|<  \gamma/\eta$. This is the most complicated case. Due to the condition in Lemma \ref{lemma_mom_clip_deterministic_large_basic}, $\|m\|\ge  \gamma/\eta $. In this case, the algorithm also performs an unnormalized update. We first bound $\eta\left\langle \nabla F(x),m\right\rangle$ using the same calculation as in the second case:
    \begin{equation*}
    \begin{aligned}
        -\eta\left\langle \nabla F(x),m\right\rangle &\le -\frac 2 5 \gamma\|m\|-\frac {3\gamma^2} {5\eta} +\frac 7 5 \gamma\|\nabla F(x)-m\|\\
        &\le -\frac 2 5 \gamma\|\nabla F(x)\|-\frac {3\gamma^2} {5\eta} +\frac 7 {5\beta} \gamma\|\delta\|
    \end{aligned}
    \end{equation*}
    where we use the fact that $\|\nabla F(x)-m\| = \|\delta\|/\beta$. We then bound $\eta\|\nabla F(x)\|^2$ as follows:
    \begin{equation*}
    \begin{aligned}
        -\eta\|\nabla F(x)\|^2
         &\le -2\gamma \|\nabla F(x)\|+\frac {\gamma^2}{\eta}\\
         &= -\frac 2 5\gamma \|\nabla F(x)\|-\frac {3\gamma^2}{5\eta}+\frac 8 5\left(\frac {\gamma^2}{\eta}-\gamma\|\nabla F(x)\|\right)\\
         &\le -\frac 2 5\gamma \|\nabla F(x)\|-\frac {3\gamma^2}{5\eta}+\frac 8 5\gamma\left(\|m\|-\|\nabla F(x)\|\right)\\
         &\le -\frac 2 5\gamma \|\nabla F(x)\|-\frac {3\gamma^2}{5\eta}+\frac 8 {5\beta}\|\delta\|
    \end{aligned}
    \end{equation*}
    Combining the two inequalities, we obtain
    \begin{equation*}
    \begin{aligned}
         -\left\langle \nabla F(x), \eta m^+\right\rangle 
        &=-\eta\left\langle \nabla F(x), \beta m + (1-\beta) \nabla F(x)\right\rangle \\
        &\le -\frac 2 5\gamma \|\nabla F(x)\|-\frac {3\gamma^2}{5\eta}+\left(\frac 7 {5\beta}\beta+\frac 8 {5\beta}(1-\beta)\right)\|\delta\|\\
        &\le -\frac 2 5\gamma \|\nabla F(x)\|-\frac {3\gamma^2}{5\eta}+\left(\frac {12} {5\beta}-1\right)\|\delta\|\\
    \end{aligned}
    \end{equation*}
\end{itemize}
Thus in all cases \eqref{MomIneq} holds. We now turn to \eqref{GradIneq} which is proven in a similar fashion. Specifically, consider the following three cases:
\begin{itemize}
    \item $\|\nabla F(x)\|\ge \gamma/\eta$. In this case
    $$ -\min \left(\eta, \frac{\gamma}{\|\nabla F(x)\|}\right)\|\nabla F(x)\|^2=-\gamma\|\nabla F(x)\|^2\le -\frac 2 5\gamma \|\nabla F(x)\|-\frac {3\gamma^2}{5\eta}$$
    \item $\|\nabla F(x)\|< \gamma/\eta$ and $\|m^+\|\ge \gamma/\eta$. In this case bound $\eta\|\nabla F(x)\|^2$ the same as in the third case of \eqref{MomIneq}:
    \begin{equation*}
    \begin{aligned}
        -\min \left(\eta, \frac{\gamma}{\|\nabla F(x)\|}\right)\|\nabla F(x)\|^2
         &= -\eta \|\nabla F(x)\|^2\\
         &\le -\frac 2 5\gamma \|\nabla F(x)\|-\frac {3\gamma^2}{5\eta}+\frac 8 5\left(\frac {\gamma^2}{\eta}-\gamma\|\nabla F(x)\|\right)\\
         &\le -\frac 2 5\gamma \|\nabla F(x)\|-\frac {3\gamma^2}{5\eta}+\frac 8 5\gamma\left(\|m^+\|-\|\nabla F(x)\|\right)\\
         &\le -\frac 2 5\gamma \|\nabla F(x)\|-\frac {3\gamma^2}{5\eta}+\frac 8 5 \gamma\|\delta\|
    \end{aligned}
    \end{equation*}
    \item $\|\nabla F(x)\|< \gamma/\eta$ and $\|m^+\|< \gamma/\eta$. In this case $\|m\|\ge \gamma/\eta$. Using the same calculation above, 
    \begin{equation*}
    \begin{aligned}
        -\min \left(\eta, \frac{\gamma}{\|\nabla F(x)\|}\right)\|\nabla F(x)\|^2
         &\le -\frac 2 5\gamma \|\nabla F(x)\|-\frac {3\gamma^2}{5\eta}+\frac 8 {5\beta}\gamma\|\delta\|
    \end{aligned}
    \end{equation*}
\end{itemize}
Thus \eqref{GradIneq} holds. Merging all the cases above, we finally obtain
\begin{equation}\begin{aligned}
 G(x^+,m^+)-G(x,m) \le\left[ -\frac{2}{5} \gamma \|\nabla F(x)\| - \frac{3}{5}\frac{\gamma^2}{\eta}+\frac{12}{5\beta}\gamma\|\delta\| \right]+ \frac{AL_0+BL_1 \|\nabla F(x)\|}{2}\gamma^2 
\end{aligned}\end{equation}
\end{proof}

Now we consider all the steps $t$ which satisfy the condition in Lemma \ref{lemma_mom_clip_deterministic_large_basic}, denoted as $ \mathcal  S=\{t\in [0,T-1]:\max(\|F(x_t)\|,\|m_{t+1}\|,\|m_t\|)\ge \gamma / {\eta}\}$. Similarly, use $\overline{\mathcal  S}=[0,T-1]\backslash\mathcal  S$. Let $T_{ \mathcal S}=| \mathcal S|$, then $T-T_{ \mathcal S}=|\overline{ \mathcal S}|$.
\begin{corollary}
\label{lemma_mom_clip_deterministic_large_sum}
Let set $ \mathcal S$ and $T_{ \mathcal S}$ be defined above. Then
\begin{equation}
\label{lemma_mom_clip_deterministic_large_sum_0}
\begin{aligned}
    &\sum_{t\in \mathcal S}G(x_{t+1},m_{t+1})-G(x_t,m_t)\\
    &\le \frac {12\gamma} {5\beta(1-\beta)}\|\delta_0\|+\left(\frac {12} {5(1-\beta)}AL_0+\frac {12\gamma}{5\eta(1-\beta)}BL_1 + \frac 1 2 AL_0\right)\gamma^2 T_{\mathcal{S}}+\\
        &\qquad\gamma\sum_{t \in \mathcal S} \left[-\frac 1 {5}(2\|\nabla F(x_t)\|+3\frac {\gamma}{\eta})+\frac {\gamma} 2 BL_1 \|\nabla F(x_t)\| +\frac {12\gamma} {5(1-\beta)} BL_1 \|\nabla F(x_t)\|\right]
\end{aligned}
\end{equation}
\end{corollary}
\begin{proof}
Using Lemma \ref{lemma_mom_clip_deterministic_large_basic},
\begin{equation}
\begin{aligned}
    \label{lemma_mom_clip_deterministic_large_sum_1}
    &\sum_{t\in \mathcal S}G(x_{t+1},m_{t+1})-G(x_t,m_t)\\
    &\le -\sum_{t\in  \mathcal S}\left[\frac {\gamma} {5} (2\|\nabla F(x_t)\|+3\frac {\gamma}{\eta})-\frac {\gamma^2} 2 BL_1 \|\nabla F(x_t)\|-\frac {12} {5\beta}\gamma \|\delta_t\|\right]+\frac {\gamma^2} 2 AL_0T_{ \mathcal S}
\end{aligned}
\end{equation}
We now focus on the summation of the term $\|\delta_t\|$. Define $S(a,b)=\nabla F(a)-\nabla F(b)$. When $\|a-b\|\le \gamma$, $\|S(a, b)\|\le  \gamma\left(A L_{0}+B L_{1}\|\nabla F(b)\|\right)$ (see Lemma \ref{DesCor}). Thus we can expand $\delta_t=m_{t+1}-\nabla F(x_t)$ using the recursive relation $\delta_t=\beta\delta_{t-1}+\beta S(x_{t-1},x_t)$ as follows
\begin{equation*}
\begin{aligned}
    \sum_{t\in \mathcal S} \|\delta_t\|
    &= \sum_{t\in \mathcal S} \left\|\beta^{t} \delta_{0} +\beta \sum_{\tau=0}^{t-1}\beta^{\tau} S\left(x_{t-\tau-1}, x_{t-\tau}\right)\right\|\\
    &\le \sum_{t\in \mathcal S} \beta^{t} \|\delta_{0}\| +\beta \sum_{t\in \mathcal S} \sum_{\tau=0}^{t-1}\beta^{\tau} \gamma (AL_0+BL_1\|\nabla F(x_{t-\tau})\|)\\
    &\le \frac 1 {1-\beta} \|\delta_{0}\| +\frac {\beta} {1-\beta}(AL_0\gamma T_{\mathcal{S}})+ \\
        &\qquad  BL_1\gamma\sum_{t\in \mathcal{S}} \left(\sum_{\tau\in[1,t]\backslash \mathcal{S}} \beta^{t-\tau+1}\|\nabla F(x_{\tau})\|+\sum_{\tau\in[1,t]\cap \mathcal{S}} \beta^{t-\tau+1}\|\nabla F(x_{\tau})\|\right)\\
    &\le \frac {\beta} {1-\beta}\left(\frac {\|\delta_0\|}{\beta}+AL_0\gamma T_{\mathcal{S}}+BL_1\frac {\gamma^2}{\eta}T_{\mathcal S}+BL_1\gamma \sum_{t\in \mathcal S}\|\nabla F(x_{t})\|\right)
\end{aligned}
\end{equation*}
where the last inequality uses the fact that $\|\nabla F(x_\tau)\|\le \gamma/\eta$ for all $\tau \in [1,t]\backslash\mathcal{S}$.\par
After substituting the above results into \eqref{lemma_mom_clip_deterministic_large_sum_1} we obtain
\begin{equation}
\begin{aligned}
    &\sum_{t\in \mathcal S}G(x_{t+1},m_{t+1})-G(x_t,m_t)\\
    &\le \frac {12\gamma} {5\beta(1-\beta)}\|\delta_0\|+\left(\frac {12} {5(1-\beta)}AL_0+\frac {12\gamma}{5\eta(1-\beta)}BL_1 + \frac 1 2 AL_0\right)\gamma^2 T_{\mathcal{S}}+\\
        &\qquad\gamma\sum_{t \in \mathcal S} \left[-\frac 1 {5}(2\|\nabla F(x_t)\|+3\frac {\gamma}{\eta})+\frac {\gamma} 2 BL_1 \|\nabla F(x_t)\| +\frac {12\gamma} {5(1-\beta)} BL_1 \|\nabla F(x_t)\|\right]
\end{aligned}
\end{equation}
\end{proof}

Now we turn to the case in which $\max(\|\nabla F(x)\|,\|m^+\|,\|m\|)\le \gamma / {\eta}$.
\begin{lemma}
\label{lemma_mom_clip_deterministic_small_basic}
Suppose $\max(\|\nabla F(x)\|,\|m^+\|,\|m\|)\le \gamma / {\eta}$. Then
\begin{equation}
    G(x^+,m^+)-G(x,m)\le -\frac {\eta} 2 \left(c_1\|\nabla F(x)\|^2+2c_2 \left\langle \nabla F(x), m\right\rangle + c_3 \|m\|^2\right)
\end{equation}
where $c_1 = \nu(1-\beta)(2-\beta)-L\eta (1-\beta \nu)^2+2 (1-\nu) , c_2 = \nu\beta (1-\beta)-L \eta \beta \nu (1- \beta \nu) , c_3 = \nu\beta (1+\beta) - L\eta (\beta \nu)^2$, and $L=AL_0+BL_1\gamma/\eta$.
\end{lemma}
\begin{proof}
In the case of $\|m\|\le \gamma / {\eta}$, we have $\eta \|m\|^2\le \gamma \|m\|$, thus
\begin{equation}
\label{lemma_mom_clip_deterministic_small_basic_0}
    G(x^+,m^+)-G(x,m)= (F(x^+)-F(x))+\frac {\nu\beta\eta} {2(1-\beta)}(\|m^+\|^2-\|m\|^2)
\end{equation}
We then bound $F(x^+)-F(x)$ and $\|m^+\|^2-\|m\|^2$. Note that $\|m^+\|\le \gamma/\eta$ implies that the algorithm performs an update without normalization.
Define $L := AL_0+BL_1\gamma/\eta$, then again by descent inequality,
\begin{equation}\begin{aligned}
\label{lemma_qhm_clip_deterministic_2}
F(x^+)-F(x) &\leq \left\langle x^+ - x, \nabla F(x) \right\rangle + \frac{AL_0+BL_1\|\nabla F(x)\|}{2}\|x^{+}-x\|^2 \\
&= - \left[ \nu \eta \left\langle m^+, \nabla F(x) \right\rangle  + (1-\nu)\eta \|\nabla F(x)\|^2 \right]  +  \\
&\quad \frac{AL_0+BL_1 \|\nabla F(x)\|}{2} \eta^2 \|(1-\beta \nu)\nabla(x)+\beta \nu m\|^2 \\
&\leq - \left[ \nu \eta \left\langle m^+, \nabla F(x) \right\rangle  + (1-\nu)\eta \|\nabla F(x)\|^2 \right] +\frac{L}{2}\eta^2 \|(1-\beta \nu)\nabla F(x)+\beta \nu m\|^2 \\
&\overset{Rearranging}{\leq} -\left[ \nu(1-\beta)\eta + (1-\nu)\eta -  \frac{L}{2} \eta^2(1-\beta \nu)^2 \right] \|\nabla F(x)\|^2 \\
&\quad - \left[ \nu\beta \eta -  L\eta^2 (\beta \nu) (1-\beta \nu) \right] \left\langle \nabla F(x), m \right\rangle + \frac{L}{2} \eta^2 \beta^2 \nu^2 \|m\|^2 
\end{aligned}\end{equation}
Since
\begin{equation}\begin{aligned}
\left\|m^{+}\right\|^{2}-\|m\|^{2}=(1-\beta)^{2}\|\nabla F(x)\|^{2}-(1+\beta)(1-\beta)\|m\|^{2}+2 \beta(1-\beta)\langle \nabla F(x), m\rangle
\end{aligned}\end{equation}
by definition of the Lyapunov function, we have 
$$G(x^+,m^+)-G(x,m) \leq - \frac{\eta}{2} \left( c_1\|\nabla F(x)\|^2 + 2c_2 \left\langle \nabla F(x),m \right\rangle + c_3\|m\|^2 \right)$$
where $c_1 = \nu(1-\beta)(2-\beta)-L\eta (1-\beta \nu)^2+2 (1-\nu) , c_2 = \nu\beta (1-\beta)-L \eta \beta \nu (1- \beta \nu) , c_3 = \nu\beta (1+\beta) - L\eta (\beta \nu)^2$.
\end{proof}

\begin{lemma}
\label{lemma_mom_clip_5}
Let $c_1,c_2,c_3$ and $L$ be defined in Lemma \ref{lemma_mom_clip_deterministic_small_basic}. If $L\eta \le 1$, then the matrix 
$$H=\left( {\begin{array}{*{20}{c}}
  [c_1-(1-\nu\beta)]I_d&c_2I_d \\ 
  c_2I_d&(c_3-\nu\beta)I_d 
\end{array}} \right)$$
is symmetric and positive semi-definite, where $I_d$ is the $d\times d$ identity matrix.
\end{lemma}
\begin{proof}
In fact we only need to consider the case when $d=1$, because the eigenvalues of $H_{2d\times 2d}$ can only be those that appears in $H_{2\times 2}$ ($d=1$). Denote two eigenvalues be $\lambda_1,\lambda_2$ when $d=1$. A direct calculation shows that
$$\lambda_1\lambda_2=\det H=[c_1-(1-\nu\beta)](c_3-\nu\beta)-c_2^2=\nu(1-\nu)\beta^2(1- L\eta)$$
$$\lambda_1+\lambda_2=c_1+c_3-1=(1-\nu\beta)^2(1-L\eta)+(\nu\beta)^2(1-L\eta)+2\beta^2\nu(1-\nu)$$
If $L\eta\le 1$, then $\lambda_1\lambda_2\ge 0$ and $\lambda_1+\lambda_2\ge 0$, which is equivalent to the semi-definiteness of $H$.
\end{proof}

\begin{corollary}
\label{lemma_mom_clip_6}
Suppose $\max(\|\nabla F(x)\|,\|m\|,\|m^+\|)\le \gamma / {\eta}$. If $L\eta\le 1$, Then
\begin{equation}
\label{lemma_mom_clip_6_1}
    G(x^+,m^+)-G(x,m)\le -\frac {\eta} 2(1-\nu\beta)\|\nabla F(x)\|^2 -\frac {\eta} 2\nu\beta\|m\|^2
\end{equation}
\end{corollary}
\begin{proof}
Let $H$ be defined in Lemma \ref{lemma_mom_clip_5}. The result of Lemma \ref{lemma_mom_clip_deterministic_small_basic} can be written in a matrix form:
\begin{equation}
    G(x^+,m^+)-G(x,m)\le -\frac {\eta} 2(1-\nu\beta)\|\nabla F(x)\|^2 -\frac {\eta} 2\nu\beta\|m\|^2-\frac {\eta} 2\left(\nabla F(x)^T,m^T\right)H\left(\nabla F(x)^T,m^T\right)^T
\end{equation}
Using the fact that $H$ is positive semi-definite, we obtain the desired result.
\end{proof}

Note that the amount of descent in Corollary \ref{lemma_mom_clip_6} is small in terms of $\|\nabla F(x)\|$ if $\beta$ and $\nu$ are close to 1. We now try to convert the term $\| m\|$ into $\|\nabla F(x)\|$, which is stated in the following lemma.
\begin{lemma}
\label{lemma_convert_m_to_grad_det}
Suppose $AL_0\eta\le c_1(1-\beta)$ and $BL_1\gamma\le c_3(1-\beta)$ for some constant $c_1$ and $c_3$. Let $ m_0=\nabla F(x_0)$ for simplicity. Let set $\mathcal S$ and $\overline{\mathcal S}$ be defined above. Then
\begin{equation}
\label{lemma_convert_m_to_grad_3}
\begin{aligned}
    \sum_{t\in \overline{\mathcal S}} \|m_{t}\|
    &\ge \frac 1 {1+c_1}\sum_{t\in \overline{\mathcal S}}  \left((1-c_1(1-\nu\beta)-c_3)\|\nabla F(x_{t})\|\right)\\
        &\qquad -\frac 1 {1-\beta}\sum_{t\in\mathcal S} (AL_0+BL_1\|\nabla F(x_t)\|)\gamma
\end{aligned}
\end{equation}
\end{lemma}
\begin{proof}
For any $t\ge 1$, we have
\begin{equation}
\label{lemma_convert_m_to_grad_det1}
\begin{aligned}
    \|m_{t}-\nabla F(x_t)\|
    &\le \| m_{t}-\nabla F(x_{t-1})\|+\|\nabla F(x_{t-1})-\nabla F(x_{t})\|\\
    &\le \beta\| m_{t-1}-\nabla F(x_{t-1})\|+(AL_0+BL_1\|\nabla F(x_{t-1})\|)\times \\
        &\left( \nu\min\left(\eta, \frac{\gamma}{\| m_t\|}\right)\| m_t\| + (1-\nu)\min\left(\eta, \frac{\gamma}{\| \nabla F(x_{t-1})\|}\right)\| \nabla F(x_{t-1})\| \right)
\end{aligned}
\end{equation}
where the last inequality follows by Corollary \ref{DesCor}. Applying \eqref{lemma_convert_m_to_grad_det1} recursively, we obtain
\begin{equation}
\begin{aligned}
    \|m_{t}-\nabla F(x_t)\| &\leq \sum_{\tau=1}^{t}\beta^{t-\tau}(AL_0+BL_1\|\nabla F(x_{\tau-1})\|)\times \\
    &\left( \nu\min\left(\eta, \frac{\gamma}{\| m_{\tau}\|}\right)\| m_{\tau}\| + (1-\nu)\min\left(\eta, \frac{\gamma}{\| \nabla F(x_{\tau-1})\|}\right)\| \nabla F(x_{\tau-1})\| \right)
\end{aligned}
\end{equation}
Therefore,
\begin{equation}
\label{lemma_convert_m_to_grad_det2}
\begin{aligned}
    &\quad \sum_{t=0}^{T-1} \| m_{t}-\nabla F(x_t)\| \\
    &\le \frac 1 {1-\beta}\sum_{t=0}^{T-1}(AL_0+BL_1\|\nabla F(x_{t})\|) \left( \nu\min\left(\eta, \frac{\gamma}{\| m_{t+1}\|}\right)\| m_{t+1}\| + (1-\nu)\min\left(\eta, \frac{\gamma}{\| \nabla F(x_{t})\|}\right)\| \nabla F(x_{t})\| \right)\\
    &\le \frac 1 {1-\beta}\left(\sum_{t=0}^{T-1}BL_1\gamma\|\nabla F(x_{t})\|+\sum_{t\in\mathcal S} AL_0\gamma+\sum_{t\in  \overline{\mathcal S}} AL_0\eta\left( (1-\nu) \|\nabla F(x_{t})\|+ \nu \|m_{t+1}\|\right)\right)
\end{aligned}
\end{equation}
Therefore we obtain
\begin{equation}
\begin{aligned}
    &\quad\sum_{t=0}^{T-1} \| m_{t}-\nabla F(x_t)\|\\
    &\le \frac 1 {1-\beta}\left(\sum_{t\in\mathcal S} (AL_0+BL_1\|\nabla F(x_t)\|)\gamma+\sum_{t\in \overline{\mathcal S}} \left[AL_0\nu\eta\|m_{t+1}\|+(BL_1\gamma+AL_0(1-\nu)\eta)\|\nabla F(x_t)\|\right]\right)\\
    &\le \frac 1 {1-\beta}\left(\sum_{t\in\mathcal S} (AL_0+BL_1\|\nabla F(x_t)\|)\gamma\right)+\\
        &\qquad \frac 1 {1-\beta}\left(\sum_{t\in  \overline{\mathcal S}} AL_0\nu\eta\beta\| m_{t}\|+(AL_0\eta(1-\nu\beta)+BL_1\gamma)\|\nabla F(x_t)\|\right)\\
    &\le \frac 1 {1-\beta}\left(\sum_{t\in\mathcal S} (AL_0+BL_1\|\nabla F(x_t)\|)\gamma\right)+ \left(\sum_{t\in  \overline{\mathcal S}}(c_1(1-\nu\beta)+c_3)\|\nabla F(x_t)\|+ c_1\nu\beta\| m_t\|\right)
\end{aligned}
\end{equation}
Using $\| m_{t}\|\ge \|\nabla F(x_{t})\|-\| m_t-\nabla F(x_{t})\|$ and some straightforward calculation, we obtain
\begin{equation}
\begin{aligned}
    (1+c_1)\sum_{t\in  \overline{\mathcal S}} \| m_{t}\|
    &\ge \left(\sum_{t\in  \overline{\mathcal S}}  (1-c_1(1-\nu\beta)-c_3)\|\nabla F(x_{t})\|\right)\\
        &\qquad -\frac 1 {1-\beta}\left(\sum_{t\in\mathcal S} (AL_0+BL_1\|\nabla F(x_t)\|)\gamma\right)
\end{aligned}
\end{equation}
\end{proof}

Now we are ready to prove the main theorem.
\begin{theorem}
\label{lemma_mom_clip_deterministic_final}
Let $F^*$ be the optimal value, and $\Delta = F(x_0)-F^*$. Assume $m_0=\nabla F(x_0)$ for simplicity. If $\gamma \le \frac {1-\beta} {10BL_1}$ and $\eta \le \frac{1-\beta}{10AL_0}$,
where constants $A=1+e^{1/10}-10{(e^{1/10}-1)<1.06}$, $B=10(e^{1/10}-1)<1.06$, and $\epsilon < \frac {\gamma}{ 5\eta}$, then 
$$\dfrac{1}{T}\sum_{t=1}^T \|\nabla F(x_t)\| \le 2\epsilon $$
as long as
\begin{equation}
\begin{aligned}
T &\ge  \dfrac{3}{\epsilon^2\eta}\Delta
\end{aligned}
\end{equation}
\end{theorem}
\begin{proof}
By calculating $L\eta=AL_0\eta+BL_1\gamma\le (1-\beta)/5<1$, we can use Corollary \ref{lemma_mom_clip_6}. Taking summation of the inequality \eqref{lemma_mom_clip_6_1} over steps $t\in \overline{\mathcal{S}}= [0,T-1]\backslash \mathcal{S}$, we obtain
\begin{equation}
\label{lemma_mom_clip_7_2}
    \sum_{t\in \overline{\mathcal{S}}} G(x_{t+1},m_{t+1})-G(x_t,m_t)\le -\frac {\eta} 2\sum_{t\in \overline{\mathcal{S}}}\left((1-\nu\beta)\|\nabla F(x_t)\|^2+\nu\beta\|m_t\|^2\right)
\end{equation}
Combining \eqref{lemma_mom_clip_7_2} and \eqref{lemma_mom_clip_deterministic_large_sum_0} in Corollary \ref{lemma_mom_clip_deterministic_large_sum} we obtain
\begin{equation}
\label{lemma_mom_clip_7_3}
\begin{aligned}
    &G(x_T,m_T)-G(x_0,m_0)=\sum_{t=0}^{T-1} G(x_{t+1},m_{t+1})-G(x_t,m_t)\\
    &\le  -\frac {\eta} 2\sum_{t\in \overline{\mathcal{S}}}\left((1-\nu\beta)\|\nabla F(x_t)\|^2+\nu\beta\|m_t\|^2\right) + \\
    &\qquad\frac {12\gamma} {5\beta(1-\beta)}\|\delta_0\|+\left(\frac {12} {5(1-\beta)}AL_0+\frac {12\gamma}{5\eta(1-\beta)}BL_1 + \frac 1 2 AL_0\right)\gamma^2 T_{\mathcal{S}}+\\
    &\qquad\gamma\sum_{t \in \mathcal S} \left[-\frac 1 {5}(2\|\nabla F(x_t)\|+3\frac {\gamma}{\eta}) +\left(\frac 1 2+\frac {12} {5(1-\beta)}\right) BL_1\gamma \|\nabla F(x_t)\|\right]
\end{aligned}
\end{equation}
By the assumption
\begin{equation}
    \label{stepsize_det}
    \gamma \le \frac {1-\beta} {10BL_1},\quad
    \eta \le \frac{1-\beta}{10AL_0}
\end{equation}
we have $AL_0\eta \le (1-\beta)/10$ and $BL_1\gamma \le (1-\beta)/10$. Using Lemma \ref{lemma_convert_m_to_grad} we have
\begin{equation}
    \begin{aligned}
     \sum_{t\in  \overline{\mathcal S}} \| m_{t}\|
    &\ge \frac {8}{11}\sum_{t\in  \overline{\mathcal S}}  \|\nabla F(x_{t})\| -\frac 1 {1-\beta}\left(\sum_{t\in\mathcal S} (AL_0+BL_1\|\nabla F(x_t)\|)\gamma\right)
    \end{aligned}
\end{equation}
Therefore by standard inequality $x^2\ge 2\epsilon x - \epsilon^2$ and \eqref{lemma_mom_clip_7_3} we obtain
\begin{equation}
    \begin{aligned}
    \label{longineq}
    &\quad G(x_0,m_0)-G(x_T,m_T) \\
    &\geq \frac {\eta} 2\sum_{t\in \overline{\mathcal{S}}}\left((1-\nu\beta)\|\nabla F(x_t)\|^2+2\nu\beta\epsilon \|m_t\| - \nu\beta\epsilon^2\right) \\
    &\quad + \left( \dfrac{3}{5}\dfrac{\gamma^2}{\eta} - \left(\frac {12} {5(1-\beta)}AL_0+\frac {12\gamma}{5\eta(1-\beta)}BL_1 + \frac 1 2 AL_0\right)\gamma^2  \right) T_{\mathcal{S}} \\
    &\quad + \gamma \left( \dfrac{2}{5} - \left( \dfrac{1}{2} + \dfrac{12}{5(1-\beta)} \right) BL_1\gamma \right) \sum_{t \in \mathcal S} \|\nabla F(x_t)\|\\
    &\geq \sum_{t \in \mathcal S} U(x_t) + \sum_{t \in \overline{\mathcal S}} V(x_t)
    \end{aligned}
\end{equation}
Where 
\begin{equation}
\begin{aligned}
    U(x) &:= \left( \dfrac{3}{5}\dfrac{\gamma^2}{\eta} - \left(\frac {12} {5(1-\beta)}AL_0+\frac {12\gamma}{5\eta(1-\beta)}BL_1 + \frac 1 2 AL_0\right)\gamma^2 - \dfrac{\nu\beta}{1-\beta}AL_0\epsilon\gamma\eta  \right) \\
    &\qquad + \gamma \left( \dfrac{2}{5} - \left( \dfrac{1}{2} + \dfrac{12}{5(1-\beta)} \right) BL_1\gamma - \dfrac{\nu\beta}{1-\beta}\epsilon\eta BL_1 \right) \|\nabla F(x)\| \\
    V(x) &:= \dfrac{\eta}{2}(1-\nu\beta)\|\nabla F(x)\|^2 + \dfrac{8}{11}\nu\beta\epsilon\eta \|\nabla F(x)\| - \dfrac{1}{2}\nu\beta\epsilon^2\eta
\end{aligned}
\end{equation}
We now simplify $U(x)$. Let $\epsilon \le \frac {\gamma}{5\eta}$. By \eqref{stepsize_det} we have
\begin{equation}
\begin{aligned}
    \frac 2 5 - \left(\frac 1 2+\frac {12} {5(1-\beta)}\right) BL_1\gamma-\dfrac{\nu\beta}{1-\beta}\epsilon\eta BL_1 \ge \frac 2 5-\frac {12}{50}-\frac 1 {20} \ge \frac 1{10}\\
     \frac 3 5 -\frac {12\gamma}{5(1-\beta)}BL_1\ge \frac 3 {10}
\end{aligned}
\end{equation}
Therefore
\begin{equation}
\begin{aligned}
\label{U0}
    U(x) &\geq \frac 3 {10} \frac {\gamma^2}{\eta} - \left(\frac {12}{5(1-\beta)}+\frac 1 2\right)AL_0\gamma^2 - \dfrac{\nu\beta}{1-\beta}AL_0\epsilon\gamma\eta + \frac 1 {10}\gamma\|\nabla F(x)\| \\
    &\geq \left(\frac 3 {5(1-\beta)}-\frac 1 2\right)AL_0\gamma^2 - \dfrac{\nu\beta}{1-\beta}AL_0\epsilon\gamma\eta + \frac 1 {10}\gamma\|\nabla F(x)\|  \\
    &\geq \dfrac{1}{10(1-\beta)}AL_0\gamma^2 + \frac 1 {10}\gamma\|\nabla F(x)\| 
\end{aligned}
\end{equation}
We can also bound $V(x)$ as follows:
\begin{equation}
\begin{aligned}
    V(x) &\geq (1-\nu\beta)\epsilon\eta \|\nabla F(x)\|-\frac {\eta} 2 (1-\nu\beta)\epsilon^2 + \dfrac{8}{11}\nu\beta\epsilon\eta \|\nabla F(x)\| - \dfrac{1}{2}\nu\beta\epsilon^2\eta \\
    &\geq \dfrac{1}{2} \epsilon\eta \|\nabla F(x)\| - \dfrac{1}{2}\epsilon^2\eta
\end{aligned}
\end{equation}
Since $\epsilon < \dfrac{\gamma}{5\eta}$, we have $U(x) \geq V(x)$. Therefore by \eqref{longineq} and Lemma \ref{GradNormBd} we have
\begin{equation}
    \begin{aligned}
    T \sum_{t=0}^{T-1}  \dfrac{1}{2} \epsilon\eta \left(\|\nabla F(x)\| - \epsilon \right) &\leq \Delta + \dfrac{\beta}{2(1-\beta)} \min \{ \eta \|\nabla F(x_0)\|^2,\gamma \|\nabla F(x_0)\| \} \\
    &\leq \Delta + \dfrac{4\beta}{1-\beta}\Delta \max \{ L_0\eta,L_1\gamma \} \\
    &\leq \dfrac{7}{5}\Delta
    \end{aligned}
\end{equation}
Thus 
\begin{equation}
    \dfrac{1}{T} \sum_{t=0}^{T-1}\|\nabla F(x_t)\| \leq 2\epsilon
\end{equation}
as long as
\begin{equation}
    T > \dfrac{3}{\epsilon^2\eta}\Delta
\end{equation}
\end{proof}

\subsection{Proof of Theorem 3.2}
\label{section_mom_clip_stochastic}
We now prove the stochastic case. As before, to simplify the notation we write the update formula as
\begin{equation}
    \begin{aligned}
        m^+&=\beta m + (1-\beta) \nabla f(x,\xi) \\
        x^+&=x - \left( \nu \min\left(\eta, \frac {\gamma} {\|m^+\|} \right)m^+ + (1-\nu) \min\left(\eta, \frac {\gamma} {\|\nabla f(x,\xi)\|} \right)\nabla f(x,\xi) \right)
    \end{aligned}
\end{equation}
when analyzing a single iteration. The error between $m^+$ and $\nabla F(x)$ is denoted as $\delta = m^+ - \nabla F(x)$. We define the true momentum $\tilde m$ as follows:
\begin{equation}
\label{tilde_m_def}
     \tilde m^+=\beta \tilde m + (1-\beta) \nabla F(x)
\end{equation}
where $\tilde m_0=m_0$. Similarly, the error between $\tilde m^+$ and $\nabla F(x)$ is denoted as $\tilde{\delta} = \tilde m^+ - \nabla F(x)$. \par
In stochastic case, we define the Lyapunov function to be
\begin{equation}
    G(x,\tilde m)=F(x)+\frac {\nu\beta} {2(1-\beta)} \min\left({\eta} \|\tilde m\|^2,{\gamma}\|\tilde m\|\right)
\end{equation}
The only change is that we use the true momentum $\tilde m$ instead of stochastic momentum $m$. Note that Lemma \ref{lemma_mom_clip_1} and Lemma \ref{lemma_mom_clip_2} can still be used in stochastic case. The momentum $m$ and error $\delta$ in Lemma \ref{lemma_mom_clip_2} will be changed to $\tilde m$ and $\tilde \delta$ respectively.\par
Suppose $\gamma\le c/L_1$ for some constant $c$, and we denote $A=1+e^c-\frac {e^c-1} c$ and $B=\frac {e^c-1} c$, just the same as in the descent inequality (Lemma \ref{DesIneq}). When $\gamma\le \frac {1-\beta}{50 L_1}\epsilon\le \frac 1 {500L_1}$ (in Theorem 3.2), we can take $c=1/500$ and $A=B=1.002$.
\begin{lemma}
\label{lemma_mom_clip_3}
The difference between $m$ and $\tilde m$  satisfies:
\begin{equation}
    \|m^+-\tilde m^+\|\le \sigma
\end{equation}
Furthermore, in expectation
\begin{equation}
    \mathbb E\|m^+-\tilde m^+\|^2\le \frac {1-\beta} {1+\beta}\sigma
\end{equation}
\end{lemma}
\begin{proof}
By expanding $m_{t+1}$ and $\tilde m_{t+1}$, we get
\begin{equation}
\begin{aligned}
    \|m_{t+1}-\tilde m_{t+1}\|
    &=(1-\beta)\left\|\sum_{\tau=0}^t \beta^{t-\tau} (\nabla  f(x_{\tau},\xi_{\tau})-\nabla F(x_{\tau}))\right\|\\
    &\le (1-\beta)\sum_{\tau=0}^t \beta^{t-\tau}\| \nabla  f(x_{\tau},\xi_{\tau})-\nabla F(x_{\tau})\|\\
    &\le (1-\beta)\sum_{\tau=0}^t \beta^{t-\tau}\sigma\le\sigma
\end{aligned}
\end{equation}
Furthermore, using the noise assumption, for different time steps $t,t'$, we have $$\mathbb E[\left\langle\nabla f(x_t,\xi_t)-\nabla F(x_t),\nabla f(x_{t'},\xi_{t'})-\nabla F(x_{t'})\right\rangle]=0$$
Therefore
\begin{equation}
    \mathbb E[\|m_{t+1}-\tilde m_{t+1}\|^2]=\mathbb E \left[\sum_{\tau=0}^t (1-\beta)^2\beta^{2(t-\tau)} \|\nabla f(x_{\tau},\xi_{\tau})-\nabla F(x_{\tau})\|^2\right] \le \frac {1-\beta}{1+\beta}\sigma^2
\end{equation}
\end{proof}

\begin{lemma}
\label{lemma_mom_clip_stochastic_large_basic}
Suppose $\max(5\|\nabla F(x)\|/4,\|m^+\|,\|\tilde m\|)\ge \gamma /\eta$. Then
\begin{equation}
\begin{aligned}
    \quad &G(x^+,\tilde m^+)-G(x,\tilde m)\\
    &\le -\frac 4 5 \times \frac {2\gamma} {5} \|\nabla F(x)\|-\frac {16}{25} \times \frac  {3\gamma^2} {5\eta}+\frac {\gamma^2} 2 (AL_0+BL_1 \|\nabla F(x)\|)+\frac {12} {5\beta}\gamma \|\tilde \delta\|\\
    &-\nu\eta\left\langle \nabla F(x), m^+-\tilde m^+\right\rangle-(1-\nu)\eta\left\langle \nabla F(x), \nabla f(x,\xi)-\nabla F(x)\right\rangle + \left(\eta \|\nabla F(x)\|+\frac 7 5 \gamma\right)\sigma
\end{aligned}
\end{equation}
\end{lemma}
\begin{proof}
Based on Lemma \ref{lemma_mom_clip_2}, we only need to bound $F(x^+)-F(x)$. We use the $(L_0,L_1)$-smooth condition:
\begin{equation}
    F(x^+)-F(x)\le \left\langle \nabla F(x), x^+-x\right\rangle + \frac {\gamma^2} 2 (AL_0+BL_1 \|\nabla F(x)\|)
\end{equation}
Now we bound $\left\langle \nabla F(x), x^+-x\right\rangle$. The calculation is similar to the deterministic setting. We first bound $-\min \left( \eta, \frac{\gamma}{\|m^{+}\|} \right)\left\langle m^+, \nabla F(x) \right\rangle$. Consider the following three cases, all of which are analogous to the proof of Lemma \ref{lemma_mom_clip_deterministic_large_basic}:
\begin{itemize}
    \item $\|m^+\|\ge \gamma/\eta$. The algorithm performs a normalized update. We have
    \begin{equation*}
         -\frac{\gamma}{\|m^{+}\|}\left\langle m^+, \nabla F(x) \right\rangle\le-\frac 2 5\gamma \|\nabla F(x)\|-  \frac 3 5\gamma \|m^+\|+\frac 7 5\gamma \|\delta\|
    \end{equation*}
    
    \item $\|m^+\|< \gamma/\eta$ and $\|\nabla F(x)\|\ge 4\gamma/5\eta$.  The algorithm performs an unnormalized update. We have
    \begin{equation*}
        -\eta\left\langle \nabla F(x),m^+\right\rangle\le- \frac 4 5\times \frac 2 5 \gamma\|\nabla F(x)\|-\frac {16}{25}\times \frac {3\gamma^2} {5\eta} +\frac 4 5\times \frac 7 5 \gamma\|\delta\|
    \end{equation*}
    
    \item $\|m^+\|< \gamma/\eta$ and $\|\nabla F(x)\|< 4\gamma/5\eta$. In this case $\|\tilde m\|\ge \gamma/\eta$. The algorithm performs an unnormalized update. We have
    \begin{equation*}
        -\eta\left\langle \nabla F(x),\tilde m^+\right\rangle\le -\frac 2 5 \gamma\|\nabla F(x)\|-\frac {3\gamma^2} {5\eta} +\left(\frac {12} {5\beta}-1\right) \gamma\|\tilde \delta\|
    \end{equation*}
\end{itemize}
Therefore in all the cases, we have
\begin{equation}
\label{lemmma_mom_clip_stochastic_large_basic_x}
\begin{aligned}
    -\min \left( \eta, \frac{\gamma}{\|m^{+}\|} \right)\left\langle m^+, \nabla F(x) \right\rangle\le&-\frac 4 5 \times \frac 2 5\gamma \|\nabla F(x)\|-\frac {16} {25} \times \frac {3\gamma^2} {5\eta}+\left(\frac {12} {5\beta}-1\right)\gamma \|\tilde \delta\|\\
    &-\eta\left\langle \nabla F(x), m^+-\tilde m^+\right\rangle + \left(\eta \|\nabla F(x)\|+\frac 7 5 \gamma\right)\sigma
\end{aligned}
\end{equation}
where \eqref{lemmma_mom_clip_stochastic_large_basic_x} uses the following two inequalities which can be obtained by Lemma \ref{lemma_mom_clip_3}:
\begin{align}
    \|\delta\|&\le\|\tilde\delta\|+\sigma\\
    -\left\langle \nabla F(x), m^+-\tilde m^+\right\rangle &\le\|\nabla F(x)\|\sigma
\end{align}
We next bound $-\min \left( \eta, \frac{\gamma}{\|\nabla f(x,\xi)\|} \right)\left\langle \nabla f(x,\xi), \nabla F(x) \right\rangle$. Consider the following cases, all of which are analogous to the proof of Lemma \ref{lemma_mom_clip_deterministic_large_basic}:
\begin{itemize}
    \item $\|\nabla f(x,\xi)\|\ge \gamma/\eta$. In this case we can use Lemma \ref{lemma_mom_clip_1} with $\mu=2/5$:
    \begin{equation}
    \begin{aligned}
        &\quad-\min \left(\eta, \frac{\gamma}{\|\nabla f(x,\xi)\|}\right) \left\langle \nabla f(x,\xi),\nabla F(x) \right\rangle\\
        &=-\gamma \frac {\left\langle \nabla f(x,\xi),\nabla F(x) \right\rangle} {\|\nabla f(x,\xi)\|}\\
        &\le \gamma\left(-\frac 2 5 \|\nabla F(x)\|-\frac 3 5 \|\nabla f(x,\xi)\|+\frac 7 5 \|\nabla F(x)-\nabla f(x,\xi)\|\right)\\
        &\le \gamma\left(-\frac 2 5 \|\nabla F(x)\|-\frac {3\gamma} {5\eta} +\frac {7} 5 \sigma\right)
    \end{aligned}
    \end{equation}
    \item $\|\nabla f(x,\xi)\|< \gamma/\eta$. In this case
    \begin{equation}
    \begin{aligned}
        &\quad-\min \left(\eta, \frac{\gamma}{\|\nabla f(x,\xi)\|}\right) \left\langle \nabla f(x,\xi),\nabla F(x) \right\rangle\\
        &=-\eta \left\langle \nabla f(x,\xi),\nabla F(x) \right\rangle\\
        &=-\eta\|\nabla F(x)\|^2-\eta \left\langle \nabla f(x,\xi)-\nabla F(x),\nabla F(x) \right\rangle
    \end{aligned}
    \end{equation}
    We now bound $-\eta \|\nabla F(x)\|^2$. If $\|\nabla F(x)\|\ge \frac {4\gamma}{5\eta}$, then  $-\eta \|\nabla F(x)\|^2\le -\frac 4 5\gamma\|\nabla F(x)\|$.\\
    If $\|\nabla F(x)\|<\frac {4\gamma}{5\eta}$ and $\|m^+\|\ge \frac {4\gamma}{5\eta}$, then using the same calculation as in the deterministic case,
    \begin{equation*}
    \begin{aligned}
        -\eta \|\nabla F(x)\|^2&\le -\frac 2 5\times \frac 4 5 \gamma \|\nabla F(x)\|-\frac {16} {25}\times \frac {3\gamma^2}{5\eta}+\frac 4 5\times \frac 8 5\gamma\left(\|m^+\|-\|\nabla F(x)\|\right)\\
        &\le -\frac 4 5\times \frac 2 5\gamma \|\nabla F(x)\|-\frac {16}{25} \times \frac {3\gamma^2}{5\eta}+\frac 7 5\gamma\left(\|\tilde\delta\|+\sigma\right)
    \end{aligned}
    \end{equation*}
    If $\|\nabla F(x)\|<\frac {4\gamma}{5\eta}$ and $\|m^+\|< \frac{4\gamma}{5\eta}$, then $\|\tilde m\|\ge {\gamma}/{\eta}$. Using the same calculation we have
    \begin{equation*}
    \begin{aligned}
        -\eta \|\nabla F(x)\|^2&\le-\frac 2 5\gamma \|\nabla F(x)\|-\frac {3\gamma^2}{5\eta}+\frac 8 5\gamma\left(\|\tilde m\|-\|\nabla F(x)\|\right)\\
        &\le  -\frac 2 5\gamma \|\nabla F(x)\|-\frac {3\gamma^2}{5\eta}+\frac 8 {5\beta}\gamma\|\tilde\delta\|
    \end{aligned}
    \end{equation*}
\end{itemize}
Therefore in all the cases we have
\begin{equation}
\label{lemmma_mom_clip_stochastic_large_basic_y}
\begin{aligned}
    -\min \left( \eta, \frac{\gamma}{\|\nabla f(x,\xi)\|} \right)\left\langle \nabla f(x,\xi), \nabla F(x) \right\rangle\le&-\frac 4 5 \times \frac 2 5\gamma \|\nabla F(x)\|-\frac {16} {25} \times \frac {3\gamma^2} {5\eta}+\frac {8} {5\beta}\gamma \|\tilde \delta\|\\
    &-\eta\left\langle \nabla F(x), \nabla f(x,\xi)-\nabla F(x)\right\rangle + \left(\eta \|\nabla F(x)\|+\frac 7 5 \gamma\right)\sigma
\end{aligned}
\end{equation}
we finally obtain
\begin{equation}
\begin{aligned}
    \quad &G(x^+,\tilde m^+)-G(x,\tilde m)\\
    &\le -\frac 4 5 \times \frac {2\gamma} {5} \|\nabla F(x)\|-\frac {16}{25} \times \frac  {3\gamma^2} {5\eta}+\frac {\gamma^2} 2 (AL_0+BL_1 \|\nabla F(x)\|)+\frac {12} {5\beta}\gamma \|\tilde \delta\|\\
    &-\nu\eta\left\langle \nabla F(x), m^+-\tilde m^+\right\rangle-(1-\nu)\eta\left\langle \nabla F(x), \nabla f(x,\xi)-\nabla F(x)\right\rangle + \left(\eta \|\nabla F(x)\|+\frac 7 5 \gamma\right)\sigma
\end{aligned}
\end{equation}
\end{proof}

Let $\mathcal S=\{t\in [0,T-1]:\max(5\|F(x_t)\|/4,\|m_{t+1}\|,\|\tilde m_t\|)\ge \gamma / {\eta}\}$ and $\overline{\mathcal  S}=[0,T-1]\backslash\mathcal  S$. Let $T_{ \mathcal S}=| \mathcal S|$, then $T-T_{ \mathcal S}=|\overline{ \mathcal S}|$. Parallel to Corollary \ref{lemma_mom_clip_deterministic_large_sum}, we directly have the following corollary. 
\begin{corollary}
\label{lemma_mom_clip_stochastic_large_sum}
Let set $ \mathcal S$ and $T_{ \mathcal S}$ be defined above. Then
\begin{equation}
\begin{aligned}
    &\quad\sum_{t\in \mathcal S}G(x_{t+1},m_{t+1})-G(x_t,m_t)\\
    &\le \frac {12\gamma} {5\beta(1-\beta)}\|\tilde \delta_0\|+\left(\frac {12} {5(1-\beta)}AL_0+\frac {12\gamma}{5\eta(1-\beta)}BL_1 + \frac 1 2 AL_0\right)\gamma^2 T_{\mathcal{S}}\\
        &\quad -\eta\sum_{t\in \mathcal{S}}(\nu\left\langle \nabla F(x_t), m_{t+1}-\tilde m_{t+1}\right\rangle+(1-\nu)\left\langle \nabla F(x_t), \nabla f(x_t,\xi_t)-\nabla F(x_t)\right\rangle)+\\
        &\quad\gamma\sum_{t \in \mathcal S} \left[-\left(\left(\frac 4 5\times\frac 2 5-\frac {\eta}{\gamma}\sigma\right)\|\nabla F(x_t)\|+\left(\frac {16} {25}\times\frac {3\gamma}{5\eta}-\frac 7 5 \sigma\right)\right)+\frac {\gamma} 2 BL_1 \|\nabla F(x_t)\| +\frac {12\gamma} {5(1-\beta)} BL_1 \|\nabla F(x_t)\|\right]
\end{aligned}
\end{equation}
\end{corollary}

Next we turn to the case in which $\max(5\|\nabla F(x)\|/4,\|m^+\|,\|\tilde m\|)\le \gamma / {\eta}$.
\begin{lemma}
\label{lemma_mom_clip_stochastic_small_basic}
Assume $\max(5\|\nabla F(x)\|/4,\|m^+\|,\|\tilde m\|)\le \gamma / {\eta}$, and $\gamma/\eta=5\sigma$. If $AL_0\eta\le1$, then
\begin{equation}
\label{lemma_mom_clip_stochastic_small_basic_1}
\begin{aligned}
    &\quad G(x^+,\tilde m^+)-G(x,\tilde m)\\
    &\le -\frac {\eta} 2 (1-\nu\beta)\|\nabla F(x)\|^2-\frac {\eta} 2 \nu\beta\|\tilde m\|^2 +\frac {\gamma^2} 2 BL_1 \|\nabla F(x)\|\\
    &\quad -\nu\eta \left\langle \nabla F(x),  m^+-\tilde m^+\right\rangle - (1-\nu)\eta \left\langle \nabla F(x),  \nabla f(x,\xi)-\nabla F(x)\right\rangle\\
    &\quad +\eta^2AL_0 \sigma \|\nu \tilde m^++(1-\nu)\nabla F(x)\| +\frac 1 2 \eta^2 AL_0\|\nu (m^+-\tilde m^+)+(1-\nu)(\nabla f(x,\xi)-\nabla F(x))\|^2
\end{aligned}
\end{equation}
where $c_1 = \nu(1-\beta)(2-\beta)-AL_0\eta (1-\beta \nu)^2+2 (1-\nu) , c_2 = \nu\beta (1-\beta)-AL_0 \eta \beta \nu (1- \beta \nu) , c_3 = \nu\beta (1+\beta) - AL_0\eta (\beta \nu)^2$. $c_1=(1-\beta)[2-\beta-AL_0\eta(1-\beta)]$,  $c_2=\beta[1-\beta-AL_0\eta(1-\beta)]$ and $c_3=\beta(1+\beta-AL_0\eta \beta)$.
\end{lemma}
\begin{proof}
Because $\|\nabla f(x,\xi)\|\le 4\gamma/5\eta+\sigma=\gamma/\eta$ and $\|m^+\|\le \gamma/\eta$, the algorithm performs an unnormalized update. The proof is similar to the one in Lemma \ref{lemma_mom_clip_deterministic_small_basic} except for bounding the term $F(x^+)-F(x)$.
\begin{equation}
\label{lemma_mom_clip_stochastic_small_basic_4}
\begin{aligned}
    &\quad F(x^+)-F(x)\\
    &\le -\left\langle \nabla F(x), \nu\eta m^+ + (1-\nu)\eta \nabla f(x,\xi)\right\rangle + \frac {\eta^2} 2 (AL_0+BL_1 \|\nabla F(x)\|)\|\nu m^+ + (1-\nu)\nabla f(x,\xi)\|^2\\
    &\le -\nu\eta \left\langle \nabla F(x), \tilde m^+\right\rangle -\nu\eta \left\langle \nabla F(x),  m^+-\tilde m^+\right\rangle\\
        &\quad -(1-\nu)\eta \left\langle \nabla F(x),\nabla F(x)\right\rangle -(1-\nu)\eta \left\langle \nabla F(x),  \nabla f(x,\xi)-\nabla F(x)\right\rangle\\
        &\quad +\frac {\eta^2} 2 AL_0\left(\|\nu \tilde m^++(1-\nu)\nabla F(x)\|^2+\|\nu (m^+-\tilde m^+)+(1-\nu)(\nabla f(x,\xi)-\nabla F(x))\|^2\right)\\
        &\quad + \eta^2AL_0 \sigma \|\nu \tilde m^++(1-\nu)\nabla F(x)\|+\frac {\eta^2} 2 BL_1\|\nabla F(x)\|\frac {\gamma^2}{\eta^2}
\end{aligned}
\end{equation}
For bounding term $-\nu\eta \left\langle \nabla F(x), \tilde m^+\right\rangle-(1-\nu)\eta \left\langle \nabla F(x),\nabla F(x)\right\rangle  +\frac {\eta^2} 2 AL_0\|\nu \tilde m^++(1-\nu)\nabla F(x)\|^2$ that is not related to noise, the subsequent steps are the same as in Lemma \ref{lemma_mom_clip_deterministic_small_basic}, \ref{lemma_mom_clip_5} and Corollary \ref{lemma_mom_clip_6} (except for $L$ in these Lemmas being replaced by $AL_0$). Other terms in \eqref{lemma_mom_clip_stochastic_small_basic_4} just appears in \eqref{lemma_mom_clip_stochastic_small_basic_1}. Proof is completed.
\end{proof}

Note that the descent inequality in Lemma \ref{lemma_mom_clip_stochastic_small_basic} is small in terms of $\|\nabla F(x)\|$ if $\nu$ and $\beta$ are close to 1. We now try to convert the term $\|\tilde m\|$ into $\|\nabla F(x)\|$, which is stated in the following lemma.
\begin{lemma}
\label{lemma_convert_m_to_grad}
Suppose $AL_0\eta\le c_1(1-\beta)$ and $BL_1\gamma\le c_3(1-\beta)$ for some constant $c_1$ and $c_3$. Let $\tilde m_0=\nabla F(x_0)$ for simplicity. Let set $\mathcal S$ and $\overline{\mathcal S}$ be defined in Corollary \ref{lemma_mom_clip_stochastic_large_sum}. Then
\begin{equation}
\begin{aligned}
    \mathbb E\sum_{t\in \overline{\mathcal S}} \|\tilde m_{t}\|
    &\ge \frac 1 {1+c_1}\mathbb E\left(\sum_{t\in \overline{\mathcal S}}  (1-c_1(1-\nu\beta)-c_3)\|\nabla F(x_{t})\|-c_1\sigma\right)\\
        &\qquad -\frac 1 {1-\beta}\mathbb E\left(\sum_{t\in\mathcal S} (AL_0+BL_1\|\nabla F(x_t)\|)\gamma\right)
\end{aligned}
\end{equation}
\end{lemma}
\begin{proof}
The proof of Lemma \ref{lemma_convert_m_to_grad} is similar to the proof of Lemma \ref{lemma_convert_m_to_grad_det}. We first write \eqref{lemma_convert_m_to_grad_det2} again as follows:
\begin{equation}
\begin{aligned}
    &\quad \sum_{t=0}^{T-1} \| m_{t}-\nabla F(x_t)\| \\
    &\le \frac 1 {1-\beta}\sum_{t=0}^{T-1}(AL_0+BL_1\|\nabla F(x_{t})\|) \times \\
        &\quad \left( \nu\min\left(\eta, \frac{\gamma}{\| m_{t+1}\|}\right)\| m_{t+1}\| + (1-\nu)\min\left(\eta, \frac{\gamma}{\| \nabla f(x_{t},\xi_t)\|}\right)\| \nabla f(x_{t},\xi_t)\| \right)\\
    &\le \frac 1 {1-\beta}\left(\sum_{t=0}^{T-1}BL_1\gamma\|\nabla F(x_{t})\|+\sum_{t\in\mathcal S} AL_0\gamma+\sum_{t\in  \overline{\mathcal S}} AL_0\eta\left( (1-\nu) \|\nabla f(x_{t},\xi_t)\|+ \nu \|m_{t+1}\|\right)\right)
\end{aligned}
\end{equation}
Therefore,
\begin{equation}
\begin{aligned}
    &\quad\sum_{t=0}^{T-1} \| m_{t}-\nabla F(x_t)\|\\
    &\le \frac 1 {1-\beta}\sum_{t\in\mathcal S} (AL_0+BL_1\|\nabla F(x_t)\|)\gamma\\
        &\quad +\frac 1 {1-\beta}\sum_{t\in \overline{\mathcal S}} \left[AL_0\nu\eta\|\tilde m_{t+1}\|+(BL_1\gamma+AL_0(1-\nu)\eta)\|\nabla F(x_t)\|+AL_0\eta\sigma\right]\\
    &\le \frac 1 {1-\beta}\sum_{t\in\mathcal S} (AL_0+BL_1\|\nabla F(x_t)\|)\gamma+\\
        &\qquad \frac 1 {1-\beta}\sum_{t\in  \overline{\mathcal S}} \left(AL_0\nu\eta\beta\| \tilde m_{t}\|+(AL_0\eta(1-\nu\beta)+BL_1\gamma)\|\nabla F(x_t)\|+AL_0\eta\sigma\right)\\
    &\le \frac 1 {1-\beta}\sum_{t\in\mathcal S} (AL_0+BL_1\|\nabla F(x_t)\|)\gamma+ \sum_{t\in  \overline{\mathcal S}}\left((c_1(1-\nu\beta)+c_3)\|\nabla F(x_t)\|+ c_1\nu\beta\| m_t\|+c_1\sigma\right)
\end{aligned}
\end{equation}
Using $\|\tilde m_{t}\|\ge \|\nabla F(x_{t})\|-\|\tilde m_t-\nabla F(x_{t})\|$ and some straightforward calculation, we obtain
\begin{equation}
\begin{aligned}
    (1+c_1)\sum_{t\in  \overline{\mathcal S}} \|\tilde m_{t}\|
    &\ge \sum_{t\in  \overline{\mathcal S}} \left( (1-c_1(1-\nu\beta)-c_3)\|\nabla F(x_{t})\|-c_1\sigma\right)\\
        &\qquad -\frac 1 {1-\beta}\mathbb E\left(\sum_{t\in\mathcal S} (AL_0+BL_1\|\nabla F(x_t)\|)\gamma\right)
\end{aligned}
\end{equation}
\end{proof}

We now merge the two cases corresponding to Corollary \ref{lemma_mom_clip_stochastic_large_sum} and Lemma \ref{lemma_mom_clip_stochastic_small_basic}. The proof of the following theorem involves many techniques which are different from the deterministic case and is far more challenging.

\begin{theorem}
\label{lemma_mom_clip_stochastic_final}
Let $F^*$ be the optimal value, and $\Delta = F(x_0)-F^*$.  Assume $m_0=\nabla F(x_0)$ for simplicity.  Fix $\epsilon\le 0.1$ be a small constant.If $\gamma\le\frac {\epsilon}{\sigma}\min\left(\frac {\epsilon}{AL_0},\frac {1-\beta}{AL_0},\frac {1-\beta}{50 BL_1}\right)$ and $\gamma/\eta=5\sigma$ where constants $A=1.01, B=1.01$, then  
 \begin{equation}
     \dfrac{1}{T} \sum_{t=1}^T \mathbb{E} \|\nabla F(x_t)\| \leq 2\epsilon
 \end{equation}
as long as
\begin{equation}
\begin{aligned}
\label{mom_cpx}
T &\ge \frac {3}{\epsilon^2\eta}\Delta
\end{aligned}
\end{equation}
\end{theorem}
\begin{proof}
Based on the previous results, we take summation over $t$ and obtain
\begin{equation}
\label{lemma_mom_clip_stochastic_0}
\begin{aligned}
    &\quad \sum_{t=0}^{T-1}(G(x_{t+1},\tilde m_{t+1})-G(x_t,\tilde m_t))\\
    &\le \frac {12\gamma} {5\beta(1-\beta)}\|\tilde \delta_0\|+\left(\frac {12} {5(1-\beta)}AL_0+\frac {12\gamma}{5\eta(1-\beta)}BL_1 + \frac 1 2 AL_0\right)\gamma^2 T_{\mathcal{S}}\\
        &\quad -\eta\sum_{t=0}^{T-1}(\nu\left\langle \nabla F(x_t), m_{t+1}-\tilde m_{t+1}\right\rangle+(1-\nu)\left\langle \nabla F(x_t),\nabla f(x_t,\xi_t)-\nabla F(x_t)\right\rangle))+\\
        &\quad\gamma\sum_{t \in \mathcal S} \left[-\left(\left(\frac 4 5\times \frac 2 5-\frac {\eta}{\gamma}\sigma\right)\|\nabla F(x_t)\|+\left(\frac {16} {25}\times \frac {3\gamma}{5\eta}-\frac 7 5 \sigma\right)\right)+\frac {\gamma} 2 BL_1 \|\nabla F(x_t)\| +\frac {12\gamma} {5(1-\beta)} BL_1 \|\nabla F(x_t)\|\right]\\
        &\quad+\sum_{t\in \overline{\mathcal{S}}}-\frac {\eta} 2 \left((1-\nu\beta)\|\nabla F(x_t)\|^2+\nu\beta\|\tilde m_t\|^2\right)+ +\frac {\gamma^2} 2 BL_1\|\nabla F(x)\|\\
        &\quad +\sum_{t\in \overline{\mathcal{S}}} AL_0\eta^2\sigma\|(1-\nu)\nabla F(x_t)+\nu \tilde m_{t+1}\|+\frac {AL_0} 2 \eta^2 \|(1-\nu)(\nabla f(x_t,\xi_t)-\nabla F(x_t))+\nu (m_{t+1}-\tilde m_{t+1})\|^2
\end{aligned}
\end{equation}
We now simplify \eqref{lemma_mom_clip_stochastic_0} by taking expectation. We first have
\begin{equation}
\label{lemma_mom_clip_stochastic_x}
    \mathbb E [\left\langle \nabla F(x_t),\nabla f(x_t,\xi_t)-\nabla F(x_t)\right\rangle]=0
\end{equation}
due to the noise assumption. For the term $\mathbb E \|(1- \nu)(\nabla f(x_t,\xi_t)-\nabla F(x_t))+\nu (m_{t+1}-\tilde m_{t+1})\|^2$, similarly using the noise assumption and Lemma \ref{lemma_mom_clip_3}, we can obtain
\begin{equation}
\label{lemma_mom_clip_stochastic_y}
    \mathbb E \|(1- \nu)(\nabla f(x_t,\xi_t)-\nabla F(x_t))+\nu (m_{t+1}-\tilde m_{t+1})\|^2\le \left((1-\beta\nu)^2+\frac{1-\beta}{1+\beta}\beta^2\nu^2\right)\sigma^2
\end{equation}
We now tackle the most challenging part: the expectation of $\left\langle \nabla F(x_t), m_{t+1}-\tilde m_{t+1}\right\rangle$ for some $t$.
\begin{equation}
\begin{aligned}
    &\quad-\mathbb E\left\langle \nabla F(x_t), m_{t+1}-\tilde m_{t+1}\right\rangle\\
    &=-\mathbb E\left[\left\langle \nabla F(x_{t}), \beta(m_{t}-\tilde m_{t})+(1-\beta)(\nabla f(x_t,\xi_t)-\nabla F(x_t))\right\rangle\right]\\
    &=-\beta\mathbb E\left\langle \nabla F(x_{t}), m_{t}-\tilde m_{t}\right\rangle\\
    &=\beta\mathbb E\left[-\left\langle \nabla F(x_{t-1}), m_{t}-\tilde m_{t}\right\rangle+\left\langle\nabla F(x_{t-1})-\nabla F(x_{t}), m_{t}-\tilde m_{t}\right\rangle\right]
\end{aligned}
\end{equation}
Applying the above equation recursively, we obtain
\begin{equation}
\begin{aligned}
    -\mathbb E\left\langle \nabla F(x_t), m_{t+1}-\tilde m_{t+1}\right\rangle\le \mathbb E\sum_{\tau=0}^{t-1}\beta^{t-\tau}\left\langle\nabla F(x_{\tau})-\nabla F(x_{\tau+1}), m_{\tau+1}-\tilde m_{\tau+1}\right\rangle
\end{aligned}
\end{equation}
Therefore
\begin{equation}
\begin{aligned}
    -\mathbb E\sum_{t=0}^{T-1}\left\langle \nabla F(x_t), m_{t+1}-\tilde m_{t+1}\right\rangle\le \frac {\beta}{1-\beta}\sum_{t=0}^{T-1}\max\left(\mathbb E\left\langle\nabla F(x_{t})-\nabla F(x_{t+1}), m_{t+1}-\tilde m_{t+1}\right\rangle,0\right)
\end{aligned}
\end{equation}
We now bound $\mathbb E[\left\langle\nabla F(x_{t})-\nabla F(x_{t+1}), m_{t+1}-\tilde m_{t+1}\right\rangle]$. 
\begin{equation}
\begin{aligned}
\label{hessian}
    &\quad \mathbb E\left\langle\nabla F(x_{t})-\nabla F(x_{t+1}), m_{t+1}-\tilde m_{t+1}\right\rangle\\
    &=\mathbb E\int_{0}^{1} (x_{t}-x_{t+1})^T\nabla^2 F(\mu x_{t} + (1-\mu) x_{t+1})(m_{t+1}-\tilde m_{t+1})\mathrm d \mu\\
    &=\mathbb E\left[\min\left(\eta,\frac{\gamma}{\|m_{t+1}\|}\right)\int_{0}^{1} \nu m_{t+1}^T\nabla^2 F(\mu x_{t} + (1-\mu) x_{t+1})(m_{t+1}-\tilde m_{t+1})\mathrm d \mu\right]\\
        &\quad +\mathbb E\left[\min\left(\eta,\frac{\gamma}{\|\nabla f(x_t,\xi_t)\|}\right)\int_{0}^{1} (1-\nu)\nabla f(x_{t},\xi_{t})^T\nabla^2 F(\mu x_{t} + (1-\mu) x_{t+1})(m_{t+1}-\tilde m_{t+1})\mathrm d \mu\right]\\
    &\le\mathbb E\left[\min\left(\eta,\frac{\gamma}{\|m_{t+1}\|}\right)\int_{0}^{1} \nu \tilde m_{t+1}^T\nabla^2 F(\mu x_{t} + (1-\mu) x_{t+1})(m_{t+1}-\tilde m_{t+1})\mathrm d \mu\right]\\
        &\quad +\mathbb E\left[\min\left(\eta,\frac{\gamma}{\|\nabla f(x_t,\xi_t)\|}\right)\int_{0}^{1} (1-\nu)\nabla F(x_{t})^T\nabla^2 F(\mu x_{t} + (1-\mu) x_{t+1})(m_{t+1}-\tilde m_{t+1})\mathrm d \mu\right]\\
        &\quad + \eta \mathbb E[(AL_0+BL_1\|\nabla F(x_t)\|)]\sigma^2(1-\beta)\left(\frac {\nu} {1+\beta}+1-\nu\right)\\
    &\le \mathbb E\left[\min\left(\eta,\frac{\gamma}{\|m_{t+1}\|}\right)\nu(AL_0+BL_1\|\nabla F(x_t)\|)\|\tilde m_{t+1}\|\sigma\right]\\
        &\quad +\mathbb E\left[\min\left(\eta,\frac{\gamma}{\|\nabla f(x_t,\xi_t)\|}\right)(1-\nu)(AL_0+BL_1\|\nabla F(x_t)\|\|\nabla F(x_t)\|\sigma\right]\\
        &\quad + \eta \mathbb E[(AL_0+BL_1\|\nabla F(x_t)\|)]\sigma^2(1-\beta)\left(\frac {\nu} {1+\beta}+1-\nu\right)\\
    &\le \mathbb E\left[\eta(\nu\|\tilde m_{t+1}\|+(1-\nu)\|\nabla F(x_t)\|)AL_0\sigma\right]\\
        &\quad +\mathbb E\left[\left(\nu\min\left(\eta,\frac{\gamma}{\|m_{t+1}\|}\right)\|\tilde m_{t+1}\|+(1-\nu)\min\left(\eta,\frac{\gamma}{\|\nabla f(x_t,\xi_t)\|}\right)\|\nabla F(x_t)\|\right)BL_1\|\nabla F(x_t)\|\sigma\right]\\
        &\quad + \eta \mathbb E[(AL_0+BL_1\|\nabla F(x_t)\|)]\sigma^2(1-\beta)\left(\frac {\nu} {1+\beta}+1-\nu\right)\\
    &\le \eta\mathbb E\left[(\nu\|\tilde m_{t+1}\|+(1-\nu)\|\nabla F(x_t)\|)AL_0\sigma\right] + \eta AL_0\sigma^2(1-\beta)\left(\frac {\nu} {1+\beta}+1-\nu\right)\\
        &\quad +\frac 6 5 \gamma \mathbb E\left[BL_1\|\nabla F(x_t)\|\sigma\right]+\frac 1 5 \gamma \mathbb E\left[BL_1\|\nabla F(x_t)\|\sigma\right]
\end{aligned}
\end{equation}
where the first inequality uses the proof of Corollary \ref{DesCor} and  Lemma \ref{lemma_mom_clip_3}, and the last inequality uses $\gamma/\eta=5\sigma$. By taking summation of the above inequality we obtain
\begin{equation}
\begin{aligned}
\label{lemma_mom_clip_stochastic_z}
-\sum_{t=0}^{T-1}\mathbb E\left\langle\nabla F(x_{t}), m_{t+1}-\tilde m_{t+1}\right\rangle
&\le \frac {\beta}{1-\beta}\sum_{t=0}^{T-1}\left(\eta AL_0+\frac 7 5 \gamma BL_1\right)\sigma \|\nabla F(x_t)\|\\
    &\quad+ \eta AL_0\sigma^2\beta\left(\frac {\nu}{1+\beta}+1-\nu\right)T+\frac {\nu\beta^2}{(1-\beta)^2}\eta AL_0\sigma\|\nabla F(x_0)\|
\end{aligned}
\end{equation}
where we uses the following inequality to convert $\|\tilde m_{t+1}\|$ to $\|\nabla F(x_t)\|$.
\begin{equation}
\label{lemma_mom_clip_stochastic_w}
\begin{aligned}
    \sum_{t=0}^{T-1}\|\tilde m_{t+1}\|&\le \frac {\beta}{1-\beta}\|\nabla F(x_0)\|+(1-\beta)\sum_{t=0}^{T-1}\sum_{\tau=0}^t \beta^{t-\tau}\|\nabla F(x_{\tau})\|\\
    &\le \frac {\beta}{1-\beta}\|\nabla F(x_0)\|+\sum_{t=0}^{T-1}\|\nabla F(x_t)\|
\end{aligned}
\end{equation}
Combining \eqref{lemma_mom_clip_stochastic_0}, \eqref{lemma_mom_clip_stochastic_x}, \eqref{lemma_mom_clip_stochastic_y}, \eqref{lemma_mom_clip_stochastic_z}, using inequality \eqref{lemma_mom_clip_stochastic_w} to get rid of the term $\|\tilde m_{t}\|$ and applying Lemma \ref{lemma_mom_clip_3}, we obtain
\begin{equation}
\label{lemma_mom_clip_stochastic_3}
\begin{aligned}
    &\quad \mathbb E\sum_{t=0}^{T-1}(G(x_{t+1},\tilde m_{t+1})-G(x_t,\tilde m_t))\\
    &\le \frac {12\gamma} {5\beta(1-\beta)}\|\tilde \delta_0\|+\frac {\nu\beta}{(1-\beta)^2}AL_0\eta^2\sigma\|\nabla F(x_0)\|+\left(\frac {12} {5(1-\beta)}AL_0+\frac {12\gamma}{5\eta(1-\beta)}BL_1 + \frac 1 2 AL_0\right)\gamma^2 T_{\mathcal{S}}+\\
        &\quad\gamma\mathbb E\sum_{t \in \mathcal S} \left[-\left(\left(\frac 4 5\times \frac 2 5-\frac {\eta}{\gamma}\sigma\right)\|\nabla F(x_t)\|+\left(\frac {16}{25}\times \frac {3\gamma}{5\eta}-\frac 7 5 \sigma\right)\right)+\frac {\gamma} 2 BL_1 \|\nabla F(x_t)\| +\frac {12\gamma} {5(1-\beta)} BL_1 \|\nabla F(x_t)\|\right] \\
        &\quad+\mathbb E\sum_{t\in\overline{\mathcal S}}\left(-\frac {\eta} 2 (1-\nu\beta)\|\nabla F(x_t)\|^2-\frac {\eta} 2 \nu\beta\|m_t\|^2+\frac {\gamma^2} 2 BL_1 \|\nabla F(x)\|\right) \\
        &\quad+\mathbb E \sum_{t=0}^{T-1}\eta^2 AL_0\sigma\left(\|\nabla F(x_t)\|+\left(\frac{(1-\nu\beta)^2}{2}+\frac {1-\beta}{2(1+\beta)} \nu^2\beta^2\right) \sigma\right)\\
        &\quad +\frac{\nu\beta\eta\sigma}{1-\beta}\mathbb E\left( \sum_{t=0}^{T-1}(AL_0\eta\|\nabla F(x_{t})\|+\frac 7 5 BL_1\gamma\|\nabla F(x_t)\|\right)+ AL_0\eta^2 \sigma^2\nu\beta\left(\frac {\nu}{1+\beta}+1-\nu\right)T\\
    &=P_0+\mathbb E\left(P_1 T_{\mathcal S} + P_2(T-T_{\mathcal S})+ \sum_{t\in \mathcal S} P_3 \|\nabla F(x_t)\|+ \sum_{t\in \overline{\mathcal S}} P_4 \|\nabla F(x_t)\|\right)\\
        &\qquad -\mathbb E\sum_{t\in \overline{\mathcal S}}\frac {\eta} 2 \left((1-\beta)\|\nabla F(x_t)\|^2+\beta\|\tilde m_t\|^2\right)
\end{aligned}
\end{equation}
where
\begin{align*}
    P_0&=\frac {12\gamma} {5\beta(1-\beta)}\|\tilde \delta_0\|+\frac {\nu\beta}{(1-\beta)^2}AL_0\eta^2\sigma\|\nabla F(x_0)\|=\frac {\nu\beta}{(1-\beta)^2}AL_0\eta^2\sigma\|\nabla F(x_0)\|\\
    P_1&=-\frac {16}{25}\times \frac {3\gamma^2}{5\eta}+\left(\frac {12\gamma^2} {5(1-\beta)} +\frac {\gamma^2} 2 \right)AL_0+\frac {12\gamma^3}{5\eta(1-\beta)}BL_1 +\frac 7 5 \gamma \sigma +P_2\\
    P_2&=AL_0\eta^2\sigma^2\left(\frac{(1-\nu\beta)^2}{2}+\frac {1-\beta}{2(1+\beta)} \nu^2\beta^2\right)+AL_0\eta^2 \sigma^2\nu\beta\left(\frac {\nu}{1+\beta}+1-\nu\right)=\frac 1 2 \eta^2 AL_0\sigma^2\\
    P_3&=-\frac 4 5 \times \frac 2 5 \gamma +\eta\sigma+\left(\frac {\gamma^2} 2+\frac {12\gamma^2}{5(1-\beta)}+\frac {\nu\beta \eta\sigma}{1-\beta}\times \frac 7 5\gamma\right)BL_1+\eta^2 AL_0\sigma+\frac {\nu\beta\sigma}{1-\beta}AL_0\eta^2\\
    P_4&=\eta^2 AL_0\sigma+\frac {\nu\beta\sigma}{1-\beta}\left(AL_0\eta+\frac 7 5BL_1\gamma\right)\eta+\frac {\gamma^2} 2 BL_1
\end{align*}
Let $\gamma\le\frac {\epsilon}{2\sigma}\min\left(\frac {\epsilon}{AL_0},\frac {1-\beta}{AL_0},\frac {1-\beta}{25 BL_1}\right)$, and fix the ratio $\gamma/\eta=5\sigma$. Then for small enough $\epsilon<0.1$ and large enough noise $\sigma>1$,
\begin{equation}
\label{lemma_mom_clip_stochastic_5}
\begin{aligned}
    P_1&\le \left(-\frac {16}{25}\times 3\sigma+\frac {3\epsilon}{2\sigma}+\frac {12\epsilon} {50}+\frac 7 5 \sigma+\frac {\epsilon^2}{100\sigma} \right)\gamma\le -\frac 3 {10}\sigma\gamma\\
    P_3 &\le \left(-\frac 4 5\times \frac 2 5+\frac 1 {5} +\left(\frac {1-\beta} 2+\frac {12} 5 +\frac 7 5\times \frac {\beta} 5\right)\frac {\epsilon} {50\sigma} +\frac {\epsilon^2} {50\sigma^2} +\frac {\epsilon} {50\sigma^2}\right)\gamma\le -\frac 1 {10} \gamma
\end{aligned}
\end{equation}
We can also bound $P_4$ as follows:
\begin{equation}
\begin{aligned}
    P_4&\le\frac 1 {1-\beta} AL_0\sigma \eta^2+\left(\frac {\beta}{1-\beta}\times\frac 7 5+\frac 5 2\right)BL_1\sigma\gamma\eta\\
    &\le \frac 1 {10} \epsilon\eta+\left(\frac {\beta}{1-\beta}\times \frac 7 5+\frac 5 2\right)\frac {\epsilon} {50} (1-\beta)\eta\\
    &\le \frac 1 {10} \epsilon\eta+\frac 1 {20} \epsilon\eta = \frac{3}{20} \epsilon\eta
\end{aligned}
\end{equation}
Applying the above estimates and rearranging \eqref{lemma_mom_clip_stochastic_3}, we have
\begin{equation}\begin{aligned}
\label{rearrange}
&\quad G(x_0)-F^* +P_0\\
&\geq \mathbb{E} \left[ \sum_{t \in \mathcal{S}} \left( \frac 3 {10}\sigma \gamma + \frac 1{10}\gamma \|\nabla F(x_t)\| \right) + \sum_{t \in \overline{\mathcal{S}}} \left( \frac{\eta}{2} \left( (1-\nu\beta)\|\nabla F(x_t)\|^2 + \nu\beta \|m_t\|^2 \right) - \frac{AL_0}{2}\sigma^2\eta^2 - \frac{3}{20} \epsilon\eta \|\nabla F(x_t)\|\right) \right] \\
&\geq \mathbb{E} \left[ \sum_{t \in \mathcal{S}} \left(\frac 3 {10} \sigma \gamma + \frac{1}{10}\gamma \|\nabla F(x_t)\| \right) + \sum_{t \in \overline{\mathcal{S}}} \left( \frac{\eta}{2} \left( (1-\nu\beta)\|\nabla F(x_t)\|^2  \right) - \frac{AL_0}{2}\sigma^2\eta^2 - \frac{3}{20} \epsilon\eta \|\nabla F(x_t)\|\right) \right] \\
&\quad + \frac{1}{2}\eta \nu\beta \mathbb{E} \left[ \sum_{t \in \overline{\mathcal{S}}} \left( 2\epsilon \|\tilde{m_t}\| - \epsilon^2 \right) \right] \\
\end{aligned}\end{equation}

Due to Lemma \ref{lemma_convert_m_to_grad} ($AL_0\eta\sigma\le \frac {\epsilon} {10} (1-\beta), BL_1\gamma \le \frac {\epsilon} {50}(1-\beta)$), we clearly have
\begin{equation}
\begin{aligned}
\label{sigma_mt}
    &\quad \mathbb E \sum_{t \in \overline{\mathcal{S}} }\|\tilde m_t\|\\
    &\ge \left(1-\frac {\epsilon} {10}\right)\mathbb{E}\left[ \sum_{t \in \overline{\mathcal{S}} } \left(\left(1-\frac {\epsilon} {5}\right)\|\nabla F(x_t)\| - \frac {\epsilon} {10} \right)\right]
        -\mathbb E \left[\sum_{\tau\in S} \left(\frac {\gamma}{1-\beta} (AL_0+BL_1\|\nabla F(x_{\tau})\|)\right) \right]\\
    &\ge \left(1-\frac 3 {10}\epsilon\right)\mathbb{E}\left[ \sum_{t \in \overline{\mathcal{S}} } \left(\| \nabla F(x_t)\| \right)\right]- \frac {\epsilon} {10}(T-T_{\mathcal S})  -\mathbb E \left[\sum_{\tau\in S} \left(\frac {\gamma}{1-\beta} (AL_0+BL_1\|\nabla F(x_{\tau})\|)\right) \right]
\end{aligned}
\end{equation}
Define 
\begin{equation}
    \begin{aligned}
    U(x) &:= \left( \frac{1}{10}\gamma - \dfrac{\nu\beta}{1-\beta} BL_1 \epsilon\gamma\eta\right) \|\nabla F(x)\| + \left( \frac 3 {10}\sigma\gamma - \dfrac{\nu\beta}{1-\beta}AL_0\epsilon\gamma\eta \right) \\
    V(x) &:= \dfrac{1}{2}\eta (1-\nu\beta )\|\nabla F(x)\|^2 + \left( \dfrac{19}{20}\nu\beta\eta\epsilon - \dfrac{3}{20}\epsilon\eta\right) \|\nabla F(x)\| - \left( \dfrac{1}{2}AL_0\sigma^2\eta^2 + \dfrac{1}{2}\nu\beta\epsilon^2\eta + \dfrac{1}{10}\nu\beta\epsilon^2\eta \right)
    \end{aligned}
\end{equation}
Plugging \eqref{sigma_mt} into \eqref{rearrange}, we obtain
\begin{equation}\begin{aligned}
\label{min_UV}
&\quad G(x_0)-F^*+P_0 \geq \mathbb{E} \left[ \sum_{t \in \mathcal{S}} U(x_t) + \sum_{t \in \overline{\mathcal{S}}} V(x_t) \right] \\
&= \mathbb{E} \left[ \sum_{t=1}^T \left( \mathbb{I}_{t \in \mathcal{S}} U(x_t) + \mathbb{I}_{t \in \overline{\mathcal{S}}} V(x_t) \right) \right] 
\geq \mathbb{E} \left[ \sum_{t=0}^{T-1} \min \{ U(x_t), V(x_t) \}\right]
\end{aligned}\end{equation}

Since 
\begin{equation}
    \begin{aligned}
    \label{Uineq}
    U(x) &\geq \left( \dfrac{1}{10} - \dfrac{\nu\beta\epsilon^2}{50\sigma^2} \right)\gamma \|\nabla F(x)\| + \left( \frac 3 {10}\sigma\gamma - \dfrac{1}{10\sigma^2}\nu\beta\epsilon^2\gamma  \right) \\
    &\geq \dfrac{1}{20}\gamma \|\nabla F(x)\| + \frac 1 5\sigma\gamma
    \end{aligned}
\end{equation}
\begin{equation}
    \begin{aligned}
    \label{Vineq}
    V(x) &\geq \dfrac{1}{2}\eta (1-\nu\beta )\|\nabla F(x)\|^2 + \left( \dfrac{19}{20}\nu\beta\eta\epsilon - \dfrac{3}{20}\epsilon\eta\right) \|\nabla F(x)\| - \left(\frac 1 {20}+\frac {3} 5\nu\beta\right)\epsilon^2\eta \\
    &\geq \dfrac{1}{2}\eta (1-\nu\beta ) \left( 2\epsilon \|\nabla F(x)\|-\epsilon^2 \right) + \left( \dfrac{19}{20}\nu\beta\eta\epsilon - \dfrac{3}{20}\epsilon\eta\right) \|\nabla F(x)\| -\left(\frac 1 {20}+\frac {3} 5\nu\beta\right)\epsilon^2\eta\\
    &\geq \dfrac{4}{5}\epsilon\eta\|\nabla F(x)\| -\frac 4 5 \epsilon^2\eta
    \end{aligned}
\end{equation}
It clearly follows that $\min\{U(x), V(x)\}\ge \frac{4}{5}\epsilon\eta\|\nabla F(x)\| -\frac 4 5 \epsilon^2\eta$. Therefore
\begin{equation}
    G(x_0)-F^*+P_0 \geq \dfrac{4}{5}\epsilon\eta\mathbb E\sum_{t=0}^{T-1}(\|\nabla F(x)\| - \epsilon)
\end{equation}
Therefore, as long as $T > \dfrac{5}{4\epsilon^2\eta} \left( G(x_0)-F^*+P_0 \right)$, we have
$\dfrac{1}{T} \mathbb{E} \left[ \sum_{t=1}^T \|\nabla F(x_t)\| \right] < 2\epsilon$. \par
We finally show $G(x_0)-F^*+P_0=O(F(x_0)-F^*)$. Using Lemma \ref{GradNormBd}, 
\begin{equation}
    \frac 1{1-\beta}\min\left(\gamma\|m_0\|,\eta\|m_0\|^2\right)\le \frac 1 {50}\min \left(\frac {\|\nabla F(x_0)\|} {L_1},\frac{\|\nabla F(x_0)\|^2} {L_0}\right)\le\frac 8 {50}(F(x_0)-F^*)
\end{equation}
For the term $P_0$, if $\|\nabla F(x_0)\|=\Omega(L_0/L_1)$, we can similarly use Lemma $\ref{GradNormBd}$ to obtain $P_0=\mathcal{O}(F(x_0)-F^*)$. If $\|\nabla F(x_0)\|=\mathcal{O}(L_0/L_1)$, using $L_0\|\nabla F(x_0)\|\le L_1\|\nabla F(x_0)\|^2$ and Lemma  \ref{GradNormBd} leads to the result.
\end{proof}

\section{Discussion of the normalized momentum algorithm}
\label{appendix_normalized_momentum}
In this section we analyze in detail the theoretical aspects of the normalized momentum algorithm, as well as some practical issues. Recall that this algorithm can be seen as a special case of our clipping framework. For convenience we re-write it in Algorithm \ref{SNM}.

\setcounter{algocf}{1}

\begin{algorithm}[!htbp]
\SetKwInOut{KIN}{Input}
\caption{The Stochastic Normalized Momentum Algorithm(SNM)}
\label{SNM}
\KIN{Initial point $x_0$, initial momentum $m_0$, the learning rate $\eta$, momentum factor $\beta$ and the total number of iterations $T$ }
\For{$i \gets 1$ \textbf{to} $T$}{
	$m_t \gets \beta m_{t-1} + (1-\beta )\nabla f(x_{t-1},\xi_{t-1} )$\;
	$x_{t} \gets x_{t-1} - \eta \dfrac{m_t}{\|m_t\|}$\;
}
\end{algorithm}

We remark that SNM is different from the clipping methods in traditional sense, in that it makes a \textit{normalized} update each iteration. This algorithm has been analyzed in \citet{cutkosky2020momentum} for $L$-smooth functions. In that setting they were able to prove that SNM achieves a complexity of $\mathcal{O}(\Delta L \sigma^2 \epsilon^{-4})$. 

For $(L_0,L_1)$-smooth functions, we show that: \textbf{\textit{(a). }}With carefully chosen momentum parameter $\beta$ and step size $\eta$, SNM can achieve a complexity of $\mathcal{O}(\Delta L_0 \sigma^2 \epsilon^{-4})$, which is the same as the complexity we obtain in Theorem 3.2. \textbf{\textit{(b). }}There are some practical issues that make SNM less favorable than traditional clipping methods (such as the other three special cases of our framework discussed in Section 3 of the main paper).

The following results provides convergence guarantee for Algorithm \ref{SNM}.
\begin{lemma}
\label{SNMLemma}
Consider the algorithm that starts at $x_0$ and make updates $x_{t+1} = x_t - \eta m_{t+1}$. Define $\delta_t := m_{t+1} - \nabla F(x_t)$ be the estimation error. Assume $\eta\le c/L_1$ for some $c>0$ and let constants $A=1+e^c-\frac {e^c-1} c,B=\frac {e^c-1} c$. Then 
$$ F(x_{t+1})-F(x_t) \leq-\left( \eta - \frac{1}{2}BL_1 \eta^2 \right) \| \nabla F(x_t) \| +\frac 1 2 AL_0\eta^2 + 2 \eta \| \delta_t \| $$
And thus, by a telescope sum we have
$$\left( 1-\dfrac{1}{2}B L_1 \eta \right) \sum_{t=0}^{T-1} \| \nabla F(x_t) \| \leq \dfrac{F(x_0)-F(x_T)}{\eta} + \frac 1 2 A L_0T \eta + 2 \sum_{t=0}^{T-1} \| \delta \| $$
\end{lemma}
\begin{proof}
Since $\|x_{t+1}-x_t\| = \eta_t$, by Lemma \ref{DesIneq} we have
\begin{equation}\begin{aligned}
F(x_{t+1})-F(x_t) &\leq - \dfrac{\eta}{\| m_{t+1} \|} \left\langle \nabla F(x_t), m_{t+1} \right\rangle + \dfrac{1}{2}\eta^2 \left(AL_0+BL_1 \|\nabla F(x_t)\| \right) \\
&\leq \eta \left( -\|\nabla F(x_t)\|+2\|\delta_t\| \right) + \dfrac{1}{2}\eta^2 \left(AL_0+BL_1 \|\nabla F(x_t)\| \right)  \\
&\leq -\left( \eta - \frac{1}{2}BL_1 \eta^2 \right) \| \nabla F(x_t) \| +\frac 1 2 AL_0\eta^2 + 2 \eta \| \delta_t \| \notag
\end{aligned}\end{equation}
where in the second inequality we use Lemma \ref{lemma_mom_clip_1}.
\end{proof}
\begin{theorem}
\label{theorem1}
Suppose that Assumptions 1,2 and 4 holds, and $\Delta = F(x_0) - F^*$ where $F^* = \inf_{x \in \mathbb{R}^d} F(x)$. Let $m_0=\nabla F(x_0)$ in Algorithm \ref{SNM} for simplicity, and denote $\alpha = 1-\beta$. If we choose $\eta = \Theta\left(\min(L_1^{-1}, L_0^{-1}\epsilon)\alpha\right)$ and $\alpha = \Theta\left( \sigma^{-2}\epsilon^{2} \right)$, then  as long as $\epsilon = \mathcal{O}\left( \min \left\{ \frac{L_0}{L_1},\sigma\right\}\right)$, we have
$$\dfrac{1}{T} \sum_{t=0}^{T-1} \mathbb{E} \left[ \|\nabla F(x_t)\| \right] \leq \epsilon$$
holds in $T = \mathcal{O}(\Delta L_0 \sigma^2 \epsilon^{-4})$ iterations.
\end{theorem}

\begin{proof}
Define the estimation errors $\delta_t := m_{t+1} - \nabla F(x_t)$. Denote $S(a,b) := \nabla F(a) - \nabla F(b)$, then for $a,b$ such that $\|a-b\|=\eta \leq c/L_1$, we can upper bound $S(a,b)$ using Corollary \ref{DesCor}:
\begin{equation}
\label{Sab}
\|S(a,b)\| \leq \eta \left( AL_0+BL_1 \|\nabla F(b) \| \right)
\end{equation}
We can use $S(a,b)$ to get a recursive relationship:
\begin{equation}
\begin{aligned}
    \delta_{t+1}&=\beta m_{t+1}+(1-\beta)\nabla f(x_{t+1},\xi_{t+1})-\nabla F(x_{t+1})\\
    &= \beta S\left(x_{t}, x_{t+1}\right)+\beta \delta_t +(1-\beta)(\nabla f(x_{t+1},\xi_{t+1})-\nabla F(x_{t+1}))
\end{aligned}
\end{equation}
Denote $\delta_t'=\nabla f(x_t,\xi_t)-\nabla F(x_t)$, then
$$\delta_{t}=\beta \sum_{\tau=0}^{t-1}\beta^{\tau} S\left(x_{t-\tau-1}, x_{t-\tau}\right)+(1-\beta) \sum_{\tau=0}^{t-1}\beta^{\tau}  \delta_{t-\tau}'+(1-\beta)\beta^t \delta_0'$$
Using triangle inequality and plugging in the estimate \eqref{Sab} , we have
\begin{equation}
\label{snm_x}
    \left\|\delta_{t}\right\| \leq (1-\beta)\left\|\sum_{\tau=0}^{t}\beta^{\tau} \delta_{t-\tau}'\right\|+\beta \eta \sum_{\tau=0}^{t-1}\beta^{\tau}\left(AL_0+BL_1 \|\nabla F(x_{t-\tau-1}) \| \right)
\end{equation}
Taking a telescope summation of \ref{snm_x} and using Assumption 2.4 we obtain
\begin{equation}\begin{aligned}
\mathbb{E} \left[ \sum_{t=0}^{T-1} \| \delta_t \| \right]
&\leq   T (1-\beta) \sqrt{\sum_{\tau=0}^{+\infty}\beta^{2 \tau} \sigma^{2}}+\dfrac{AL_0\eta T}{1-\beta} + \dfrac{BL_1\eta}{1-\beta} \sum_{t=0}^{T-1} \mathbb{E}\left[\left\|\nabla F\left(x_{t}\right)\right\|\right] \\
&\leq \sqrt{\alpha}T\sigma +\dfrac{ATL_0\eta}{\alpha} + \dfrac{BL_1\eta}{\alpha} \sum_{t=0}^{T-1} \mathbb{E}\left[\left\|\nabla F\left(x_{t}\right)\right\|\right] 
\end{aligned}\end{equation}
Now we use Lemma \ref{SNMLemma}:
$$\left( 1-\left(\dfrac{1}{2}+\frac 2 {\alpha}\right) BL_1 \eta \right) \mathbb E\sum_{t=0}^{T-1} \| \nabla F(x_t) \| \leq \dfrac{\Delta}{\eta} + \frac 1 2 AL_0T \eta + 2\left(  \sqrt{\alpha}T\sigma +\dfrac{AL_0\eta T}{\alpha} \right)$$
If we choose $\eta = \Theta\left(\min(L_1^{-1}, L_0^{-1}\epsilon)\alpha\right)$ and $\alpha = \Theta\left( \sigma^{-2}\epsilon^{2} \right)$, then
$$\left( 1-\left(\dfrac{1}{2}+\frac 2 {\alpha}\right) BL_1 \eta \right) =\Theta(1)$$
In this case
$$\frac 1 T \mathbb E\sum_{t=0}^{T-1} \| \nabla F(x_t) \| =\mathcal O \left(\dfrac{\Delta}{\eta T} + \frac 1 2 AL_0 \eta  + \sqrt{\alpha}\sigma +\dfrac{AL_0\eta }{\alpha} \right)=\mathcal O\left(\frac{\Delta}{\eta T}  +\epsilon\right)$$
Therefore for $T=\Theta\left(\frac{\Delta}{\eta \epsilon}\right)$, we have $\frac 1 T \mathbb E\sum_{t=0}^{T-1} \| \nabla F(x_t) \| =\mathcal O(\epsilon)$. If $\epsilon=\mathcal O(L_0/L_1)$, then $\frac{\Delta}{\eta \epsilon}$ reduces to $\Delta L_0\sigma^2\epsilon^{-4}$.
\end{proof}

We have shown the theoretical superiority of Algorithm \ref{SNM}. Specifically, it enjoys the same complexity as Theorem 3.2. However we notice some potential drawbacks of Algorithm \ref{SNM}:
\begin{itemize}
    \item \textit{Firstly,} the step size of Algorithm \ref{SNM} is at the order of $\mathcal{O}\left( \epsilon^{3} \right)$, while the step size we chose in Theorem 3.2 is $\mathcal{O}\left( \epsilon^{2} \right)$. Previous works have noticed that a smaller step size makes it easier to be trapped in a sharp local minima , which may result in worse generalization \citep{kleinberg2018alternative}.
    \item \textit{Secondly,} although the complexity of Algorithm \ref{SNM} is the same as Theorem 3.2 for small $\epsilon$, it requires a more restrictive upper bound of $\epsilon$ to ensure the $\epsilon^{-4}$ term dominates. For instance with a poor initialization, $\Delta$ may very large. This suggests that in practice, where we do not get into a very small neighbourhood of stationary point, the performance of Algorithm \ref{SNM} may be worse.
\end{itemize}

\section{Details of Lower Bounds in Section 3.3}
\label{appendix_lower_bound}
In this section we discuss the lower bound for SGD in \citet{drori2019complexity} in detail. The following result is taken from this paper:

\begin{theorem}
\label{lower_bound_SGD}
\textbf{[Theorem 2 in \citet{drori2019complexity}]} Consider a first-order method that given a function $F: \mathbb{R}^{d} \rightarrow \mathbb{R}$ and an initial point ${x}_{0} \in \mathbb{R}^{d}$ generates a sequence of points $\left\{{x}_{i}\right\}$ satisfying
\[
{x}_{t+1}={x}_{t}+\eta_{{x}_{0}, \ldots, {x}_{t}} \cdot\left(\nabla F\left({x}_{t}\right)+\xi_{t}\right), \quad t \in[T-1]
\]
where $\xi_{i}$ are some random noise vectors, and returns a point ${x}_{\text {out }} \in \mathbb{R}^{d}$ as a non-negative linear combination of the iterates:
\[
{x}_{{out}}=\sum_{t=0}^{T} \zeta_{{x}_{0}, \ldots, {x}_{T}}^{(t)} {x}_{t}
\]
We further assume that the step sizes $\eta_{{x}_{0}, \ldots, {x}_{t}}$ and aggregation coefficients $\zeta_{{x}_{0}, \ldots, {x}_{T}}^{(t)}$ are deterministic functions of the norms and inner products between the vectors ${x}_{0}, \ldots, {x}_{t}, \nabla F\left({x}_{0}\right)+$
$\xi_{0}, \ldots, \nabla F\left({x}_{t}\right)+\xi_{t} .$ Then for any $L, \Delta, \sigma > 0$ and $T \in \mathbb{N}$ there exists a function $F: \mathbb{R}^{d} \mapsto \mathbb{R}$
with $L-$Lipschitz gradient, a point ${x}_{0} \in \mathbb{R}^{d}$ and independent random variables $\xi_{t}$ with $\mathbb{E}\left[\xi_{t}\right]=0$
and $\mathbb{E}\left[\left\|\xi_{t}\right\|^{2}\right]=\sigma^{2}$ such that $\forall t \in[T]$
\[
\begin{array}{l}
F\left({x}_{0}\right)-F\left({x}_{t}\right) \stackrel{\text {a.s.}}{\leq} \Delta \\
\nabla F\left({x}_{t}\right) \stackrel{\text {a.s.}}{=} \gamma
\end{array}
\]
and in addition
\[
F\left({x}_{0}\right)-F\left({x}_{\text {out }}\right) \stackrel{a . s .}{\leq} \Delta+\frac{\sigma}{2 L} \sqrt{\frac{L \Delta}{T}}
\]
\[
\nabla F\left({x}_{\text {out }}\right) \stackrel{a . s .}{=} \gamma
\]
where $\gamma \in \mathbb{R}^{d}$ is a vector such that
\[
\|\gamma\|^{2}=\frac{\sigma}{2} \sqrt{\frac{L \Delta}{T}}
\]
\end{theorem}

Now we discuss why this shows the optimality of clipped SGD under Assumptions 2.1, 2.2 and 2.4. 

\textit{Firstly}, Theorem \ref{lower_bound_SGD} assumes an upper bound $\Delta$ on $F(x_0)-F(x_t)$ rather than the one assumed in Assumption 2.1 ($F(x_0)-F^*\le \Delta$). However, in fact we only need to assume that $F(x_0)-F(x_T) \leq \Delta$ to prove Theorem 3.2 for clipped SGD. The reason is as follows. In fact, since $\beta = 0$ for clipped SGD , the momentum term in the Lyapunov function disappears, as well as the term  $\frac{\nu\beta}{(1-\beta)^2}AL_0\eta^2\sigma \|\nabla F(x_0)\|$ in \eqref{lemma_mom_clip_stochastic_3}. So we no longer need to use Lemma \ref{GradNormBd} to bound the term $\|\nabla F(x_0)\|$. The rest of the proof only needs $F(x_0)-F(x_T) \leq \Delta$ (which is used in the telescope sum in \eqref{lemma_mom_clip_stochastic_3}).

\textit{Secondly}, although Theorem \ref{lower_bound_SGD} only assume that the variance of stochastic gradient is bounded, in their construction the noise is actually defined as
\begin{equation}P\left(\boldsymbol{\xi}_{t}=\pm \sigma \mathbf{e}_{t+1}\right)=\frac{1}{2}, \quad t \in[T-1]\end{equation}
Therefore the norm of the noise is bounded by $\sigma$, and the example used to prove Theorem \ref{lower_bound_SGD} still works under Assumption 2.4.

Now suppose we need an output such that $\|\nabla f(x_{\text{out}})\| = \|\gamma\|\leq \epsilon$, then it follows from Theorem \ref{lower_bound_SGD} that $T = \Omega \left( L\Delta\sigma^2\epsilon^{-4} \right)$. Therefore we have shown the optimality of clipped SGD in this class of algorithms, as stated in Section 3.3.

\section{Justifications on the Mixed Clipping}
\label{section_explanation}
We will show in this section that combining gradient and momentum can be better than using only one of them. We consider a basic optimization problem: $\min_{x\in \mathbb R} F(x)=\min_{x\in \mathbb R} \mathbb E_{\xi} [f(x,\xi)] $ where $f(x,\xi)=\frac 1 2 (x+\xi)^2$, and the noise $\xi\in \mathbb R$ follows the uniform distribution $U[-\sqrt 3, \sqrt 3]$ so that $\mathbb E[\xi^2]=1$. To simplify the analysis, we set $\gamma$ in Algorithm 1 to be sufficiently large such that clipping will never be triggered, since the function $F(x)=\frac 1 2 x^2$ is (1,0)-smooth.\par
In the above optimization problem, the general update formula can be written as:
\begin{equation}
\label{proposition_qhm}
\begin{aligned}
    m_{t+1}&=\beta m_t+(1-\beta)(x_t+\xi_t)\\
    x_{t+1}&=x_t-\nu\eta m_{t+1}-(1-\nu)\eta (x_t+\xi_t)
\end{aligned}
\end{equation}
We have the following proposition:
\begin{proposition}
\label{qhm_good}
Let $x_0,m_0\in \mathbb R$ be arbitrary real numbers. Let $\xi_i$s be i.i.d. random noises such that $\mathbb E[\xi_i^2]=1$. Let the sequence $\{x_t\}$ be defined in \eqref{proposition_qhm}, where $0<\eta< 1, 0\le \beta<1$ and $0\le\nu\le 1$ are constant hyper-parameters. Then in the limit
\begin{equation}
    \label{eq_x2_limit}
    \lim_{t\rightarrow \infty}\mathbb E[F(x_t)]=\frac {\eta} 2 \times \frac {(1+\beta)(1-\beta+\beta\eta)-\nu\eta\beta(1+3\beta-2\nu\beta)} {(2-\eta)(1+\beta)(1-\beta+\beta\eta)-\nu\eta\beta(4\beta-\eta-3\beta\eta+2\nu\eta\beta)}
\end{equation}
\end{proposition}

We now analyze three cases based on the proposition:
\begin{itemize}
    \item Only use gradient in an update. Set $\nu=0$ in \eqref{eq_x2_limit}, we obtain $\lim_{t\rightarrow \infty}\mathbb E[F(x_t)]=\frac {\eta}{4-2\eta}$.
    \item Only use momentum in an update. Set $\nu=1$ in \eqref{eq_x2_limit}, we obtain $\lim_{t\rightarrow \infty}\mathbb E[F(x_t)]=\frac {\eta}{4-2\eta\frac {1-\beta}{1+\beta}}$.
    \item Combine gradient and momentum in an update. It can be verified that for proper $0<\nu<1$, \eqref{eq_x2_limit} is less than $\frac {\eta}{4-2\eta\frac {1-\beta}{1+\beta}}$ (therefore less than $\frac {\eta}{4-2\eta}$). Furthermore, when $\beta\rightarrow 1$, a straightforward calculation shows that $\lim_{t\rightarrow \infty}\mathbb E[F(x_t)]\rightarrow\frac {\eta}{\frac 4 {1-\nu} -2\eta}$. Thus $\lim_{t\rightarrow \infty}\mathbb E[F(x_t)]$ can be arbitrarily close to zero if $\nu$ is close to 1. However, this does not happen in the previous two cases, where $\lim_{t\rightarrow \infty}\mathbb E[F(x_t)]$ there must be greater than $\frac {\eta} 4$.
\end{itemize}

We further plot the value of \eqref{eq_x2_limit} with respect to $\nu$ and $\beta$ in Figure \ref{fig_qhm_x2} to visualize the above finding. It can be clearly seen that the using both gradient and momentum with a proper interpolation factor $\nu$ outperforms both SGD and SGD with momentum by a large margin (Figure \ref{fig_qhm_x2}(a)). Furthermore, we can drive $\beta\rightarrow 1$ to further improve convergence (Figure \ref{fig_qhm_x2}(b)), while in SGD with momentum we can not. 

\begin{figure}
\centering
\subfigure[convergence value w.r.t. $\nu$]{
\includegraphics[width=0.35\textwidth]{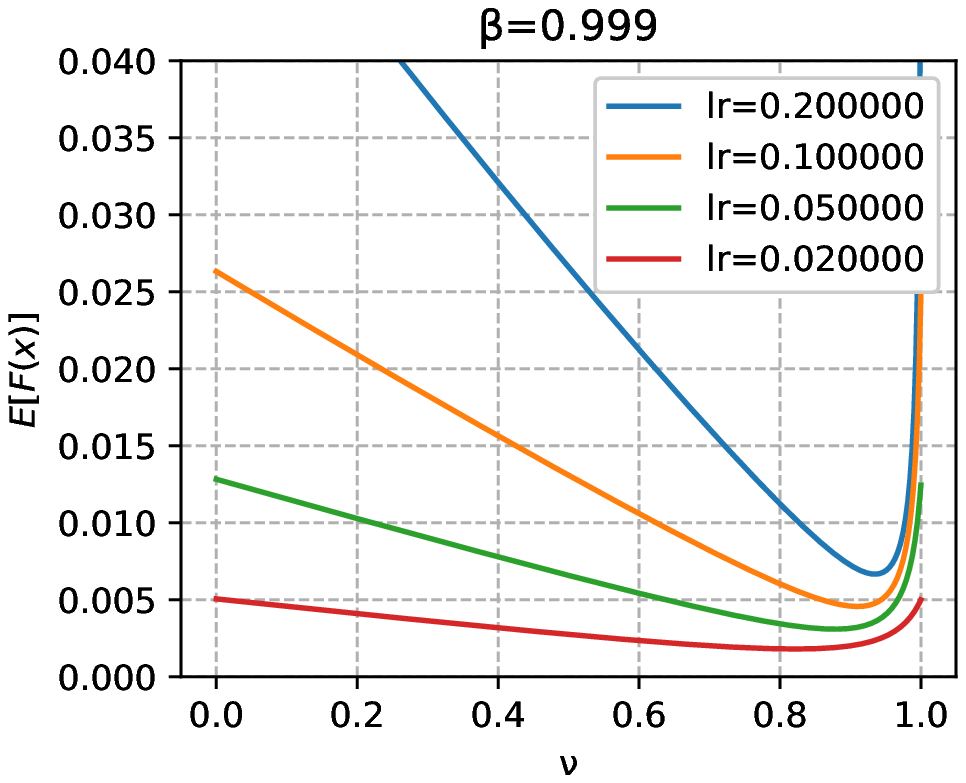}}
\hspace{20pt}
\subfigure[convergence value w.r.t. $\beta$]{
\includegraphics[width=0.35\textwidth]{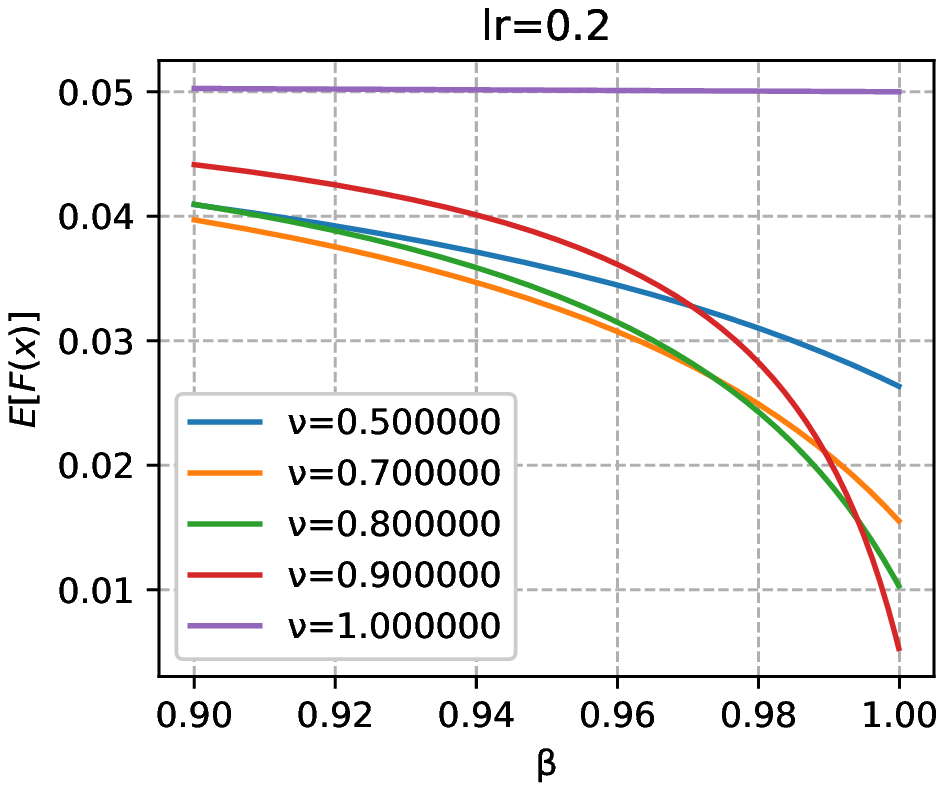}}
\caption{Convergence value of different hyper-parameters $\eta,\beta,\nu$ over stochastic function $f(x,\xi)=\frac 1 2 (x+\xi)^2$. The mixed update with proper $\nu$ outperforms both SGD and SGD with momentum by a large margin. Furthermore, for the mixed update we can drive $\beta\rightarrow 1$ to further improve convergence, while for SGD with momentum we can not.}
\label{fig_qhm_x2}
\end{figure}

Although we use the simple function $F(x)=\frac 1 2 x^2$ as an example, similar result exists in any general quadratic form with positive definite Hessian. Furthermore, the experiments in Section 4 also demonstrate that the mixed clipping outperforms both gradient clipping and momentum clipping.

\subsection{Proof of Proposition \ref{qhm_good}}
\subsubsection{Proof of a simple case}
For clarity, we first assume $\nu=1$. Consider a specific time step $t$. We first calculate $\mathbb E[m_t^2]$.
\begin{equation}
\label{mom_clip_lower_bound_0}
\begin{aligned}
    \mathbb E[m_{t+1}^2]&=\mathbb E[(\beta m_t+(1-\beta)(x_t+\xi_t))^2]\\
    &= \beta^2 \mathbb E [m_t^2]+(1-\beta)^2(\mathbb E[x_t^2]+\mathbb E[\xi_t^2])+2(1-\beta)^2\mathbb E[x_t\xi_t]+2\beta(1-\beta)(\mathbb E[m_tx_t]+\mathbb E[m_t\xi_t])\\
    &=\beta^2 \mathbb E [m_t^2]+(1-\beta)^2(\mathbb E[x_t^2]+1)+2\beta(1-\beta)\mathbb E[m_tx_t])
\end{aligned}
\end{equation}
where we use the fact that $\xi_t$ is independent with $x_t$ and $m_t$. We then calculate $\mathbb E[x_t^2]$.
\begin{equation}
\label{mom_clip_lower_bound_1}
\begin{aligned}
    \mathbb E[x_{t+1}^2]&=\mathbb E[(x_t-\eta m_t)^2]\\
    &= \mathbb E[x_t^2+\eta^2m_{t+1}^2-2\eta x_tm_{t+1}]\\
    &=(1+\eta^2(1-\beta)^2-2\eta(1-\beta))\mathbb E[x_t^2]+\eta^2\beta^2\mathbb E[m_t^2]+\eta^2(1-\beta)^2+2(\eta^2\beta(1-\beta)-\eta\beta)\mathbb E[m_tx_t]
\end{aligned}
\end{equation}
where in the last equation we use \eqref{mom_clip_lower_bound_0}. To complete the recursive relationship, we also need to calculate $\mathbb E[x_{t+1}m_{t+1}]$.
\begin{equation}
\label{mom_clip_lower_bound_2}
\begin{aligned}
    \mathbb E[x_{t+1}m_{t+1}]&=\mathbb E[m_{t+1}x_t-\eta m_{t+1}^2]\\
    &= \mathbb E[\beta m_tx_t+(1-\beta)x_t^2-\eta m_{t+1}^2]\\
    &=(1-\beta)(1-\eta(1-\beta))\mathbb E[x_t^2]-\eta\beta^2\mathbb E[m_t^2]-\eta(1-\beta)^2+(\beta-2\eta\beta(1-\beta))\mathbb E[m_tx_t]
\end{aligned}
\end{equation}
Combining \eqref{mom_clip_lower_bound_0}, \eqref{mom_clip_lower_bound_1} and \eqref{mom_clip_lower_bound_2}, we can write the recursive relationship into a matrix form:
\begin{equation}
    \left( {\begin{array}{*{20}{c}}
  {\mathbb E x_{t + 1}^2} \\ 
  {\mathbb E m_{t + 1}^2} \\ 
  {\mathbb E {x_{t + 1}}{m_{t + 1}}} \\ 
  1 
\end{array}} \right) = \left( {\begin{array}{*{20}{c}}
  {1+\eta^2(1-\beta)^2-2\eta(1-\beta)}&{\eta^2\beta^2}&{2(\eta^2\beta(1-\beta)-\eta\beta)}&{\eta^2(1-\beta)^2} \\ 
  {(1-\beta)^2}&{\beta^2}&{2\beta(1-\beta)}&{(1-\beta)^2} \\ 
  {(1-\beta)(1-\eta(1-\beta))}&{-\eta\beta^2}&{\beta-2\eta\beta(1-\beta)}&{-\eta(1-\beta)^2} \\ 
  {0}&{0}&{0}&{1} 
\end{array}} \right)\left( {\begin{array}{*{20}{c}}
  {\mathbb E x_t^2} \\ 
  {\mathbb E m_t^2} \\ 
  {\mathbb E {x_t}{m_t}} \\ 
  1 
\end{array}} \right)
\end{equation}
Denote the above matrix as $M$. After a straightforward calculation, we can find that $\lambda_1=1$ is an eigenvalue of $M$, and
$$u=\left(-\eta\frac {1+\beta}{1-\beta},-2,\eta, \eta-2\frac {1+\beta}{1-\beta}\right)^T$$
is the only eigenvector associated with $\lambda_1=1$. Similarly, $\lambda_2=\beta$ is also an eigenvalue of $M$. Let the other two eigenvalues be $\lambda_3$ and $\lambda_4$, then
\begin{equation*}
\begin{aligned}
    \lambda_1 \lambda_2 \lambda_3 \lambda_4&=\det M=\beta^3 \\
    \lambda_1 +\lambda_2 +\lambda_3+ \lambda_4&=\operatorname{tr} M = 1-\beta+(\eta(1-\beta)-\beta-1)^2
\end{aligned}
\end{equation*}
It follows that $\lambda_3\lambda_4=\beta^2$ and $\lambda_3+\lambda_4=(\eta(1-\beta)-\beta-1)^2-2\beta$. Since $(1+\beta-\eta(1-\beta))^2< (1+\beta)^2$, we have $\lambda_3+\lambda_4< 1+\beta^2$. Therefore $|\lambda_3|<1$ and $|\lambda_4|<1$ (note that $\lambda_3$ and $\lambda_4$ can be composite numbers). If $\eta<1$, we can further conclude that the four eigenvalues are different from each other (otherwise $\lambda_3=\lambda_4=\beta$, which contradicts to $\lambda_3+\lambda_4=(\eta(1-\beta)-\beta-1)^2-2\beta$).\par
Based on the above calculation, for any initial vector $v$, $\lim_{t\rightarrow\infty} M^t v$ converges to a vector proportional to $u$. In our case, $(\mathbb E x_0^2,\mathbb E m_0^2,\mathbb E x_0m_0,1)^T=(0,0,0,1)^T$, and we also know that the the last element of the vector $\lim_{t\rightarrow\infty} M^t (0,0,0,1)^T$ is 1. As a result,
$$\lim_{t\rightarrow\infty} M^t (0,0,0,1)^T=-\frac {1-\beta}{2(1+\beta)}u$$
Namely,$$\lim_{t\rightarrow\infty} \mathbb E [x_{t+1}^2]=\frac {\eta}{2-\eta\frac {1-\beta}{1+\beta}}$$

\subsubsection{Proof of the general case}
Now we prove Proposition \ref{qhm_good} for general $\nu$.
\begin{equation}
\label{qhm_clip_lower_bound_4}
\begin{aligned}
    \mathbb E[x_{t+1}^2]&=(1-\eta+\nu\eta\beta)^2\mathbb E[x_t^2]+\nu^2\eta^2\beta^2\mathbb E[m_t^2]-2\nu\eta\beta(1+\nu\eta\beta-\eta)\mathbb E[m_tx_t]+(\nu\eta(1-\beta)+(1-\nu)\eta)^2
\end{aligned}
\end{equation}

\begin{equation}
\label{qhm_clip_lower_bound_5}
\begin{aligned}
    \mathbb E[x_{t+1}m_{t+1}]&=(1-\eta+\nu\eta\beta)(1-\beta)\mathbb E[x_t^2]-\nu\eta\beta^2\mathbb E[m_t^2] + (1-\eta-\nu\eta+2\nu\eta\beta)\beta E[x_tm_t] \\
    &\quad- \nu\eta(1-\beta)^2-(1-\nu)\eta(1-\beta)
\end{aligned}
\end{equation}
Combining \eqref{mom_clip_lower_bound_0}, \eqref{qhm_clip_lower_bound_4} and \eqref{qhm_clip_lower_bound_5}, we obtain the following recursive matrix $M$:
\begin{equation*}
 M=\left( {\begin{array}{*{20}{c}}
  {(1-\eta+\nu\eta\beta)^2}&{\nu^2\eta^2\beta^2}&{-2\nu\eta\beta(1+\nu\eta\beta-\eta)}&{(\nu\eta(1-\beta)+(1-\nu)\eta)^2} \\ 
  {(1-\beta)^2}&{\beta^2}&{2\beta(1-\beta)}&{(1-\beta)^2} \\ 
  {(1-\eta+\nu\eta\beta)(1-\beta)}&{-\nu\eta\beta^2}&{(1-\eta-\nu\eta+2\nu\eta\beta)\beta}&{-\nu\eta(1-\beta)^2-(1-\nu)\eta(1-\beta)} \\ 
  {0}&{0}&{0}&{1} 
\end{array}} \right)
\end{equation*}
Using the same calculation as in the previous section, we finally get
\begin{equation}
\label{qhm_clip_lower_bound_6}
    \lim_{t\rightarrow\infty} \mathbb E[x_{t+1}^2]=\eta\frac {(1+\beta)(1-\beta+\beta\eta)-\nu\eta\beta(1+3\beta-2\nu\beta)} {(2-\eta)(1+\beta)(1-\beta+\beta\eta)-\nu\eta\beta(4\beta-\eta-3\beta\eta+2\nu\eta\beta)}
\end{equation}

\section{Soft Clipping}
\label{section_soft_clip}
\begin{algorithm}[!htbp]
\SetKwInOut{KIN}{Input}
\caption{The General Soft Clipping Framework}
\label{soft_clip}
\KIN{Initial point $x_0$, learning rate $\eta$, clipping parameter $\gamma$, momentum $\beta\in [0,1)$, interpolation parameter $\nu\in [0,1]$ and the total number of iterations $T$}
Initialize $m_0$ arbitrarily\;
\For{$t \gets 0$ {to} $T-1$}{
    Compute the stochastic gradient $\nabla f(x_t,\xi_t)$ for the current point $x_t$\;
	$m_{t+1} \gets \beta m_{t} + (1-\beta )\nabla f(x_t,\xi_t)$\;
	$x_{t+1} \gets x_t - \left[ \nu \eta\dfrac {m_{t+1}}{1+{\eta\|m_{t+1}\|}/{\gamma}} + (1-\nu )\eta \dfrac{\nabla f(x_t,\xi_t)}{1+{\eta\|\nabla f(x_t,\xi_t)\|}/{\gamma}} \right]$\;
}
\end{algorithm}
For Algorithm 1, as long as the norm of the gradient (or momentum) exceeds a constant, it is then clipped; we refer to this form of clipping as \textit{hard clipping}. One can also consider a \textit{soft} form of clipping, as presented in Algorithm \ref{soft_clip}.\par
\begin{figure}
\centering
\includegraphics[width=0.45\textwidth]{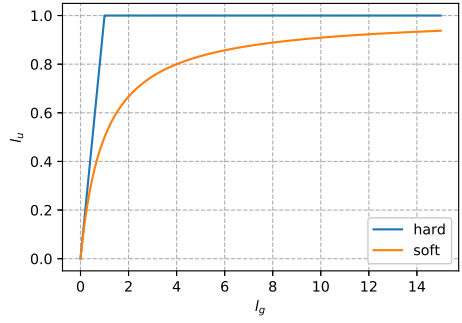}
\caption{The update norm $l_{\text{u}}$ w.r.t. the gradient norm $l_{\text{g}}$ for hard clipping and soft clipping ($\eta=1,\gamma=1$).}
\label{fig_soft_clip}
\end{figure}

We take $\nu=0$ for example to analyze soft clipping. For any gradient norm $l_{\text g}$, the norm of the update $l_{\text u}$ is a function of $l_{\text g}$:
\begin{equation}
\label{function_h}
    l_{\text u} = h_{\text{soft}}(l_{\text g})= \eta \frac{l_{\text g}}{1+{\eta l_{\text g}}/{\gamma}}
\end{equation}
For hard clipping, we can similarly write
\begin{equation}
\label{function_hhard}
    l_{\text u} = h_{\text{hard}}(l_{\text g})= \min(\eta l_g,\gamma)
\end{equation}
A straightforward calculation shows that
\begin{equation}
     \frac 1 2 \min(\eta l_g,\gamma)\le \eta \frac{l_{\text g}}{1+{\eta l_{\text g}}/{\gamma}}\le \min(\eta l_g,\gamma)
\end{equation}
Therefore soft clipping is in fact equivalent to hard clipping up to a constant factor 2 in the step size choice. Thus it's easy to see that our results also hold for Algorithm \ref{soft_clip}. However, compared to hard clipping, soft clipping has the advantage that the function $h_{\text{soft}}$ in \eqref{function_h} is smooth while $h_{\text{hard}}$ in \eqref{function_hhard} is not, as shown in Figure \ref{fig_soft_clip}. We also empirically observe that the training curve of soft clipping is more smooth than hard clipping.

\section{Experimental Details in Section 4}
\label{appendix_experiment}
Based on the discussion in Appendix \ref{section_soft_clip}, we use the soft version of clipping algorithms in all the experiments.

\subsection{CIFAR-10}
The CIFAR-10 dataset contains 50k images for training and 10k for testing. All the images are $32\times 32$ RGB bitmaps. We use the standard ResNet-32 architecture. The total number of parameters is 466,906. For all algorithms, we use mini-batch size 128 and weight decay $5\times 10^{-4}$. For the baseline algorithm, we use SGD with momentum using learning rate $lr=1.0$ and momentum factor $\beta=0.9$. Note that we use the momentum defined in Algorithm 1, which is equivalent to a Pytorch implementation with $lr=0.1$ and $\beta=0.9$. We optimize ResNet-32 for 150 epochs, and decrease the learning rate at epoch 80 and epoch 120. For other algorithms, we perform a course grid search for $lr$ an $\gamma$, while keeping all the training strategy the same as SGD. We use 5 random seeds ranging from 2016 to 2020, and the results are similar. The plot in Figure 2 uses the random seed 2020.\par

\subsection{PTB}
The Penn Treebank dataset has a vocabulary of size 10k, and 887k/70k/78k words for training/validation/testing. We use the state-of-the-art AWD-LSTM architecture using hidden size 1150 and embedding size 400. The total number of parameters is 23,941,600. For the baseline algorithm, we follow \citet{merity2017regularizing} who use averaged SGD clipping without momentum using learning rate $lr=30$ and $\gamma=7.5$. Note that here $\gamma=7.5$ means that the gradient norm will be clipped to be no more than 0.25. We use the same dropout rate and regularization hyper-parameters in \citep{merity2017regularizing}. We train AWD-LSTM for 250 epochs, and averaging is triggered when the validation perplexity stops improving. For other algorithms, we perform a course grid search for $lr$ an $\gamma$, while keeping all the training strategy the same as SGD clipping. We use 5 random seeds ranging from 2016 to 2020, and the results are similar. The plot in Figure 2 uses the random seed 2020.

\subsection{ImageNet}
We also conduct experiments on ImageNet dataset. This dataset contains about 1.28 million training images and 50k validation images with various sizes. We train the standard ResNet-50 architecture on this dataset. The total number of parameters is 25,557,032. We use a batch size of 256 on 4 GPUs and a weight decay of $10^{-4}$. For the baseline algorithm, we choose SGD with learning rate $lr=1.0$ and momentum $\beta=0.9$, following \citet{goyal2017accurate}. Note that we use the momentum defined in Algorithm 1, which is equivalent to a Pytorch implementation with $lr=0.1$ and $\beta=0.9$. We train the ResNet-50 for 90 epochs, and decrease the learning rate in epoch 30, epoch 60 and epoch 80. For the other algorithms, we perform a course grid search for $lr$ an $\gamma$, while keeping all the training strategy the same as SGD.\par

\section{Additional Experimental Results}
\label{appendix_experiment_other}
\textbf{Comparsion with Adam optimizer.} Adam is a popular optimizer to train neural networks on a variety of tasks. We also run the same experiments in Section \ref{section_experiments} using Adam optimizer with best hyper-parameters in order to compare it with clipping algorithms. We turn the learning rate of Adam based on a grid search, and choose the momentum hyper-parameters $\beta_1=0.9$ and $\beta_2=0.999$. The results are shown in Figure \ref{adam_cifar} and \ref{adam_lstm}.\par
It is clear from the figure that Adam trains really fast on both datasets. It is even faster than gradient clipping and momentum clipping. However, it does not outperform the mixed clipping method. Note that like clipping algorithms, Adam also uses adaptive gradient. However, Adam generalizes much worse than vanilla SGD or clipping algorithms. Particularly, the test accuracy of CIFAR-10 using Adam is only 91.5\%, while all other algorithms reach 93\% test accuracy.

\textbf{CIFAR-10 classificatoin with ResNet-18.} The original work of gradient clipping in \cite{zhang2019gradient} conducts experiments using ResNet-18 on CIFAR-10. To compare with it, we also run the same experiments using different algorithms, e.g. SGD, SGD clipping, momentum clipping and mixed clipping. All the algorithms reach 95\% test accuracy, and the result is similar to that of using ResNet-32 architeture (see Figure \ref{resnet18}).

\begin{figure}
  \centering
  \subfigure[CIFAR-10]{
  \includegraphics[width=0.31\textwidth]{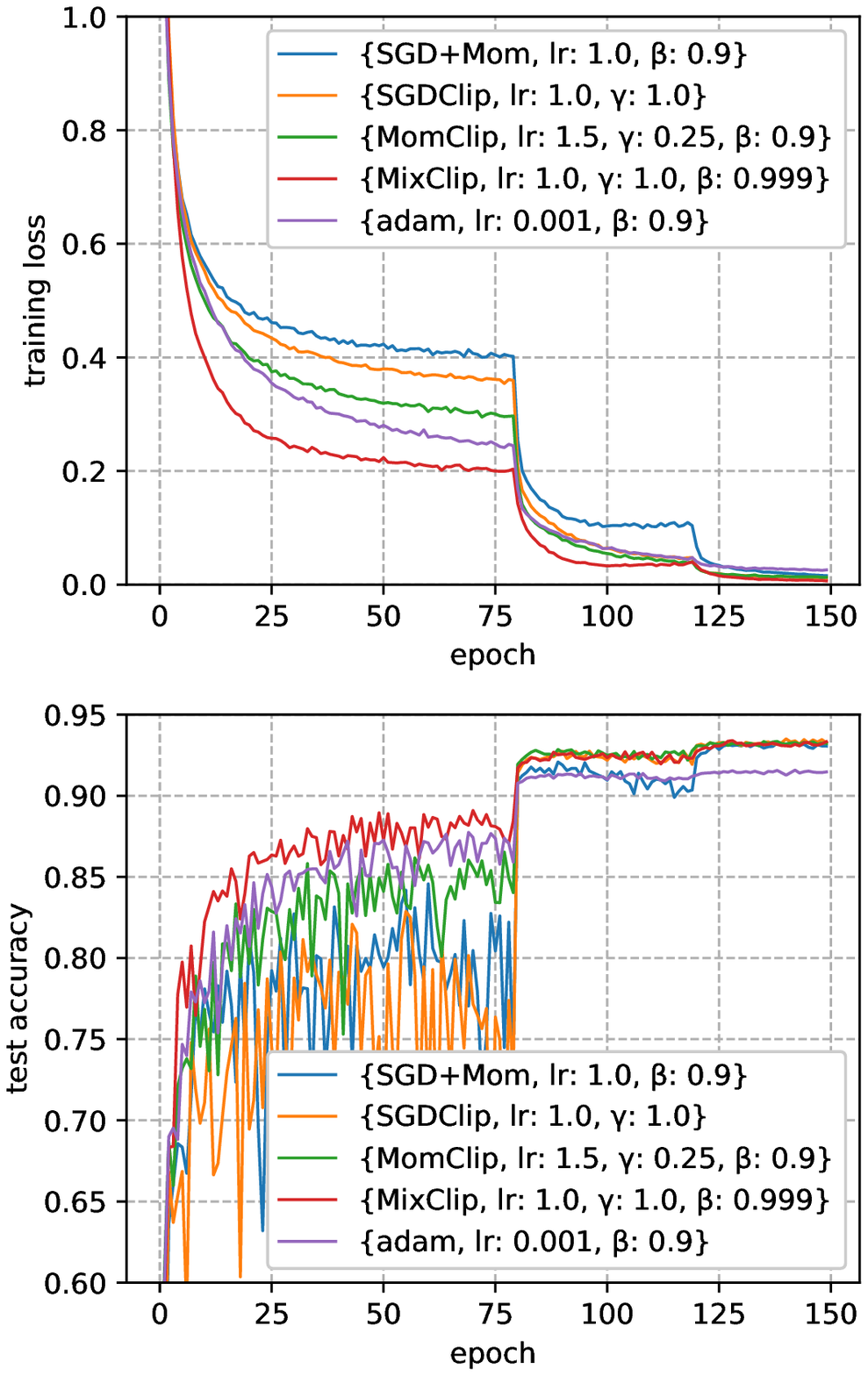}
  \label{adam_cifar}
  }
  \subfigure[PTB]{
  \includegraphics[width=0.31\textwidth]{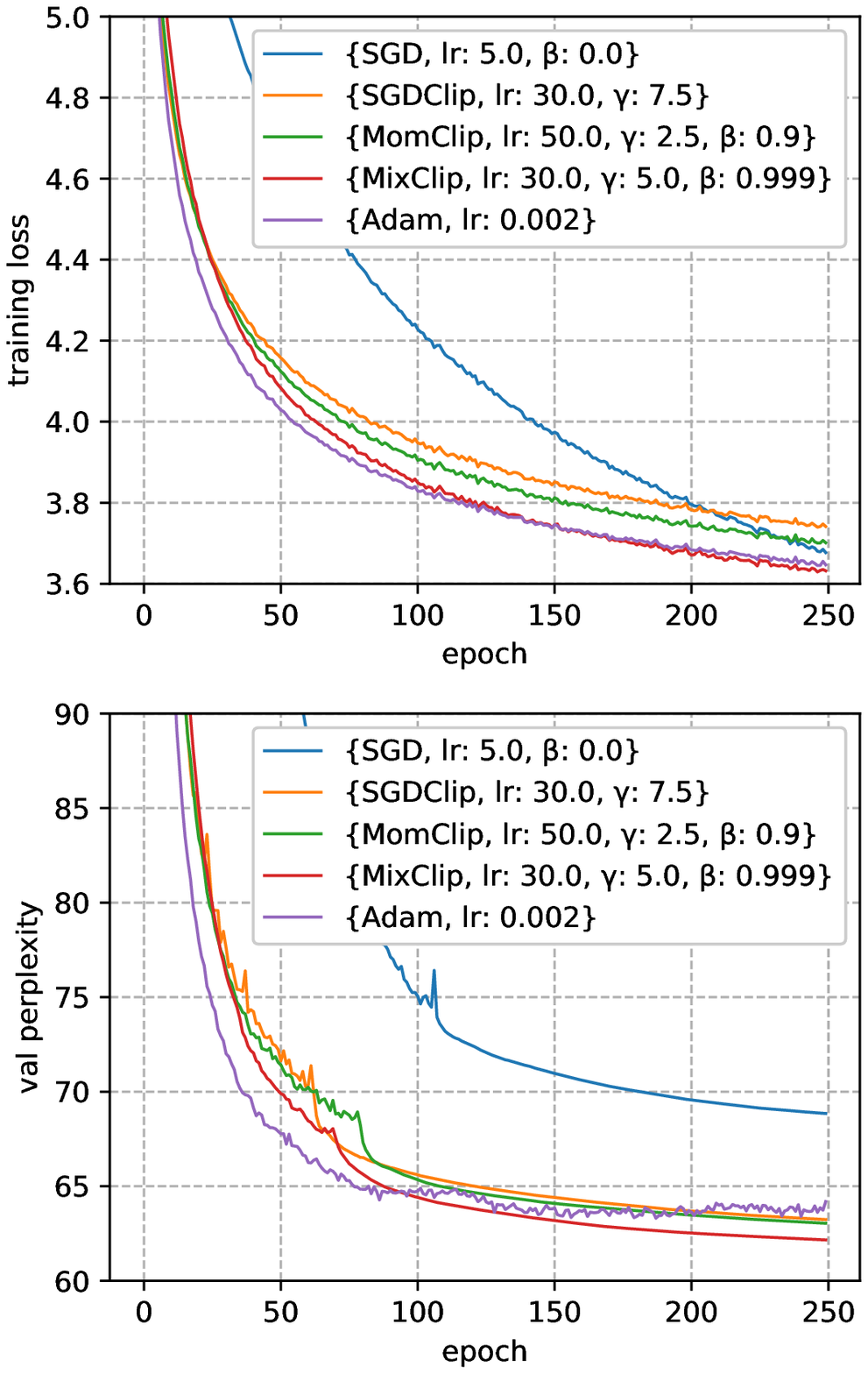}
  \label{adam_lstm}
  }
  \subfigure[CIFAR10 (ResNet-18)]{
  \includegraphics[width=0.31\textwidth]{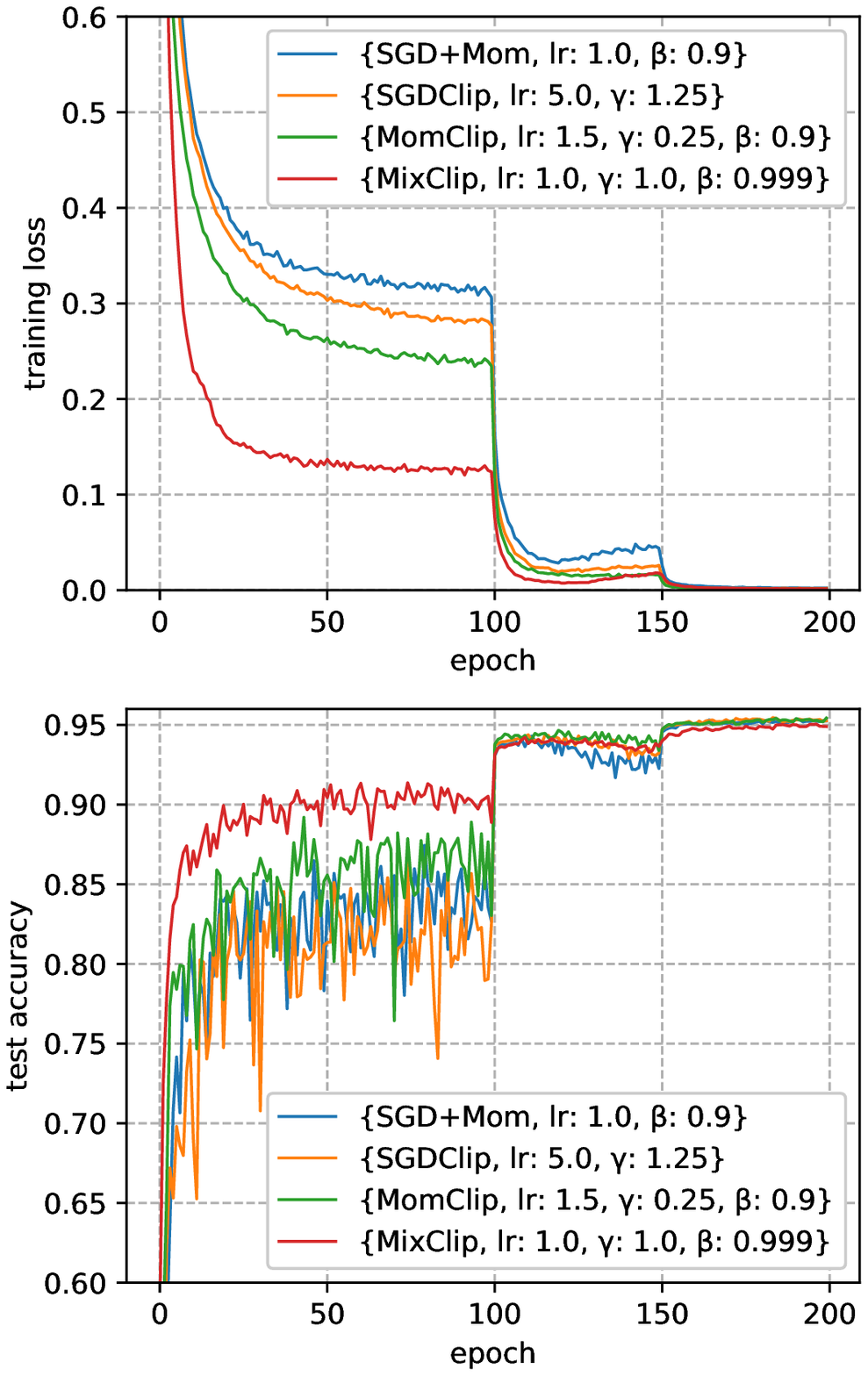}
  \label{resnet18}
  }
  \caption{(a, b) Comparison between various clipping algorithms and Adam on CIFAR-10 and PTB dataset. Adam has a fast training speed but generalize much worse than all the clipping algorithms. (c) CIFAR-10 classification using ResNet18 architecture. The result is similar to that of using ResNet-32 architecture.}
\end{figure}

\section{Additional Experiments in $(L_0,L_1)$-Smooth Setting using MNIST Dataset}
\label{appendix_experiment_mnist}
In this section, we are aiming to construct an optimization problem which provably satisfies the $(L_0,L_1)$-smoothness condition in this paper rather than the traditional $L$-smoothness condition. We then conduct experiments in both deterministic setting and stochastic setting.\par

We first cosider a binary classification problem. Suppose a  dataset $\mathcal D$ contains $n$ samples, denoted as $\{(x_i, y_i)\}_{i=1}^n$, where $x_i$ is a $d$-dimensional input vector and $y_i\in \{-1,+1\}$ is the corresponding label. A discriminant function $f$ with parameter $w,b$ is a mapping from $\mathbb R^d$ to $\mathbb R$ such that $f_{w,b}(x)=w^Tx+b$. We use the empirical error under the exponential loss function \eqref{eq:exp_loss}:
\begin{equation}
    \label {eq:exp_loss}
    L(w,b)=\mathop{\mathbb E} \limits_{(x,y)\sim \mathcal D} \exp(-yf_{w,b}(x))=\frac 1 n\sum_{i=1}^n \exp(-y_i(w^Tx_i+b))
\end{equation}
In fact, if the exponential function $\exp(\cdot)$ is replaced by $\log(1+\exp(\cdot))$, the problem becomes the well-known logistic regression. However, logistic loss has bounded second-order derivative (thus is $L$-smooth), while $\exp(\cdot)$ does not. Furthermore, exponential function is (0,1)-smooth, thus we expect $L(w,b)$ is also $(L_0,L_1)$-smooth for some $L_0,L_1$ (see the following proposition). This is why we use exponential loss here. We point out that such exponential loss is also used in a variety of algorithms, such as boosting (AdaBoost).\par
When the dataset is linearly separable, parameter $w$ will be driven to infinity through optimization, thus adding some regularization is prevalent in linear classification. We use the following term \eqref{eq:weight_decay} rather than $L_2$ norm for regularization, in order to be compatible with $L(w, b)$.
\begin{equation}
    \label {eq:weight_decay}
    R_{\lambda}(w)=\sum_{i=1}^d \left[\frac {\exp(\lambda w_i)+\exp(-\lambda w_i)} 2-1\right]
\end{equation}
In fact, $R_{\lambda}(w)$ is similar to weight decay regularization in that $R_{\lambda}(w)=\frac 1 2 \lambda^2 \|w\|^2+O(\lambda^4 \|w\|^4)$ when $w$ is small.\par
The total loss $E_{\lambda}(w, b)=L(w,b)+R_{\lambda}(w)$. We now claim that $E_{\lambda}(w, b)$ is indeed $(L_0,L_1)$-smooth.\par
\begin{proposition}
\label{dataset}
 Assume bias term $b=0$ for simplicity. Suppose the data points have bounded norm, i.e. $\|x_i\|\le R$ for all $i$ and $\lambda < R$. Let the loss function $E_{\lambda}(w, 0)$ be defined above. Then for every $\rho_1>0,\rho_2>0$, $\rho=\rho_1+\rho_2$, $E_{\lambda}(w, 0)$ is $(L_0,L_1)$-smooth w.r.t $w$ for $$L_0=\max\left(\frac{(1+\rho)\sqrt d} {\lambda}  R^2(R+d\lambda), (R^2+d\lambda^2) \left(\frac{n(R^2+d\lambda^2)} {\rho_1 R^2} \right)^{1+\frac 1 {\rho_2}}\right),L_1=\frac{(1+\rho)\sqrt d} {\lambda}  R^2.$$
\end{proposition}
We use MNIST dataset in this section, which contains 60,000 hand-writing training images. We only evaluate the training speed for different algorithms on the training set rather than the generalization capability. The loss functions is defined to be the sum of ten losses, each of which corresponds to the loss of a binary classification problem to recognize number 0 to 9. Regularization coefficient $\lambda$ is set to be 0.02.\par
To compare different algorithms, we choose the best hyperparameters $lr$ and $\gamma$ for each algorithm based on a \textit{careful} grid search. $\nu$ is set to be 0.7 for mixed clipping. The parameter initalization and all inputs in the schocastic setting are the same for all algorithms. For each run, we average the loss of the last 5 epoch in order to reduce variance. In the deterministic setting we train 500 epochs, each of which uses the entire dataset. In the stochastic setting we train 50 epochs with a mini-batch size 200. We run on 5 different random seeds ranging from 2016 to 2020 altogether and average their results.\par
Figure \ref{fig_mnist} plots the results. It is clear that in both settings, clipping is vital to a fast convergence. Also, momentum helps training, and mixed clipping performs the best in the stochastic setting.

\begin{figure}
\centering
\subfigure[The deterministic setting]{
\includegraphics[width=0.4\textwidth]{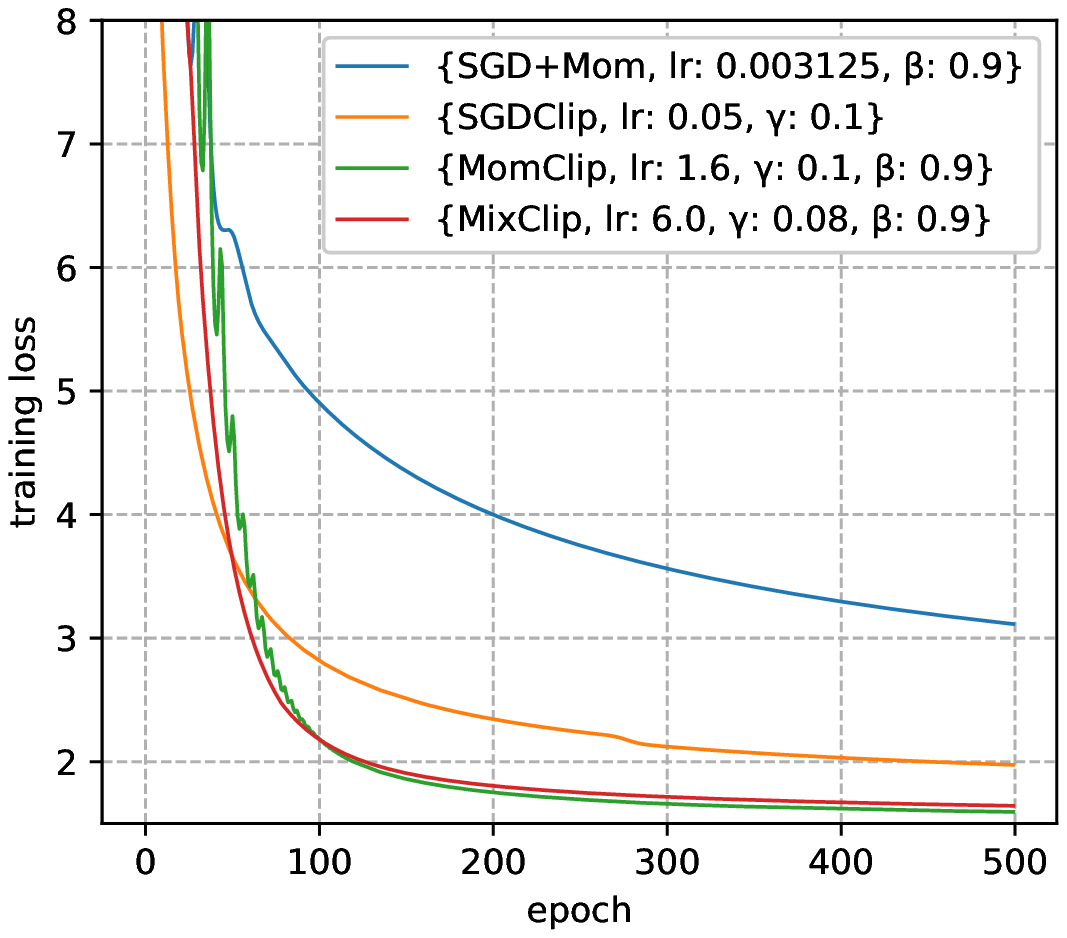}}
\hspace{10pt}
\subfigure[The stochastic setting]{
\includegraphics[width=0.4\textwidth]{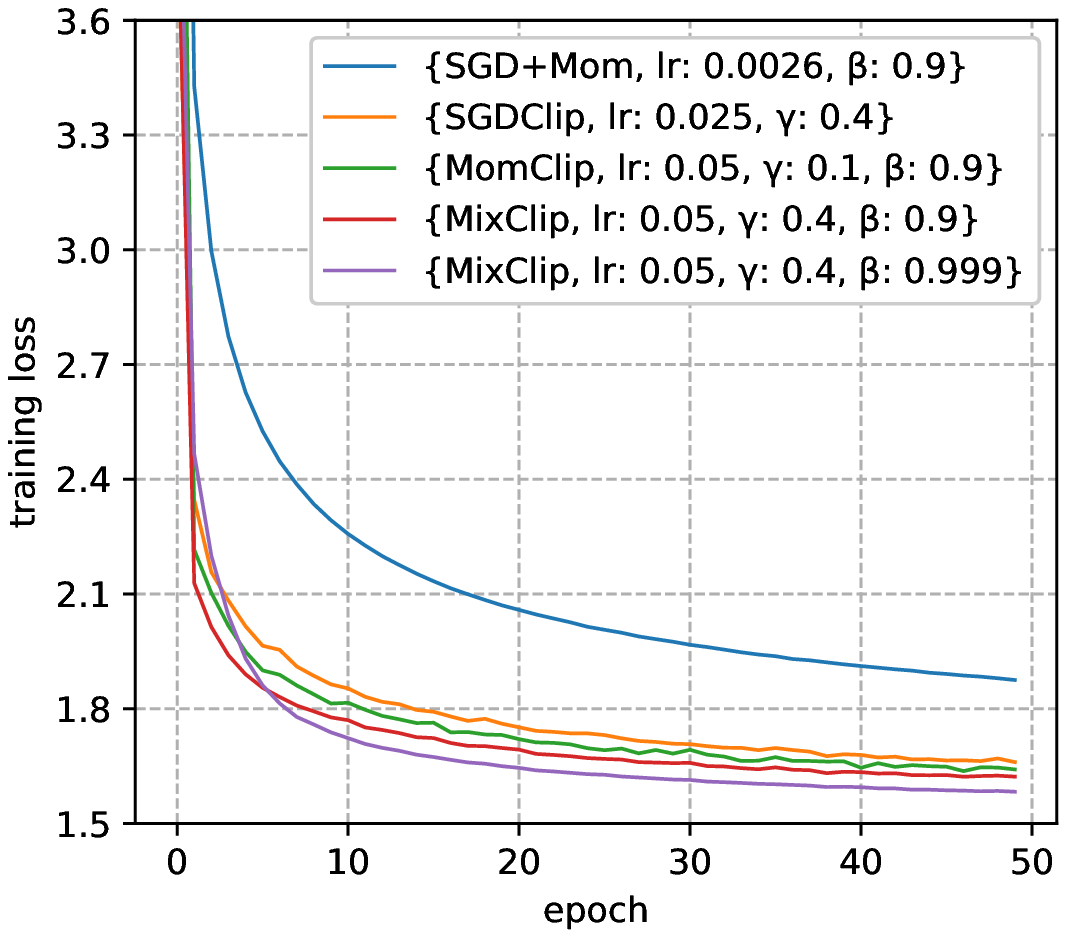}}
\caption{Experimental results on MNIST.}
\label{fig_mnist}
\end{figure}

\subsection{Proof of Proposition \ref{dataset}}
Consider the augmented dataset $\mathcal {\tilde D}$ containing $n+2d$ data points $\{z_i\}_{i=1}^{n+2d}$, with
\begin{equation}
    \label{eq:appendix_0}
    z_{i}=\left\{\begin{array}{ll}-x_{i} y_{i} & i \leq n \\ \lambda e_{i-n} & n<i \leq n+d \\ -\lambda e_{i- n-d} & n+d<i \leq n+2 d\end{array}\right.
\end{equation}
where $e_i$ is the vector with all zero entries except the $i$th entry which is one. Denote coefficient vector $c\in \mathbb R^{n+2d}$ with elements $c_i=1/n$ if $i \le n$ and $c_i=1/2$ otherwise. It directly follows that the original problem with regularization term can be written as:
\begin{equation}
    E_{\lambda}(w)=\frac 1 n\sum_{i=1}^n \exp(w^Tz_i) + \frac 1 2\sum_{i=n+1}^{n+2d} \exp(w^Tz_i) -d= \sum_{i=1}^{n+2d} c_i\exp(w^Tz_i)-d
\end{equation}
Let $M=\mathop{\max} \limits_{i\in [n+2d]} w^T z_i$. Let $\rho_1>0,\rho_2>0$ be two constants. Pick $M_0=\left(1+\frac 1 {\rho_2}\right)\log \frac {n(R^2+d\lambda^2)}{\rho_1 R^2} $. We consider the following two cases:\par
(1)$M\le M_0$. In this case $\|\nabla ^2 E(w)\|$ can be directly upper bounded:
\begin{equation}
\label{eq:appendix_1}
\begin{aligned}
    \|\nabla ^2 E(w)\|&\le \frac 1 n\sum_{i=1}^n \exp(w^Tz_i)\|z_i\|^2 + \frac 1 2 \sum_{i=n+1}^{n+2d} \exp(w^Tz_i) \|z_i\|^2 \\
    & \le (R^2+d\lambda^2)\exp(M)\\
    &\le  (R^2+d\lambda^2) \left(\frac{n(R^2+d\lambda^2)} {\rho_1 R^2} \right)^{1+\frac 1 {\rho_2}}
\end{aligned}
\end{equation}
The first inequality in \eqref{eq:appendix_1} uses the triangular inequality of matrix spectral norm and $\|zz^T\|=\|z^Tz\|=\|z\|^2$.\par
(2)$M > M_0$. Decompose $M_0$ to be $M_0=M_1+M_2$ where 
$$M_1=\log \frac  {n(R^2+d\lambda^2)}{\rho_1R^2}, M_2=\frac 1 {\rho_2}\log \frac  {n(R^2+d\lambda^2)}{\rho_1 R^2}.$$
Define set $I=\{i\in [n+2d]:w^Tz_i\ge M-M_1\}$ and $I_2=\{i\in [n+2d]:w^Tz_i< 0\}$. Then
\begin{align}
    \|\nabla E(w)\|&= \sum_{i=1}^{n+2d} c_i\exp(w^Tz_i)z_i\\
    \label{eq:appendix_2}
    &\ge \sum_{i=1}^{n+2d} c_i\exp(w^Tz_i)\frac {w^Tz_i} {\|w\|} \\
    \label{eq:appendix_3}
    &\ge \sum_{i\in I} c_i\exp(w^Tz_i)\frac {M-M_1} {\|w\|}-\sum_{i\in I_2} c_i\|z_i\|\\
    \label{eq:appendix_4}
    &\ge \sum_{i\in I} c_i\exp(w^Tz_i)\frac {M-M_1} {\|w\|}-(R+d\lambda)
\end{align}
In \eqref{eq:appendix_2} we use the Cauchy-Schwartz inequality; In \eqref{eq:appendix_3} we partition the index $\{i:i\in [n+2d]\}$ to three subsets $I$, $I_2$ and $[n+2d]\backslash (I\cup I_2)$, and use the lower bound and upper bound of $w^T z_i>0$ for each set.\par
Similar, we can upper bound $\|\nabla^2 E(w)\|$:
\begin{align}
    \|\nabla^2 E(w)\|&\le \sum_{i\in I} c_i\exp(w^Tz_i)\|z_i\|^2+ \sum_{i\notin I} c_i\exp(w^Tz_i)\|z_i\|^2\\
    \label{eq:appendix_5}
    &\le \sum_{i\in I} c_i\exp(w^Tz_i)R^2+(R^2+d\lambda^2)\exp(M-M_1)
\end{align}
To bound $\exp(M_1)$, we again bound $\|\nabla E(w)\|$ from a different perspective:
\begin{align}
    \|\nabla E(w)\| &\ge \sum_{i\in I} c_i\exp(w^Tz_i)\frac {w^Tz_i} {\|w\|}-\sum_{i\in I_2} c_i\|z_i\|\\
    \label{eq:appendix_6}
    &\ge \frac 1 n \exp(M)\frac M {\|w\|}-(R+d\lambda)
\end{align}
where \eqref{eq:appendix_6} is obtained by selecting the $i$ with the largest $w^Tz_i$ which is equal to $M$. Substitute \eqref{eq:appendix_4} and \eqref{eq:appendix_6} into \eqref{eq:appendix_5} then we get
\begin{align}
    \|\nabla^2 E(w)\|&\le  \left(\frac{R^2}{M-M_1}+\frac{n(R^2+d\lambda^2)} {M\exp(M_1)}\right){\|w\|}(\|\nabla E(w)\|+R+d\lambda) \\
    \label{eq:appendix_7}
    &=  \left(\frac{ R^2}{M-M_1}+\frac{\rho_1 R^2}{M}\right){\|w\|}(\|\nabla E(w)\|+R+d\lambda)
\end{align}
\par
Since $M=\mathop{\max} \limits_{i\in [n+2d]} w^T z_i$ implies that $|\lambda w_k|\le M$ for all $k\in [d]$ from \eqref{eq:appendix_0}, we can upper bound the norm of $w$: $\|w\|\le \frac {M\sqrt d} {\lambda}$. Substitute this into \eqref{eq:appendix_7} we get
\begin{align}
    \|\nabla^2 E(w)\|&\le  \left(\frac M {M-M_1}+\rho_1\right) \frac {\sqrt d } {\lambda}R^2(\|\nabla E(w)\|+R+d\lambda)\\
    &\le  \frac{(1+\rho_1+\rho_2)\sqrt d} {\lambda}  R^2(\|\nabla E(w)\|+R+d\lambda)
\end{align}
Combining the above two cases concludes the proof.

\end{document}